\documentclass[jair,twoside,11pt,theapa]{article}
\usepackage[bookmarks=true]{hyperref}
\hypersetup{%
    bookmarks=true,    
    colorlinks=true,       
    linkcolor=blue,       
    citecolor=black,       
    filecolor=black,        
    urlcolor=purple,        
    linktoc=page            
}
\usepackage{jair, theapa, rawfonts}

\usepackage[utf8]{inputenc} 
\usepackage[T1]{fontenc}    
\usepackage{url}            
\usepackage{booktabs}       
\usepackage{amsfonts}       
\usepackage{nicefrac}       
\usepackage{microtype}      
\usepackage{xcolor}         
\usepackage{footmisc}

\let\cite\shortcite

\usepackage{graphicx}

\usepackage{wrapfig}
\usepackage{tikz}

\usepackage{xfrac}
\usepackage{bm}      
\usepackage{algorithm}
\usepackage{algorithmic}
\usepackage{amsmath,amsthm,amssymb}
\usepackage{mathtools}
\usepackage{subfig}
\usepackage{caption}
\theoremstyle{definition}

\newsavebox\curwrapfig
\makeatletter
\let\NAT@parse\undefined
\long\def\wrapfiguresafe#1#2#3{%
  \sbox\curwrapfig{#3}%
  \par\penalty-100%
  \begingroup 
    \dimen@\pagegoal \advance\dimen@-\pagetotal 
    \advance\dimen@-\baselineskip 
    \ifdim \ht\curwrapfig>\dimen@ 
      \break%
    \fi%
  \endgroup%
  \begin{wrapfigure}{#1}{#2}%
    \usebox\curwrapfig%
  \end{wrapfigure}%
}
\makeatother
\usepackage{multirow}
\usepackage[normalem]{ulem}
\useunder{\uline}{\ul}{}

\def\one{\mbox{1\hspace{-4.25pt}\fontsize{12}{14.4}\selectfont\textrm{1}}}

\usepackage{tikz}
\usepackage[most]{tcolorbox}
\newtcolorbox{mybox}[3][]
{
  colframe = #2!25,
  colback  = #2!10,
  coltitle = #2!20!black,  
  title    = {#3},
  #1,
}

\usetikzlibrary{arrows}
\usetikzlibrary{trees}

\usepackage{enumitem}

\usepackage{color}

\usepackage{listings}
\DeclareFixedFont{\ttb}{T1}{txtt}{bx}{n}{9.5} 
\DeclareFixedFont{\ttm}{T1}{txtt}{m}{n}{9.5}  
\definecolor{codeblue}{rgb}{0,0,0.6}
\definecolor{codegreen}{rgb}{0,0.6,0}
\definecolor{dark-blue}{rgb}{0.15,0.15,0.4}
\definecolor{codepurple}{rgb}{0.6,0,0.6}

\newcommand\pythonstyle{\lstset{
    language=Python,
    basicstyle=\scriptsize\ttfamily,
    otherkeywords={self,with},             
    keywordstyle=\color{codepurple},
    emph={__init__, dim, None},
    emphstyle=\color{codeblue},
    stringstyle=\color{codegreen},
    commentstyle=\color{codegreen},
    frame=none,              
    showstringspaces=false,
    breaklines=true,
    numbers=left,
    numbersep=3pt,
    tabsize=2,
    breakatwhitespace=false,
    abovecaptionskip=2ex,
    captionpos=b,
}}

\lstnewenvironment{python}[1][]
{
    
    \pythonstyle
    \lstset{#1}
}{}

\lstnewenvironment{pythoninline}[1][]
{
    \pythonstyle
    \lstset{#1}
}{}

\usepackage{float}
\newfloat{codeexample}{thp}{lop}
\floatname{codeexample}{Code Example}

\usepackage{xfrac}
\usepackage{adjustbox}
\usepackage{collectbox}

\usepackage{microtype}



\newenvironment{claim}[1]{\par\noindent\underline{Claim:}\space#1}{}

\newenvironment{theorem}[1]{\par\noindent\textbf{Theorem:}\space#1}{}

\jairheading{70}{2021}{1-15}{10/2020}{1/2021}
\ShortHeadings{\texttt{HEBO}: Pushing The Limits of Sample-Efficient Hyperparameter Optimisation}
{Cowen-Rivers et al. }
\firstpageno{1}

\begin{document}

\title{\texttt{HEBO}: Pushing The Limits of Sample-Efficient Hyperparameter Optimisation}

\author{\name Alexander I. Cowen-Rivers \thanks{\:\:Equal contribution} \thanks{\:\:Huawei Noah's Ark Lab} \thanks{\:\:Technische Universität Darmstadt} \email alexander.cowen.rivers@huawei.com \thanks{\:\:Corresponding Author}
\AND
\name Wenlong Lyu\footnotemark[1] \footnotemark[2] \email lvwenlong2@huawei.com
\AND
\name Rasul Tutunov\footnotemark[1] \footnotemark[2] \email rasul.tutunov@huawei.com \\
\AND
\name Zhi Wang \footnotemark[2] \email wangzhi55@huawei.com
\AND
\name Antoine Grosnit \footnotemark[2]  \email  antoine.grosnit@huawei.com
\AND
\name Ryan Rhys Griffiths \footnotemark[2] \thanks{\:\:University of Cambridge}   \email ryan.rhys.griffiths1@huawei.com
\AND
\name Alexandre Max Maravel \footnotemark[2]   \email alexandre.maravel@huawei.com
\AND
\name Hao Jianye \footnotemark[2]  \email haojianye@huawei.com
\AND
\name Jun Wang \footnotemark[2] \thanks{\:\:University College London}  \email w.j@huawei.com
\AND
\name Jan Peters \footnotemark[3]   \email peters@ias.tu-darmstadt.de
\AND
\name Haitham Bou-Ammar \footnotemark[2] \footnotemark[6] \thanks{\:\: Honorary position.}  \email  haitham.ammar@huawei.com
}


\maketitle

\begin{abstract}
In this work we rigorously analyse assumptions inherent to black-box optimisation hyper-parameter tuning tasks. Our results on the Bayesmark benchmark indicate that heteroscedasticity and non-stationarity pose significant challenges for black-box optimisers. Based on these findings, we propose a Heteroscedastic and Evolutionary Bayesian Optimisation solver (\texttt{HEBO}). \texttt{HEBO} performs non-linear input and output warping, admits exact marginal log-likelihood optimisation and is robust to the values of learned parameters. We demonstrate \texttt{HEBO}'s empirical efficacy on the NeurIPS 2020 Black-Box Optimisation challenge, where \texttt{HEBO} placed first. Upon further analysis, we observe that \texttt{HEBO} significantly outperforms existing black-box optimisers on 108 machine learning hyperparameter tuning tasks comprising the Bayesmark benchmark. Our findings indicate that the majority of hyper-parameter tuning tasks exhibit heteroscedasticity and non-stationarity, multi-objective acquisition ensembles with Pareto front solutions improve queried configurations, and robust acquisition maximisers afford empirical advantages relative to their non-robust counterparts. We hope these findings may serve as guiding principles for practitioners of Bayesian optimisation. All code is made available at \href{https://github.com/huawei-noah/HEBO}{https://github.com/huawei-noah/HEBO}.
\end{abstract}

\section{Introduction}
Although achieving significant success across numerous applications ~\cite{bobadilla2013recommender,litjens2017survey,fatima2017survey,kandasamy2018neural,cowen2020samba}, the performance of machine learning models chiefly depends on the correct setting of hyper-parameters. As models grow larger and more complex, efficient and autonomous hyper-parameter tuning algorithms become crucial determinants of performance. A variety of methods from black-box and multi-fidelity optimisation \cite{kandasamy2017multi,sen2018multi} have been adopted~ for hyperparameter tuning with varying degrees of success. Techniques such as Bayesian optimisation (BO), for example, enable sample efficiency (in terms of black-box evaluations) at the expense of high computational demands, while ``unguided'' bandit-based approaches can fail to converge~\cite{falkner2018bohb}. Identifying such failure modes, the authors in~\cite{falkner2018bohb} built on~\cite{li2017hyperband} and proposed a combination of bandits and BO that achieves the best of both worlds; fast convergence and computational scalability. More recently in the context of the 2020 NeurIPS competition on Black-Box Optimisation, many BO variants have been convincingly demonstrated to be superior to random search for the task of hyper-parameter tuning \cite{2021_Turner}. Though impressive, such successes of BO and alternative black-box optimisers, belie a set of restrictive modelling and acquisition function assumptions. We begin by describing these assumptions. \\


\noindent \textbf{Modelling Assumptions:} A core determinant of BO performance is the set of data modelling assumptions required to specify an appropriate probabilistic model of the black-box objective (e.g., the choice of validation loss in hyper-parameter tuning tasks). The model should not only provide accurate point estimates, but should also maintain calibrated uncertainty estimates to guide exploration of the objective. Amongst many possible surrogates~\cite{2016_Springenberg,2011_Hutter}, Gaussian processes~\cite{williams1996gaussian} (GPs) are the default choice due to their flexibility and sample efficiency. Growing interest in applications of Bayesian optimisation has catalysed engineering feats that enhance scalability and training efficiency of GP surrogates by exploiting graphical processing units~\cite{knudde2017gpflowopt,balandat2020botorch}.

Similar to any other framework, the correct specification of a GP model is dictated by the data modelling assumptions imposed by the user. For instance, a homoscedastic GP suffers from misspecification when required to model data with heteroscedastic noise whilst stationary GPs fail to track non-stationary targets. The aforementioned shortcomings are not unnatural across a range of real-world problems \cite{2007_Kersting,2021_Griffiths,2021_Griffiths_mrk} and hyper-parameter tuning of machine learning algorithms is no exception, as illustrated in our hypothesis tests of Section~\ref{Sec:Hetero}. Hence, even if one succeeds in improving computational efficiency, frequently-made assumptions such as homoscedasticity and stationarity can easily inhibit the performance of any BO-based hyper-parameter tuning algorithm. Despite the importance of these assumptions in practice, GPs that presume homoscedasticity and stationarity still constitute the most common choice of surrogate. \\

\noindent \textbf{Acquisition Function \& Optimiser Assumptions:} Modelling choices such as those described above are not unique to the GP fitting procedure but rather transcend to other steps in the BO algorithm. Precisely, given a model that adheres to some (or all) assumptions mentioned above, the second step involves maximising an acquisition function to query novel input locations that are then evaluated. Hence, practitioners introduce additional constraints relating to the category of optimisation variables and the choice of acquisition function. When it comes to variable categories, mainstream implementations~\cite{knudde2017gpflowopt,balandat2020botorch} assume continuous domains and employ first and second-order optimisers such as LBFGS~\cite{liu89} and ADAM~\cite{Adam} to propose query locations. Real-valued configurations cover but a subset of possible machine learning hyper-parameters rendering discrete variable categories out of scope, an example being the hidden layer size in deep networks. Moreover, from the point of view of acquisition functions, libraries tend to presuppose that one unique acquisition performs best in a given task, while research has shown that benefits that can arise from a combined solution \cite{2014_Shahriari,2016_Shahriari,lyu2018batch} as we demonstrate in Section~\ref{Sec:Exp}. \\

\noindent \textbf{Contributions:} Having identified important modelling choices in BO, our goal in this paper is to provide empirical insight into the impact of modelling choice on empirical performance. As a case study, we consider best practices for hyper-parameter tuning. We wish for our findings to be applicable across a broad range of tasks and datasets, be attentive to the effect of random initialisation on algorithmic performance, and naturally, be reproducible. As such, we prefer to build on established benchmark packages, especially those that facilitate fast and scalable evaluations with multi-seeding protocols. To that end, we undertake our evaluation in 2140 experiments from 108 real-world problems from the UCI repository~\cite{2019_Dua}, which was also the testbed of choice for the NeurIPS 2020 Black-Box Optimisation challenge~\cite{2021_Turner}. Our findings point towards the following conclusions: 
\begin{enumerate}
\item Hyper-parameter tuning tasks exhibit significant levels of heteroscedasticity and non-stationarity.
\item Input and output warping mitigate the effects of heteroscedasticity and non-stationarity giving rise to better performing tuning algorithms with higher mean and median performance across all 108 black-box functions under examination. 
\item Individual acquisition functions tend to conflict in their solution (i.e., an optimum for one acquisition function can be a sub-optimal point for another and vice versa). Using a multi-objective formulation significantly improves performance;.
\end{enumerate}

To verify our principal conclusions, we conduct additional ablation studies on our proposed solution method, Heteroscedastic and Evolutionary Bayesian Optimisation (HEBO) which attempts to address the shortcomings identified in our analysis and placed first in the 2020 NeurIPS Black-Box Optimisation Challenge. We obtain a ranked order of importance for significant components of HEBO, finding that output warping, multi-objective acquisitions and input warping lead to the most significant improvements followed by robust acquisition function formulations. 

\section{Standard Design Choices in BO}  
As discussed earlier, the problem of hyper-parameter tuning can be framed as an instance of black-box optimisation:
\begin{equation}
\label{Eq:BB}
 \arg\max_{\bm{x} \in \mathcal{X}} f(\bm{x}),  
\end{equation}
with $\bm{x}$ denoting a configuration choice, $\mathcal{X}$ a (potentially) mixed design space, and $f(\bm{x})$ a validation accuracy we wish to maximise. In this paper, we focus on BO as a solution concept for black-box problems of the form depicted in Equation~\ref{Eq:BB}. BO considers a sequential decision approach to the global optimisation of a black-box function $f: \mathcal{X} \rightarrow \mathbb{R}$ over a bounded input domain $\mathcal{X}$. At each decision round, $i$, the algorithm selects a collection of $q$ inputs $\bm{x}^{(\text{new})}_{1:q} \in \mathcal{X}^q$ and observes values of the \emph{black-box} function $\bm{y}^{(\text{new})}_{1:q} = f(\bm{x}^{(\text{new})}_{1:q})$. The goal is to rapidly approach the maximum $\bm{x}^{\star} = \arg\max_{\bm{x} \in \mathcal{X}} f(\bm{x})$. Since both $f(\cdot)$ and $\bm{x}^{\star}$ are unknown, solvers need to trade off exploitation and exploration during this search process. 

To achieve this goal, BO algorithms operate in two steps. In the first, a Bayesian model is learned, while in the second an acquisition function determining new query locations is maximised. Next, we survey frequently-made assumptions in mainstream BO implementations and contemplate their implications for performance.

\subsection{Modelling Assumptions} 
When black-boxes are real-valued, Gaussian processes~\cite{2006_Williams} are effective surrogates due to their flexibility and ability to maintain calibrated uncertainty estimates. In established implementations of BO, designers place GP priors on latent functions, $f(\cdot)$, which are fully specified through a mean function, $m(\bm{x})$, and a covariance function or kernel $k_{\bm{\theta}}(\bm{x}, \bm{x}^{\prime})$ with $\bm{\theta}\in\mathbb{R}^p$ representing kernel hyper-parameters. The model specification is completed by defining a likelihood. Here, practitioners typically assume that observations $y_{l}$ adhere to a Gaussian noise model such that $y_l = f(\bm{x}_l) + \epsilon_l$ where $\epsilon_l \sim \mathcal{N}(0, \sigma_{\text{noise}}^{2})$. This, in turn, generates a Gaussian likelihood of the form $y_l | \bm{x}_{l} \sim \mathcal{N}(f_l, \sigma_{\text{noise}}^{2})$ where we use $f_l$ to denote $f(\bm{x}_{l})$ with $f(\bm{x}) \sim \mathcal{G}\mathcal{P}(m(\bm{x}), k_{\bm{\theta}}(\bm{x}, \bm{x}^{\prime}))$. Additionally, a further design choice commonly made by practitioners is that the GP kernel is stationary, depending only on the norm between $\bm{x}$ and $\bm{x}^{\prime}$, $||\bm{x} - \bm{x}^{\prime}||$. From this exposition, we conclude two important modelling assumptions stated as \emph{data stationarity} and \emph{homoscedasticity of the noise distribution}. Where \textbf{homoscedasticity} implies a constant noise term $\sigma_{\text{noise}}^{2}$. \textbf{Heteroscedasticity} is usually harder to model as implies $\sigma_{\text{noise}}^{2}$ is a function of the input: i.e., depending on the data, the noise changes around the mean. Of course, it is clear that there are significant differences between homoscedastic functions and heteroscedastic functions, and later we show indeed heteroscedastic functions require a different approach to optimise over than the typical homoscedastic (synthetic) functions usually researched in Bayesian Optimisation. If the true latent process does not adhere to these assumptions, the resultant model will be a poor approximation to the black-box. Realising the potential empirical implications of these modelling choices, we identify the first two questions addressed by this paper: \\

\par{\textbf{Q.I.}} Are hyper-parameter tuning tasks stationary? \\
\par{\textbf{Q.II.}} Are hyper-parameter tuning tasks homoscedastic? \\

\noindent In Section~\ref{Sec:Hetero}, we show that even amongst the simplest hyper-parameter tuning tasks, the null hypothesis may be rejected in the case of statistical hypothesis tests for heteroscedasticity and non-stationarity. 

\subsection{Acquisition Function \& Optimisation Assumptions} \label{Sec:AcqAssumptions}
Acquisition functions trade off exploration and exploitation by utilising statistics from the posterior $p_{\bm{\theta}}(f(\cdot)|\mathcal{D})$ with $\mathcal{D}$ denoting the data (hyper-parameter configurations as inputs and validation accuracy as outputs) collected so far. Under a GP surrogate with Gaussian-corrupted observations $y_\ell = f(\boldsymbol{x}_\ell) + \epsilon_\ell$ where $\epsilon_\ell \sim \mathcal{N}(0, \sigma^2)$, and given a data set $\mathcal{D} = \{\boldsymbol{x}, \boldsymbol{y}\}$, the joint distribution of $\mathcal{D}$ and an arbitrary set of input points $\boldsymbol{x}_{1:q}$ is given by

\begin{align*}
&\left[\begin{array}{c}
      \bm{y}  \\
      f(\bm{x}_{1:q}) 
        \end{array}
        \right] \Bigg| \ \bm{\theta} \sim \nonumber  \mathcal{N}\left(\left[\begin{array}{cc}
        m(\bm{x}) \\
        m(\bm{x}_{1:q})
        \end{array}
        \right], \left[\begin{array}{cc}
      \bm{K}_{\bm{\theta}} + \sigma^{2} \bm{I} &  \bm{k}_{\bm{\theta}}(\bm{x}_{1:q})  \\
      \bm{k}^{\mathsf{T}}_{\bm{\theta}}(\bm{x}_{1:q}) & \bm{k}_{\bm{\theta}}(\bm{x}_{1:q}, \bm{x}_{1:q})  
        \end{array}
        \right]\right),
    \end{align*}
where $\bm{K}_{\bm{\theta}} = \bm{K}_{\bm{\theta}}(\bm{x}, \bm{x})$ and $\bm{k}_{\bm{\theta}}(\bm{x}_{1:q}) = \bm{k}_{\bm{\theta}}(\bm{x}, \bm{x}_{1:q})$. From this joint distribution one can derive though marginalisation \cite{2006_Williams} the posterior predictive $p(f(\bm{x}_{1:q})|\mathcal{D}) = \mathcal{N}(\bm{\mu}_{\bm{\theta}}(\bm{x}_{1:q}), \bm{\Sigma}_{\bm{\theta}}(\bm{x}_{1:q}))$ with:
\begin{align*}
    \bm{\mu}_{\bm{\theta}}(\bm{x}_{1:q}) &= m(\bm{x}_{1:q}) + \bm{k}_{\bm{\theta}}(\bm{x}_{1:q})^\top(\bm{K}_{\bm{\theta}} + \sigma^{2} \bm{I})^{-1}(\bm{y}-m(\bm{x})) \\
     \bm{\Sigma}_{\bm{\theta}}(\bm{x}_{1:q}) & = \bm{K}_{\bm{\theta}}(\bm{x}_{1:q}, \bm{x}_{1:q}) - \bm{k}_{\bm{\theta}}(\bm{x}_{1:q})^\top(\bm{K}_{\bm{\theta}} + \sigma^{2} \bm{I})^{-1}\bm{k}_{\bm{\theta}}(\bm{x}_{1:q}).
\end{align*} 

\noindent As such we note that $p(f(\bm{x}_{1:q})|\mathcal{D}) = \mathcal{N}(\bm{\mu}_{\bm{\theta}}(\bm{x}_{1:q}), \bm{\Sigma}_{\bm{\theta}}(\bm{x}_{1:q}))$. In this paper, we focus on three widely-used myopic acquisition functions which in a reparameterised form can be written as~\cite{wilson2018marginal}: \\

\noindent \textbf{Expected Improvement (EI):}
\begin{align*}
     \label{Eq:q_EI}
    \alpha^{\bm{\theta}}_{\text{EI}}(\bm{x}_{1:q}|\mathcal{D}) &=  \mathbb{E}_{\text{post.}}\Bigg[\max_{j \in 1:q}\{\text{ReLU}(f(\bm{x}_{j})-  f(\bm{x}^{+}))\}\Bigg],
\end{align*}
where the subscript $\text{'post.'}$ is the predictive posterior of a GP~\cite{2006_Williams}, $\bm{x}_{j}$ is the $j^{th}$ vector of $\bm{x}_{1:q}$, and $\bm{x}^{+}$ is the best performing input in the data so far. \\

\noindent \textbf{Probability of Improvement (PI):} 
\begin{align*}
    \alpha^{\bm{\theta}}_{\text{PI}}(\bm{x}_{1:q}|\mathcal{D}) &=  \mathbb{E}_{\text{post.}}\Bigg[\max_{j \in 1:q}\{\one\{{f}(\bm{x}_{j})-  f(\bm{x}^{+})\}\}\Bigg],
\end{align*}
where $\one\{\cdot\}$ is the left-continuous Heaviside step function. \\

\noindent \textbf{Upper Confidence Bound (UCB):}
\begin{align*}
    \alpha^{\bm{\theta}}_{\text{UCB}}(\bm{x}_{j}) &= \mathbb{E}_{\text{post.}}\Bigg[\max_{j\in1:q}\Bigg\{{\mu}_{\bm{\theta}}(\bm{x}_{j}) + \sqrt{\sfrac{\beta \pi}{2}}|{\gamma}_{\bm{\theta}}(\bm{x}_{j})|\Bigg\}\Bigg],
\end{align*}
where ${\mu}_{\bm{\theta}}(\bm{x}_{j})$ is the posterior mean of the predictive distribution and ${\gamma}_{\bm{\theta}}(\bm{x}_{j}) = {f}(\bm{x}_{j})  - {\mu}_{\bm{\theta}}(\bm{x}_{j})$. When it comes to practicality, generic BO implementations make additional assumptions during the acquisition maximisation step. First, it is assumed that one of the aforementioned acquisitions works best for a specific task, and that the GP model is an accurate approximation to the black-box. However, when it comes to real-world applications, both of these assumptions are difficult to validate; the best-performing acquisition is challenging to identify upfront and GP models may easily be misspecified. With this in mind, we identify a third question that we wish to address: \\

\par{\textbf{Q.III.}} Can acquisition function solutions conflict in hyper-parameter tuning tasks? \\

\noindent In the following section, we affirm that acquisitions can conflict even on the simplest of hyper-parameter tuning tasks. Moreover, we show that a robust formulation to tackle misspecification of acquisition maximisation can improve overall performance (see Section~\ref{Sec:Robust}).

\section{Modelling Assumption Analysis}\label{Sec:Answers}
Before discussing the improvements afforded to BO via our solution method, we detail analyses conducted to answer questions ($\textbf{Q.I.}$, $\textbf{Q.II.}$, and $\textbf{Q.III.}$) posed in the previous section. Our analyses indicate: \\ \\
\underline{\textbf{A.I.}:} Even simple hyper-parameter tuning tasks exhibit significant heteroscedasticity.\\
\underline{\textbf{A.II.}:} Even simple  hyper-parameter tasks exhibit significant non-stationarity.\\
\underline{\textbf{A.III.}:} Acquisition functions conflict in their optima, occasionally leading to opposing solutions. \\ \\
\textbf{Experiment Setting:} We create a wide range of hyper-parameter tasks (108) across a variety of classification and regression problems. We use nine models, (e.g. multilayer perceptrons, support vector machines) and six datasets (two regression and four classification) from the UCI repository, and two metrics per dataset (such as negative log-likelihood or mean squared error). Each model possesses tuneable hyper-parameters, e.g. the number of hidden units of a neural network. The goal is to fit these hyper-parameters so as to maximise/minimise one of the specified metrics. Values of the black-box objective are stochastic with noise contributions originating from the train-test splits used to compute the losses. Experimentation was facilitated by the \texttt{Bayesmark}\footnote{\href{https://github.com/uber/bayesmark}{https://github.com/uber/bayesmark}} package. Full hyper-parameter search spaces are defined in \autoref{tab:search-space} and \autoref{tab:search-space-reg}.~\footnote{It is these search spaces that are used by the random search baseline.}. \\

\noindent \textbf{Statistical Hypothesis Testing for Heteroscedasticity and Non-Stationarity:} We describe here the statistical hypothesis tests we use to answer \textbf{Q.II.}. GP regression typically considers a conditional normal distribution of the observations $y | \cdot \sim \mathcal{N}(f(\cdot), \sigma^2(\cdot))$ and in most cases $\sigma(\cdot)^2$ is assumed to be constant, in which case the GP is termed homoscedastic. To assess whether the homoscedasticity assumption holds for the tasks under examination, we make use of Levene's test and the Fligner-Killeen test. To give the reader intuition as to how we apply these tests, Levene’s test asseses whether the variance is equal in two groups of data, assuming the data is normally distributed. I.e for a given task, given multiple evaluations of the black-box of two distinct hyperparameter sets, do the share the same variance (Homoscedasticity), or do their variances differ (Heteroscedasticity). Secondly, the Fligner-Killeen test is similarly a test for Homoscedasticity, however it is particularly useful when the data is non-normal. We refer the reader to the Appendix~\ref{test:additionalinfo} for additional information regarding the tests. 

To run these tests on a given task, we evaluate $k=50$ distinct sets of hyperparameters $\{x_i\}_{1\leq i \leq k}$ for $n=10$ times and obtain scores $\{Y_{i j}\}_{1\leq i \leq k, 1 \leq j \leq n}$, where $Y_{i j}$ is the $j^\text{th}$ score observed when evaluating the $i^\text{th}$ configuration. For $i = 1, \dots, k$, let $\sigma_i^2$ denote the observed variance of $y | x_i$, then both Levene's test and the Fligner-Killeen test share the same null hypothesis of homoscedasticity:

\begin{equation*}
    H_0: \sigma_1^2 = \dots = \sigma_k^2.
\end{equation*}

In all 108 tests, we see a p-value significantly lower than $0.05$ in $72$ tasks using Levene's test, and in $73$ tasks using Fligner-Killeen test. Such results (shown in detail in Appendix~\ref{sec:hyp_test_app}) imply that at least $66\%$ of the experimental tasks exhibit heteroscedastic behaviour. \\

\begin{table}[ht!]
\caption{Hypothesis Testing for 108 tasks with respect to GP fit. In the table below we show, out of all 108 tasks, whether the GP fit (marginal log-likelihoods) was improved (Better) when either the Output transform or Input warping was added into the surrogate model, or was worse. Furthermore, we include significantly testing using the one sided t-test and detail how many tasks the GP fit was significantly better or worse with these additional modelling components. We find that output transformations which tackle heteroscedasticity significantly improve GP modelling capabilities in general (improve marginal log-likelihoods). Similarly, input transformations which tackle non-stationarity significantly improve GP modelling capabilities in general.}
\centering
\begin{tabular}{lllll}
\hline
 & Better & Sig. Better & Worse & Sig. Worse \\ \hline
Heteroscedasticity (Output Transform) & 70 (65\%)  & 58 (54\%) & 38 (35\%) & 25 (23\%)\\
Non-Stationarity (Input Warping) & 106 (98\%) & 79 (73\%) & 2 (2\%) & 0 (0\%)
\end{tabular}
\label{stat}
\end{table}

\subsection{Answer A.I.: Simple Hyper-parameter Tuning Tasks are Non-Stationary} \label{Sec:nonStat}
To assess the impact of the extent of non-stationarity on BO performance, we conduct probabilistic regression experiments to gauge the predictive performance of a stationary GP on the hyper-parameter tuning tasks with and without input warping transformations which correct for non-stationarity. We first run a one-sided t-test for each of the 108 tasks where the null hypothesis is that the application of the input warping yields no difference in the log probability metric. In \autoref{stat} significance tests show that in 106/108 tasks, the log probability metric is more favourable when input warping is applied. In 79/108 tasks, the gain is significant at the 95\% level of confidence (p-value $< 0.025$). It is clear that tackling Non-stationarity improves GP fit as shown in Table~\ref{stat} and improves BO performance, as shown by the algorithm BO Base w Non-stationarity in Figure~\ref{ablation}. We thus conclude that non-stationarity is an important consideration for BO performance due to the observed effect on the log probability metric.

%
\subsection{Answers A.II.: Simple Hyper-parameter Tuning Tasks are Heteroscedastic}\label{Sec:Hetero}
We perform an analogous hypothesis test as in Section~\ref{Sec:nonStat}, assessing a vanilla GP's performance with and without output transformations (Box-Cox/ Yeo-Johnson). We run a two-sided paired t-test for each of the 108 tasks where the null hypothesis is that the application of the output transform yields no difference in the log probability metric. In \autoref{stat} significance tests show that in 70/108 tasks, the log probability metric is more favourable when output transformations are applied. In 58/108 tasks, the gain is significant at the 95\% level of confidence (p-value $< 0.025$). It is clear that tackling Heteroscedasticity improves GP fit as shown in Table~\ref{stat} and improves BO performance, as shown by the algorithm BO Base w Heteroscedasticity in Figure~\ref{ablation}. We thus conclude that heteroscedasticity is an important consideration for BO performance due to it isimpact on the log probability metric.

Furthermore, to gauge the level of heteroscedasticity in the underlying data, we use the Fligner-Killeen~\cite{fligner1976distribution} and Levene~\cite{levene1960contributions} tests. For both tests, the null hypothesis is that the underlying black-box function noise process is homoscedastic. In all 108 tests, we see a p-value significantly lower than $0.05$ in $72$ tasks using Levene's test, and in $73$ tasks using Fligner-Killeen. Such results (shown in detail in Appendix~\ref{sec:hyp_test_app}) imply that at least $66\%$ of the experimental tasks exhibit heteroscedastic behaviour.

\subsection{Answer A.III.: No Clear Winner}\label{Sec:AnswerQ3}
\begin{figure*}
    \centering
    \includegraphics[width=1.0\textwidth]{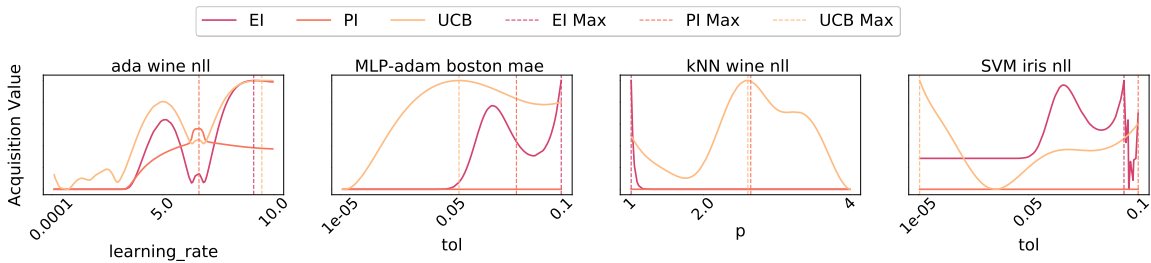}
    \caption{Examples depicting conflicting acquisitions across data sets (Wine, Boston Housing, and Iris) and models (AdaBoost, Multilayer perceptron, K-Nearest neighbours, and support vector machines). The y-axis shows the acquisition value, and x-axis a given configuration of hyperparameters. Clearly, in these examples, not only different acquisitions lead to different optima, but it can be seen that such solutions might conflict (minimum value for one acquisition function is a maximum value for another acquisition function).}
    \label{Fig:AcqF}
\end{figure*}
It has previously been observed that acquisition functions can conflict in their optima~\cite{2014_Shahriari}. To provide further support for the answer to \textbf{Q.III.}, we collect 128 samples from each task by evaluating various hyper-parameter configurations across metrics. We then assemble a data set $\mathcal{D} = \{\textbf{hyper-param}_i, y_i\}_{i=1}^{32}$, where $\textbf{hyper-param}_i$ is a vector with dimensionality dependent on the number of hyper-parameters in a given model, and $y_i$ is an evaluation metric, (e.g., mean squared error) We subsequently fit a GP surrogate model and consider each of the three acquisition functions from Section~\ref{Sec:AcqAssumptions}. Given the difficulty involved in the graphical depiction of an acquisition function conflict in more than two dimensions, we examine a simple, two-dimensional illustrative example.
From Figure~\ref{Fig:AcqF}, it is apparent that even in the simplest 2D case, many examples of conflicting acquisitions exist. Thus, in higher dimensions this behaviour will also occur. 

\section{Optimising Bayesian Optimisation}\label{Sec:Improve}
In this section we describe the component design choices that may mitigate for heteroscedastic and non-stationary aspects of commonly-encountered BO problems. Input and output transformations as well as multi-objective acquisition functions have been introduced in isolation previously, whilst acquisition function robustness is unique to this work. The overall design choices produce the method which we name Heteroscedastic and Evolutionary Bayesian Optimisation (HEBO).

\subsection{Tackling Heteroscedasticity and Non-Stationarity}\label{Sec:HeteroSol}

To parsimoniously handle heteroscedasticity and non-stationarity, we leverage ideas from the warped GP literature ~\cite{snelson2004warped} where output transformations facilitate the modelling of complex noise processes. We observe that the well-known \texttt{Box-Cox}~\cite{box1964analysis} and \texttt{Yeo-Jonhson}~\cite{yeo2000new} output transformations in conjunction with the \texttt{Kumaraswamy}~\cite{kumaraswamy1980generalized} input transformation, offer a balance between simplicity of implementation and empirical performance. In our ablation study (Section~\ref{Sec:Exp}), we demonstrate that the addition of these two modelling components alone yields large performance gains. Note, \underline{we refit the parameters for the output transformation before we refit the GP} after receiving a new samples. \\

\noindent \textbf{Output Transformation for Heteroscedasticity:} We consider the \texttt{Box-Cox} transformation most frequently used as a corrective mapping for non-Gaussian data. The transform depends on a tuneable parameter $\zeta$ and applies the following map to each of the labels: $\text{T}_{\zeta}(y_l) = \sfrac {{y}_{l}^{\zeta }-1}{\zeta }$ for $\zeta \neq 0$ and $\text{T}_{\zeta}(y_l) = \log y_l$ if $\zeta = 0$, where in our case $y_l$ denotes the validation accuracy of the $l^{th}$ hyper-parameter configuration. $\zeta$ must be fit based on the observed data such that the distribution of the transformed labels closely resembles a Gaussian distribution. This is achieved by minimising the negative \texttt{Box-Cox} likelihood function:

\begin{align*}
\log\left[\sum_{l=1}^n \frac{(\text{T}_{\zeta}(y_l)  - \overline{\text{T}}_{\zeta}(\bm{y}))^2}{n}\right]^{\frac{n}{2}} + \sum_{l=1}^n \log\left[\text{T}_{\zeta}(y_l)\right]^{(1- \zeta)},
\end{align*}
where $n$ is the number of datapoints and $\overline{\text{T}_{\zeta}(\bm{y})}$ is the sample mean of the transformed labels. \texttt{Box-Cox} transforms only consider strictly positive (or strictly negative) labels $y_l$. \\

When labels take on arbitrary values, we use the \texttt{Yeo-Johnson} transform in place of the \texttt{Box-Cox} transform. The \texttt{Yeo-Johnson} transform is defined as follows:
\begin{align*}
        \texttt{Y.J.}_{\zeta} (y_l) = \left\{\begin{array}{lr}
       \frac{(y_l + 1)^{\zeta} - 1}{\zeta}, &  \text{if $\zeta \neq 0$, $y_l \geq 0$}\\
        \log (y_l + 1), &   \text{if $\zeta  = 0$, $y_l \geq 0$}\\
        \frac{(1 - y_l)^{2 - \zeta} - 1}{\zeta - 2} &   \text{if $\zeta  \neq 2$, $y_l < 0$}\\
        - \log (1 - y_l)    & \text{if $\zeta  = 2$, $y_l < 0$.}
        \end{array}\right.
\end{align*}
In an analogous fashion to the \texttt{Box-Cox} transform, the \texttt{Yeo-Johnson}'s parameter is fit based on the observed data through solving the following 1-dimensional optimisation problem: 

\begin{align*}
\max_{\zeta} &-\frac{n}{2} \log \left[\frac{\sum_{j=1}^n (\texttt{Y.J.}_{\zeta}({y}_l) - \overline{\texttt{Y.J.}_{\zeta}(\bm{y})})^2}{n - 1}\right] + (\zeta - 1) \sum_{i=1}^n \left[\text{sign}({y}_l) \log(|{y}_l|+1)\right],
\end{align*}

\noindent with $\overline{\texttt{Y.J.}_{\zeta}(\bm{y})}$ the sample average computed after applying the \texttt{Yeo-Johnson} transformation. \\



\noindent \textbf{Input Transformations for Non-Stationarity:} As a general solution concept for correcting for non-stationarity, we consider input warping see~\cite{snoek2012practical}. Input warping performs a (usually non-linear and learnable) transformation to the input variables $(\bm{x}_{l})$. It was proven in~\cite{snoek2012practical} that Input warping also helps tackle non-stationary functions. We rely on the \texttt{Kumaraswamy} input warping transform as used in~\cite{snoek2012practical}, which operates as follows for each input dimension:

\begin{equation*}
    [\texttt{Kumaraswamy}_{\bm{\gamma}} (\bm{x}_{l})]_{k} = 1 - \left(1 - [\bm{x}_{l}]_{k}^{a_k}\right)^{b_{k}} \ \forall k \in [1:d],
\end{equation*}

where $d$ is the dimensionality of the decision variable (i.e. the number of free hyper-parameters), $a_{k}$ and $b_{k}$ are tuneable warping parameters for each of the dimensions, and $\bm{\gamma}$ is a vector concatenating all free parameters, i.e., $\bm{\gamma}= [a_{1:d}, b_{1:d}]^{\mathsf{T}}$. $\bm{\gamma}$ is fit based on the observed data. Similar to~\cite{balandat2020botorch}, we optimise $\bm{\gamma}$ under the marginal likelihood objective used to fit the GP surrogate. \\

\noindent {\textbf{All Modelling Improvements Together:}} Combining the above corrective measures for heteroscedasticity and non-stationarity leads us to an improved GP surrogate with more flexible modelling capabilities. The implementation of such a model is relatively simple and involves maximising a new marginal likelihood which may be written as:
\begin{align*}
    \max_{\bm{\theta}, \bm{\gamma}} &- \frac{1}{2} \text{T}_{\zeta^{\star}}(\bm{y})^{\mathsf{T}}(\bm{K}^{\bm{\gamma}}_{\bm{\theta}} + \sigma_{\text{noise}}\bm{I})^{-1}\text{T}_{\zeta^{\star}}(\bm{y}) - \frac{1}{2}|\bm{K}^{\bm{\gamma}}_{\bm{\theta}} + \sigma_{\text{noise}}^{2} \bm{I}| - \text{const,}
\end{align*}

\noindent where $\bm{\theta}$ are GP hyper-parameters, $\bm{\gamma}$ indicates the use of non-stationary transformations, and $\zeta^{\star}$ denotes the solution to a $\texttt{Box-Cox}$ likelihood objective. It is worth noting that we use $\texttt{Box-Cox}$ as a running example but as mentioned previously we interchange $\texttt{Box-Cox}$ with $\texttt{Yeo-Johnson}$ transforms based on the properties of the label $y_l$. We use $\bm{K}^{\bm{\gamma}}_{\bm{\theta}} \in \mathbb{R}^{n \times n}$ to represent a matrix such that each entry depends on both $\bm{\theta}$ and $\bm{\gamma}$, where $k_{\bm{\theta}}^{\gamma} (\bm{x}, \bm{x}^{\prime}) = k_{\bm{\theta}}(\texttt{Kumaraswamy}_{\bm{\gamma}} (\bm{x}), \texttt{Kumaraswamy}_{\bm{\gamma}} (\bm{x}^{\prime}))$.  

\subsection{Tackling Acquisition Conflict \& Robustness}
Having proposed modifications to the surrogate model component of the Bayesian optimisation scheme, we now turn our attention to the acquisition maximisation step. In particular, we focus on two considerations, the first related to the assumption of a perfect GP surrogate, and the second centred on conflicting acquisitions. 

\subsubsection{A Robust Acquisition Objective}\label{Sec:Robust}
As mentioned in Section~\ref{Sec:AcqAssumptions}, the acquisition maximisation step assumes that an adequate surrogate model is readily available. During early rounds of training especially, where data is scarce, such a property is often violated, leading to potentially severe model misspecification. One way to tackle such model misspecification is to adopt a robust formulation~\cite{kirschner2020distributionally,klein2017robo} which attempts to identify the best-performing query location under the worst-case GP model, i.e., solving $\max_{\bm{x}} \min_{\bm{\theta}} \alpha^{\bm{\theta}} (\bm{x}|\mathcal{D})$. Though such a formulation admits a solution $\bm{x}^{\star}$ that is robust to worst-case misspecification in $\bm{\theta}$, having a $\max \min$ acquisition is problematic for several reasons. From a conceptual perspective $\max \min$ formulations are known to lead to very conservative solutions if not correctly constrained or regularised since the optimiser possesses the power to impair the GP fit while updating  $\bm{\theta}$\footnote{One can make a case for augmenting the objective with a constraint such that updates for $\bm{\theta}$ remain close to $\bm{\theta}^{\star}$ of the marginal likelihood. The ideal enforced proximity value however remains unclear in the robust acquisition literature to date~\cite{WRL,kirschner2020distributionally}.}. From the perspective of implementation, one encounters two further issues. First, no global convergence guarantees are known for the non-convex, non-concave case~\cite{MJ}, and second, ensuring gradients can propagate through the computation graph restricts surrogates and acquisition functions to be within the same programming framework. 

To avoid worst-case solutions and engender independence between acquisition functions and surrogate models, given a set of parameters from a trained GP $\bm{\theta}$, we leverage ideas from domain randomisation~\cite{DR} and consider an expected formulation instead over these parameters: $
   \max_{\bm{x}} \alpha^{\bm{\theta}}_{\text{rob.}}(\bm{x}|\mathcal{D}) \equiv \max_{\bm{x}} \mathbb{E}_{\epsilon \sim \mathcal{N}(\bm{0}, \sigma_{\epsilon}^{2}\bm{I})}\left[\alpha^{\bm{\theta}+\epsilon}(\bm{x}|\mathcal{D})\right]$. 
Importantly, this problem seeks to find new query locations that perform well on average over a distribution of surrogate models in favour of assuming a perfect surrogate. Despite on an intractable nature of  $\alpha^{\bm{\theta}}_{\text{rob.}}(\cdot|\mathcal{D})$, in \text{HEBO} we show (the rigorous representation of this result is presented in Appendix~\ref{app:robacq}) that it can be approximated with any arbitrary precision and high confidence with $\overline{\alpha}^{\bm{\theta}} (\bm{x}|\mathcal{D}) = \alpha^{\bm{\theta}} (\bm{x}|\mathcal{D}) + \mathcal{N}(0,\sigma^2_{n})$ by properly choosing  parameters $\sigma_{\epsilon}$ and $\sigma_n$:
\begin{theorem}(\text{ Informal })
Let us consider the  stochastic version of the acquisition function utilised in HEBO and given by $\overline{\alpha}^{\bm{\theta}} (\bm{x}|\mathcal{D}) = \alpha^{\bm{\theta}} (\bm{x}|\mathcal{D}) + \mathcal{N}(0,\sigma^2_{n})$ and Let $\alpha^{\bm{\theta}}_{\text{rob.}}(\bm{x}|\mathcal{D}) \equiv \mathbb{E}_{\epsilon \sim \mathcal{N}(\bm{0}, \sigma_{\epsilon}^{2}\bm{I})}\left[\alpha^{\bm{\theta}+\epsilon}(\bm{x}|\mathcal{D})\right]$ be the robust form of the standard acquisition function given as expectation over random perturbation of parameter $\bm{\theta}$. Then,  with proper choice of parameters $\sigma_n$ and  $\sigma_{\epsilon}$, \text{HEBO} acquisition function $\overline{\alpha}^{\bm{\theta}} (\bm{x}|\mathcal{D}) $ accurately approximates the robust acquisition function $\alpha^{\bm{\theta}}_{\text{rob.}}(\bm{x}|\mathcal{D})$ with high probability \footnote{Here we use the common approach for proving stochastic expressions with high probability (see \cite{jordan_cubic}, \cite{AZ01}). Specifically, we show that for any confidence parameter   $\delta\in(0,1)$ the stochastic expression under consideration is valid with probability at least $1 - \delta$.} and for any $\bm{\theta},\bm{x}$. 
\end{theorem}

\subsubsection{Multi-Objective Acquisition functions}

As a final component of our general framework, we propose the use of multi-objective acquisitions seeking a Pareto-front solution. This formulation facilitates the process of ``hedging'' between different acquisitions such that no single acquisition dominates the solution~\cite{lyu2018batch}. Formally, we solve
\begin{equation}
\label{Eq:MOO}
    \max_{\bm{x}} \left(\overline{\alpha}^{\bm{\theta}}_{\text{EI}}(\bm{x}|\mathcal{D}), \overline{\alpha}^{\bm{\theta}}_{\text{PI}}(\bm{x}|\mathcal{D}), \overline{\alpha}^{\bm{\theta}}_{\text{UCB}}(\bm{x}|\mathcal{D}) \right),
\end{equation}
where $\overline{\alpha}^{\bm{\theta}}_{\text{type}}(\bm{x}|\mathcal{D})$ is a robust acquisition of $\text{type} \in \{ \text{EI}, \text{PI}, \text{UCB}\}$ as introduced in the previous section. We also note that our formulation is designed to admit the use of a robust objective value of $\overline{\alpha}^{\bm{\theta}} (\bm{x}|\mathcal{D}) = \alpha^{\bm{\theta}} (\bm{x}|\mathcal{D}) + \eta_{n}$ with $\eta_{n}$ being a sample from $\mathcal{N}(0,\sigma_n)$ at each iteration of the evolutionary solver. 
    
Although solving the problem in Equation~\ref{Eq:MOO} is a formidable challenge, we note the existence of many mature multi-objective optimisation algorithms. These range from first-order~\cite{2014_Kingma} to zero-order~\cite{loshchilov2016cma,2020_Gabillon} and evolutionary methods~\cite{2016_Hansen,deb2002fast}. Due to the discrete nature of hyper-parameters in machine learning tasks, we advocate the use of evolutionary solvers that naturally handle categorical and integer-valued variables. In our experiments, we employ the non-dominated sorting genetic algorithm II (\texttt{NSGA-II}) which allows for mixed variable crossover and mutation to optimise real-valued and integer-valued inputs~\cite{deb2002fast}. We use the implementation of \texttt{NSGA-II} found in the \texttt{Pymoo}~\cite{pymoo} library. Alternatively, one may use the GP Hedge acquisition as used in \texttt{Dragonfly}~\cite{JMLR:v21:18-223} in \cite{hoffman2011portfolio} or in \texttt{SkOpt} to select between acquisitions. We however, observed this formulation to perform poorly when compared against individual acquisitions. 

\section{Experiments and Results 
}\label{Sec:Exp}

   \begin{figure}[t]
    \centering
    \includegraphics[width=1.0\textwidth]{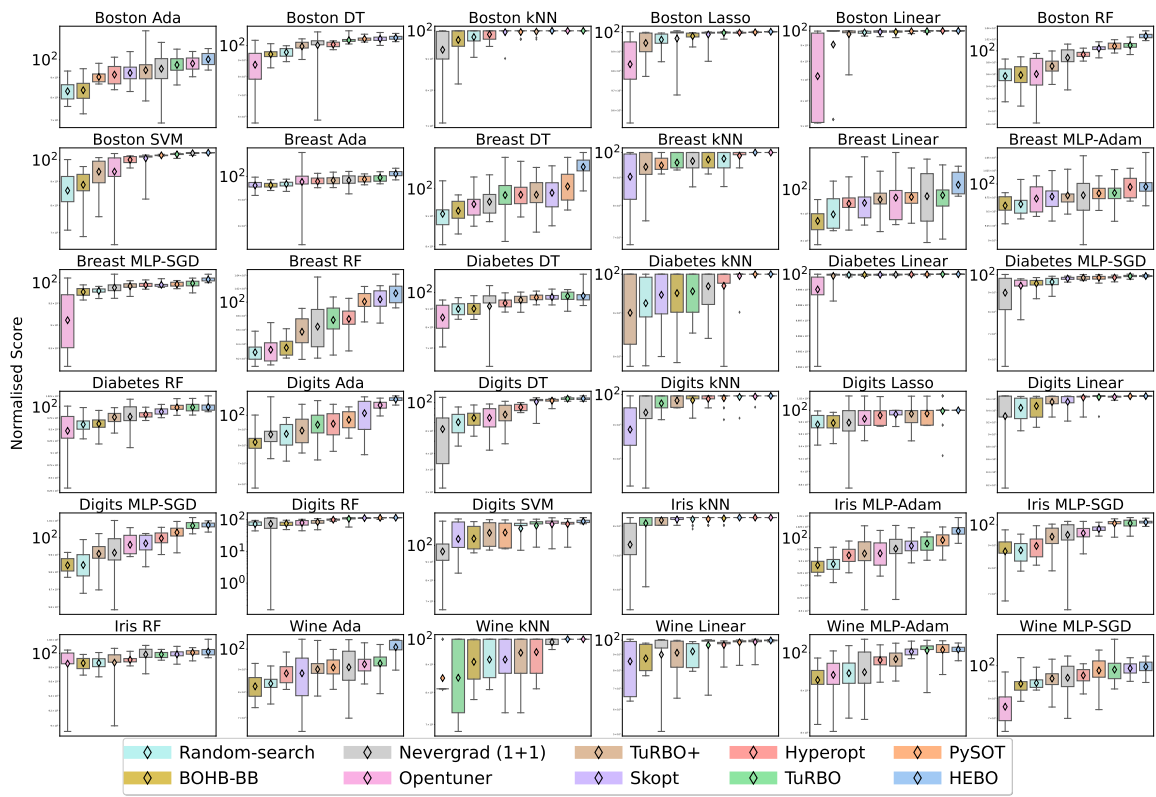}
    \caption{\texttt{HEBO} compared against all baselines for 16 iterations and a batch size of 8 query points per iteration. Each experiment is repeated with 20 random seeds. We average each seed over both metrics in all tasks and display a subset of 36 summary plots for the 108 black-box functions. \texttt{HEBO} achieves the highest normalised mean score in $68.5\%$ of the 108 black-box functions. Full results for the 108 tasks are presented in Appendix~\ref{fig:summary_all_models} in tabular format.}
    \label{fig:best_models_datasets}
    \end{figure}

In this section, we continue our empirical evaluation and validate gains (if any) that arise from the improvements proposed in Section~\ref{Sec:Improve}. The experimental setup remains as described in Section~\ref{Sec:Answers}. To assess performance, we use the normalised task score\footnote{Note, we don't report the time to compute query points per algorithm as this was under 20 seconds per query point batch.}. We run experiments on either 16 iterations with a batch of 8 query points per iteration or 100 iterations with 1 query point. Each experiment is repeated for 20 random seeds. We baseline against a wide range of solvers that either rely on BO-strategies or follow zero-order techniques such as differential evolution or particle swarms. These include \texttt{SkOpt}~\cite{scikit-learn}~\footnote{\href{https://github.com/scikit-optimize/scikit-optimize}{https://github.com/scikit-optimize/scikit-optimize}} \texttt{pySOT}~\footnote{\href{https://github.com/dme65/pySOT}{https://github.com/dme65/pySOT}} a parallel global optimisation package~\cite{eriksson2019pysot}, \texttt{HyperOpt}~\cite{bergstra2013hyperopt}~\footnote{\href{https://github.com/hyperopt/hyperopt}{https://github.com/hyperopt/hyperopt}}, \texttt{OpenTuner}~\footnote{\href{https://github.com/jansel/opentuner}{https://github.com/jansel/opentuner} } a package for ensembling methods~\cite{ansel2014opentuner}, \texttt{NeverGrad}~\cite{rapin2018nevergrad}~\footnote{\href{https://github.com/facebookresearch/nevergrad}{https://github.com/facebookresearch/nevergrad}} a gradient-free optimisation toolbox (where we use the One Plus One optimiser with the associated label \texttt{NeverGrad (1+1)}), \texttt{BOHB}~\cite{falkner2018bohb}~\footnote{ \href{https://github.com/automl/HpBandSter}{https://github.com/automl/HpBandSter}} and \texttt{Dragonfly}~\cite{JMLR:v21:18-223}~\footnote{  \href{https://github.com/dragonfly/dragonfly}{https://github.com/dragonfly/dragonfly}}. Additionally, we carried our modelling improvements to \texttt{TuRBO}~\footnote{ \href{https://github.com/rdturnermtl/bbo_challenge_starter_kit/}{https://github.com/rdturnermtl/bbo\_challenge\_starter\_kit/} \label{footnote:bbomodels}} \cite{2019_turbo}, augmenting the standard GP with mitigation strategies from Section~\ref{Sec:Improve} producing a new baseline that we entitle \texttt{TuRBO+}. Finally, we introduce Heteroscedastic Evolutionary Bayesian Optimisation (\texttt{HEBO}), in which we construct an optimiser with the improvements introduced in Section~\ref{Sec:Improve}. \\

\noindent \textbf{Implementation Details for \texttt{BOHB}:} \texttt{BOHB} is a scalable hyper-parameter tuning algorithm introduced in \cite{falkner2018bohb} mixing bandits and BO approaches to achieve both competitive anytime and final performance. Contrary to the other solvers considered in this paper, \texttt{BOHB} is specifically designed to tackle multi-fidelity optimisation and uses the Hyperband~\cite{li2017hyperband} routine to define the fidelity levels under which points are asynchronously evaluated. The selection of points follows a BO strategy based on the Tree Parzen Estimator (TPE) 
method. Given a data set $\mathcal{D}$ of observed data points and a threshold $\alpha\in \mathbb{R}$, the TPE models $p(\bm{x}|y)$, using kernel density estimates of 
\begin{align*}
    \ell(\bm{x}) &= p(y < \alpha|\bm{x}, \mathcal{D}) \\
    g(\bm{x}) &= p(y \geq \alpha | \bm{x},\mathcal{D}). 
\end{align*}
 In the TPE algorithm, maximising the expected improvement criterion
 \begin{equation*}
     \alpha_\text{EI}(\boldsymbol{x}) = \int \max(0, \alpha - p(y|\boldsymbol{x}))  p(y|\boldsymbol{x}) d y
 \end{equation*}  
 is equivalent to maximising the ratio   $r(x) = \tfrac{\ell(\bm{x})}{g(\bm{x})}$  
 which is carried out to select a single new candidate point at a time.

In the absence of a multi-fidelity setup in our experiments, we run a modified version of the \texttt{BOHB} algorithm implemented in the \texttt{HpBandSter} package. We leave the TPE method for modelling unchanged but ignore the fidelity level assignment from Hyperband. Moreover, as our experimental setup involves batch acquisitions, we tested two alternatives to the standard BOHB acquisition procedure to support synchronous suggestion of multiple points. In the first approach, we run $q$ independent maximisation processes of $r(\bm{x})$ from random starting points and recover a single candidate from each process to form the $q$-batch suggestion. In the second approach, we obtain one point as a result of a single maximisation of $r(\bm{x})$ and we sample $q-1$ random points to complete the $q$-batch suggestion. As the latter method yields better overall performance, the results reported under the \texttt{BOHB-BB} label are obtained using the second approach.

\subsection{Black-Box Functions} 
As discussed in Section 3, we evaluate black-box optimisation solvers on a large set of tasks from the \texttt{Bayesmark} package. Each task involves optimising the hyper-parameters of a machine learning algorithm to minimise the cross validation loss incurred when this model is applied in a regression (reg) or a classification (clf) setting for a given data set. 
 Thus, a task is characterised by a model, a data set and a loss function (metric) quantifying the quality of the regression or classification performance. In total, $108$ distinct tasks can be defined from the valid combinations of the nine models specified in Table~\ref{tab:search-space}, the following six real-world UCI datasets \cite{2019_Dua}, Boston (reg), Breast Cancer (clf), Diabetes (reg), Digits (clf), Iris (clf) and Wine (clf); the following two regression metrics, negative mean-squared error (MSE), negative mean absolute error (MAE), and two classification metrics, negative log-likelihood (NLL) and negative accuracy (ACC).
 The results reported in Figures 3 and 4 have been obtained by applying each black-box optimisation method using $16$ iterations of $8$-batch acquisition steps on all of the $108$ tasks. In order to provide a reliable evaluation of the different solvers, we repeated each run with $20$ random seeds and considered the normalised score given by:
 
    \begin{equation}
    \label{eq:score}
        \textbf{Normalised Score} = 100 \times \frac{\mathcal{L} - \mathcal{L}^*}{\mathcal{L}^{\text{rand}} - \mathcal{L}^*}
    \end{equation}
    
where $\mathcal{L}$ is the best-achieved cross validation loss at the end of the $16$ acquisition steps, $\mathcal{L}^*$ is the estimated optimal loss for the task and $\mathcal{L}^{\text{rand}}$ is the mean loss (across multiple runs) obtained using random search with the same number of acquisition steps. The normalisation procedure permits aggregation of the scores across tasks despite the different cross-validation loss functions used. 

\subsection{Black-Box Optimisation Input Variables} %
    
    We provide in Table~\ref{tab:search-space} and Table \ref{tab:search-space-reg} the list of the hyper-parameters controlling the behaviour of each model along with their optimisation domains, which can differ depending on whether the model is used for a classification or a regression task. The search domain may include a mix of continuous and integer-valued variables (e.g. the MLP-SGD hyper-parameter set includes an integer-valued hidden layer size, and a continuous-valued initial learning rate that can take on values between $10^{-5}$ and $10^{-1}$). The dimensionality of the input space, i.e. the number of hyper-parameters to tune, ranges from $2$ to $9$. We specify in the final column of the tables whether the search domain is modified through a standard transformation ( $\log$ or $\operatorname{logit}$) in order to facilitate optimisation. 
    
\begin{table}[t!]
\centering
\caption{Search spaces for hyper-parameter tuning on classification tasks. We specify the variable type of each hyper-parameter (with $\mathbb{R}$ for real-valued and $\mathbb{Z}$ for integer- valued) as well as the search domain. We specify $\log-\mathcal{U}$ (resp. $\text{logit}-\mathcal{U}$) to indicate that a $\log$ (resp. $\text{logit}$) transformation is applied to the optimisation domain.}
\label{tab:search-space}
\begin{tabular}{llcl}
\toprule
Model & Parameter  &  Type &  Domain \\
\midrule
\textbf{kNN} &  n\_neighbors & $\mathbb{Z}$ & $\mathcal{U}(1, 25)$ \\ 
\ & p & $\mathbb{Z}$ & $\mathcal{U}(1, 4)$ \\ 
\textbf{Support Vector Machine} &  C & $\mathbb{R}$ & $\log-\mathcal{U}(1, 10^3)$ \\ 
\ & gamma & $\mathbb{R}$ & $\log-\mathcal{U}(10^{-4}, 10^{-3})$ \\ 
\ & tol & $\mathbb{R}$ & $\log-\mathcal{U}(10^{-5}, 10^{-1})$ \\ 
\textbf{Decision Tree} &  max\_depth & $\mathbb{Z}$ & $\mathcal{U}(1, 15)$ \\ 
\ & min\_samples\_split & $\mathbb{R}$ & $\text{logit}-\mathcal{U}(0.01, 0.99)$ \\ 
\ & min\_samples\_leaf & $\mathbb{R}$ & $\text{logit}-\mathcal{U}(0.01, 0.49)$ \\ 
\ & min\_weight\_fraction\_leaf & $\mathbb{R}$ & $\text{logit}-\mathcal{U}(0.01, 0.49)$ \\ 
\ & max\_features & $\mathbb{R}$ & $\text{logit}-\mathcal{U}(0.01, 0.99)$ \\ 
\ & min\_impurity\_decrease & $\mathbb{R}$ & $\mathcal{U}(0, 0.5)$ \\ 
\textbf{Random Forest} &  max\_depth & $\mathbb{Z}$ & $\mathcal{U}(1, 15)$ \\ 
\ & max\_features & $\mathbb{R}$ & $\text{logit}-\mathcal{U}(0.01, 0.99)$ \\ 
\ & min\_samples\_split & $\mathbb{R}$ & $\text{logit}-\mathcal{U}(0.01, 0.99)$ \\ 
\ & min\_samples\_leaf & $\mathbb{R}$ & $\text{logit}-\mathcal{U}(0.01, 0.49)$ \\ 
\ & min\_weight\_fraction\_leaf & $\mathbb{R}$ & $\text{logit}-\mathcal{U}(0.01, 0.49)$ \\ 
\ & min\_impurity\_decrease & $\mathbb{R}$ & $\mathcal{U}(0, 0.5)$ \\ 
\textbf{MLP-Adam} &  hidden\_layer\_sizes & $\mathbb{Z}$ & $\mathcal{U}(50, 200)$ \\ 
\ & alpha & $\mathbb{R}$ & $\log-\mathcal{U}(10^{-5}, 10^{1})$ \\ 
\ & batch\_size & $\mathbb{Z}$ & $\mathcal{U}(10, 250)$ \\ 
\ & learning\_rate\_init & $\mathbb{R}$ & $\log-\mathcal{U}(10^{-5}, 10^{-1})$ \\ 
\ & tol & $\mathbb{R}$ & $\log-\mathcal{U}(10^{-5}, 10^{-1})$ \\ 
\ & validation\_fraction & $\mathbb{R}$ & $\text{logit}-\mathcal{U}(0.1, 0.9)$ \\ 
\ & beta\_1 & $\mathbb{R}$ & $\text{logit}-\mathcal{U}(0.5, 0.99)$ \\ 
\ & beta\_2 & $\mathbb{R}$ & $\text{logit}-\mathcal{U}(0.9, 1 - 10^{-6})$ \\ 
\ & epsilon & $\mathbb{R}$ & $\log-\mathcal{U}(10^{-9}, 10^{-6})$ \\ 
\textbf{MLP-SGD} &  hidden\_layer\_sizes & $\mathbb{Z}$ & $\mathcal{U}(50, 200)$ \\ 
\ & alpha & $\mathbb{R}$ & $\log-\mathcal{U}(10^{-5}, 10^{1})$ \\ 
\ & batch\_size & $\mathbb{Z}$ & $\mathcal{U}(10, 250)$ \\ 
\ & learning\_rate\_init & $\mathbb{R}$ & $\log-\mathcal{U}(10^{-5}, 10^{-1})$ \\ 
\ & power\_t & $\mathbb{R}$ & $\text{logit}-\mathcal{U}(0.1, 0.9)$ \\ 
\ & tol & $\mathbb{R}$ & $\log-\mathcal{U}(10^{-5}, 10^{-1})$ \\ 
\ & momentum & $\mathbb{R}$ & $\text{logit}-\mathcal{U}(0.001, 0.999)$ \\ 
\ & validation\_fraction & $\mathbb{R}$ & $\text{logit}-\mathcal{U}(0.1, 0.9)$ \\ 
\textbf{AdaBoost} &  n\_estimators & $\mathbb{Z}$ & $\mathcal{U}(10, 100)$ \\ 
\ & learning\_rate & $\mathbb{R}$ & $\log-\mathcal{U}(10^{-4}, 10^{1})$ \\ 
\textbf{Lasso} &  C & $\mathbb{R}$ & $\log-\mathcal{U}(10^{-2}, 10^{2})$ \\ 
\ & intercept\_scaling & $\mathbb{R}$ & $\log-\mathcal{U}(10^{-2}, 10^{2})$ \\ 
\textbf{Linear} &  C & $\mathbb{R}$ & $\log-\mathcal{U}(10^{-2}, 10^{2})$ \\ 
\ & intercept\_scaling & $\mathbb{R}$ & $\log-\mathcal{U}(10^{-2}, 10^{2})$ \\ 
\bottomrule
\end{tabular}
\end{table}

\begin{table}[t!]
\centering
    \caption{Models and search spaces for hyper-parameter tuning on regression tasks. Models having the same search spaces for classification and regression tasks are omitted (cf. Table~\ref{tab:search-space}).} 
\label{tab:search-space-reg}
\begin{tabular}{llcl}
\toprule
Model & Parameter  &  Type &  Domain \\
\midrule
\textbf{AdaBoost} &  n\_estimators & $\mathbb{Z}$ & $\mathcal{U}(10, 100)$ \\ 
\ & learning\_rate & $\mathbb{R}$ & $\log-\mathcal{U}(10^{-4}, 10^{1})$ \\ 
\textbf{Lasso} &  alpha & $\mathbb{R}$ & $\log-\mathcal{U}(10^{-2}, 10^{2})$ \\ 
\ & fit\_intercept & $\mathbb{Z}$ & $\mathcal{U}(0, 1)$ \\ 
\ & normalize & $\mathbb{Z}$ & $\mathcal{U}(0, 1)$ \\ 
\ & max\_iter & $\mathbb{Z}$ & $\log-\mathcal{U}(10, 5000)$ \\ 
\ & tol & $\mathbb{R}$ & $\log-\mathcal{U}(10^{-5}, 10^{-1})$ \\ 
\ & positive & $\mathbb{Z}$ & $\mathcal{U}(0, 1)$ \\ 
\textbf{Linear} &  alpha & $\mathbb{R}$ & $\log-\mathcal{U}(10^{-2}, 10^{2})$ \\ 
\ & fit\_intercept & $\mathbb{Z}$ & $\mathcal{U}(0, 1)$ \\ 
\ & normalize & $\mathbb{Z}$ & $\mathcal{U}(0, 1)$ \\ 
\ & max\_iter & $\mathbb{Z}$ & $\log-\mathcal{U}(10, 5000)$ \\ 
\ & tol & $\mathbb{R}$ & $\log-\mathcal{U}(10^{-4}, 10^{-1})$ \\ 
\bottomrule
\end{tabular}
\end{table}

Table~\ref{tab:summary-perf-compare} synthesises the performance achieved on the $108$ tasks by the black-box optimisation solvers considered in our experiments. We note that the distribution of the scores attained by \texttt{HEBO} has the largest mean and the smallest standard deviation, indicating that \texttt{HEBO} significantly outperforms competitor algorithms. 

    \begin{table}[t!]
    \centering
\begin{tabular}{lrrrrrrrr}
\toprule
      Algorithm &    Mean &    Std &     Median &     40$^{th}$ Centile &    30$^{th}$ Centile &    20$^{th}$ Centile &  5$^{th}$ Centile \\
\midrule
          \texttt{HEBO} &  $\bm{100.12}$ &   $\bm{8.70}$ &  $\bm{100.01}$ &  $\bm{100.00}$ &  $\bm{99.88}$ &  $\bm{98.64}$ &  $\bm{85.71}$ \\
          \texttt{PySOt} &   98.18 &   9.03 &  100.00 &   99.81 &  98.60 &  95.36 &  80.00 \\
          \texttt{TuRBO} &   97.95 &  10.80 &  100.00 &   99.88 &  98.75 &  95.26 &  78.63 \\
       \texttt{HyperOpt} &   96.37 &   8.79 &   99.31 &   98.16 &  95.94 &  92.38 &   78.52 \\
          \texttt{SkOpt} &   96.18 &  11.51 &   99.78 &   98.66 &  96.73 &  91.62 &  74.77 \\
\texttt{TuRBO+} &   95.29 &  10.93 &   98.97 &   97.60 &  95.27 &  90.92  &  74.77 \\
      \texttt{OpenTuner} &   94.32 &  14.18 &   98.44 &   96.93 &  93.84 &  89.97  &  68.96 \\
\texttt{Nevergrad (1+1)} &   93.20 &  17.52 &   99.65 &   97.84 &  94.57 &  88.28 &  55.34 \\
\texttt{BOHB} &   92.03 &  11.16  &  96.02 &  93.55 &  90.14  & 85.71   & 67.82 \\
  \texttt{Random-Search} &   92.00 &  11.71 &   96.18 &   93.55 &  90.05 &  85.16  &  69.55 \\
\bottomrule
\end{tabular}
    \caption{Mean and n-th percentile normalised scores over $108$ black-box functions, each repeated with 20 random seeds. We observe significant mean improvements from \texttt{HEBO} compared to all competitor algorithms.
}
    \label{tab:summary-perf-compare}
\end{table}


\begin{figure*}%
\centering
\subfloat[Empirical Performance Gain]{%
\includegraphics[width=0.8\textwidth]{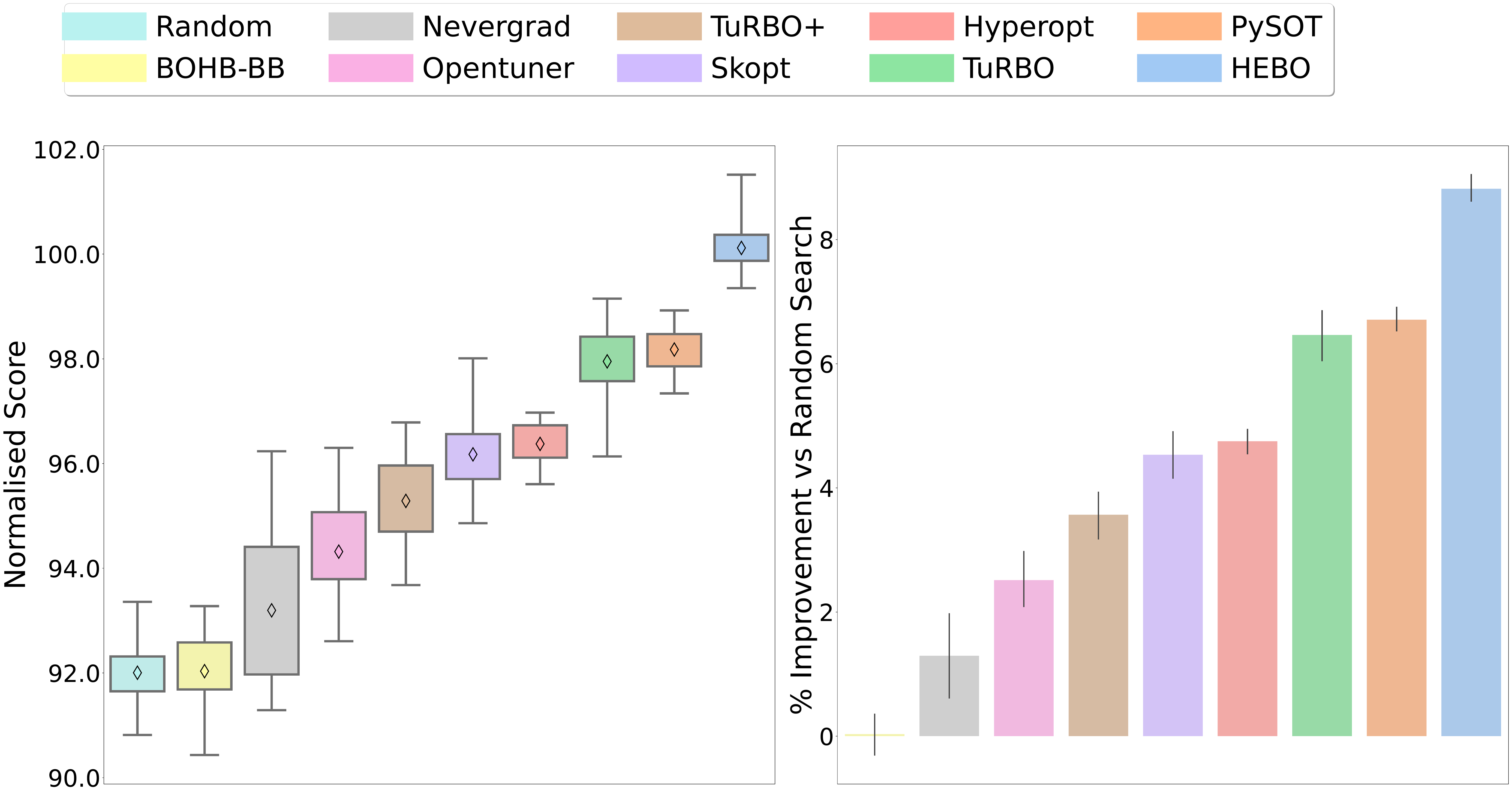}}%
\caption{Analysis of the results on 108 tuning tasks. (Left) Normalised score comparison demonstrating that \texttt{HEBO} (i.e., BO with improvements from Section~\ref{Sec:Improve}) outperforms competitor algorithms. We observe a 5\% relative improvement to SOTA optimisers such as TuRBO. (Right) \texttt{HEBO} yields an 8\% improvement compared to random search.}
\label{Fig:ResOne1}
\end{figure*}

Figure~\ref{fig:best_models_datasets} and Figure~\ref{Fig:ResOne1} demonstrates gains from adopting the general \texttt{HEBO} framework. We note that due to optimising over numerous regression and classification metrics, we show that irrelevant of the validation score HEBO performs better than other optimisers. In Figure~~\ref{Fig:ResOne2}, we compare \texttt{HEBO} against baselines and report up to an $8\%$ performance gain relative to a random search strategy. It is also worth noting that \texttt{TuRBO+} tends to underperform~\footnote{We believe this due to the trust region not being modelled correctly with input warping.}, achieving ca. $4 \%$ improvement relative to random search. We believe such a result is related to the interplay between our approach's capabilities to address heteroscedasticity and non-stationarity as well as the size of the trust regions; an interesting avenue that we plan to explore in future work, as well as experimenting with deeper neural networks as well as other architectures such as convolutional/ recurrent neural networks. Overall, $\texttt{HEBO}$ achieves the highest normalised mean scores on 74 of the 108 datasets. Complete results on all tasks may be found in Appendix~\ref{fig:summary_all_models}. \\

\textbf{Comparison to Asynchronous BO Algorithms:} We perform a comparison to black-box optimisers, such as \texttt{Dragonfly} and \texttt{BOHB}, which operate in the asynchronous setting. We run each method for 100 iterations of data collection with a single query location per iteration. We label the asynchronous algorithms without their multi-fidelity components with an addition BB for black-box optimiser (\texttt{Dragonfly-BB} and \texttt{BOHB-BB}) to assess black-box optimisation performance only. The results of Figure~\autoref{singlebatch} show that in the asynchronous setting, both \texttt{Dragonfly-BB} and \texttt{BOHB-BB} under-perform relative to other black-box optimisers, with \texttt{HEBO} performing best. However, this result is not surprising as asynchronous methods trade off sample efficiency with speed. Nevertheless, this experiment reveals a large gap in suggestion power between SOTA asynchronous and synchronous methods.

\begin{figure}[t!]%
\centering
\subfloat[Competitor Comparison.]{%
\includegraphics[width=0.3\linewidth]{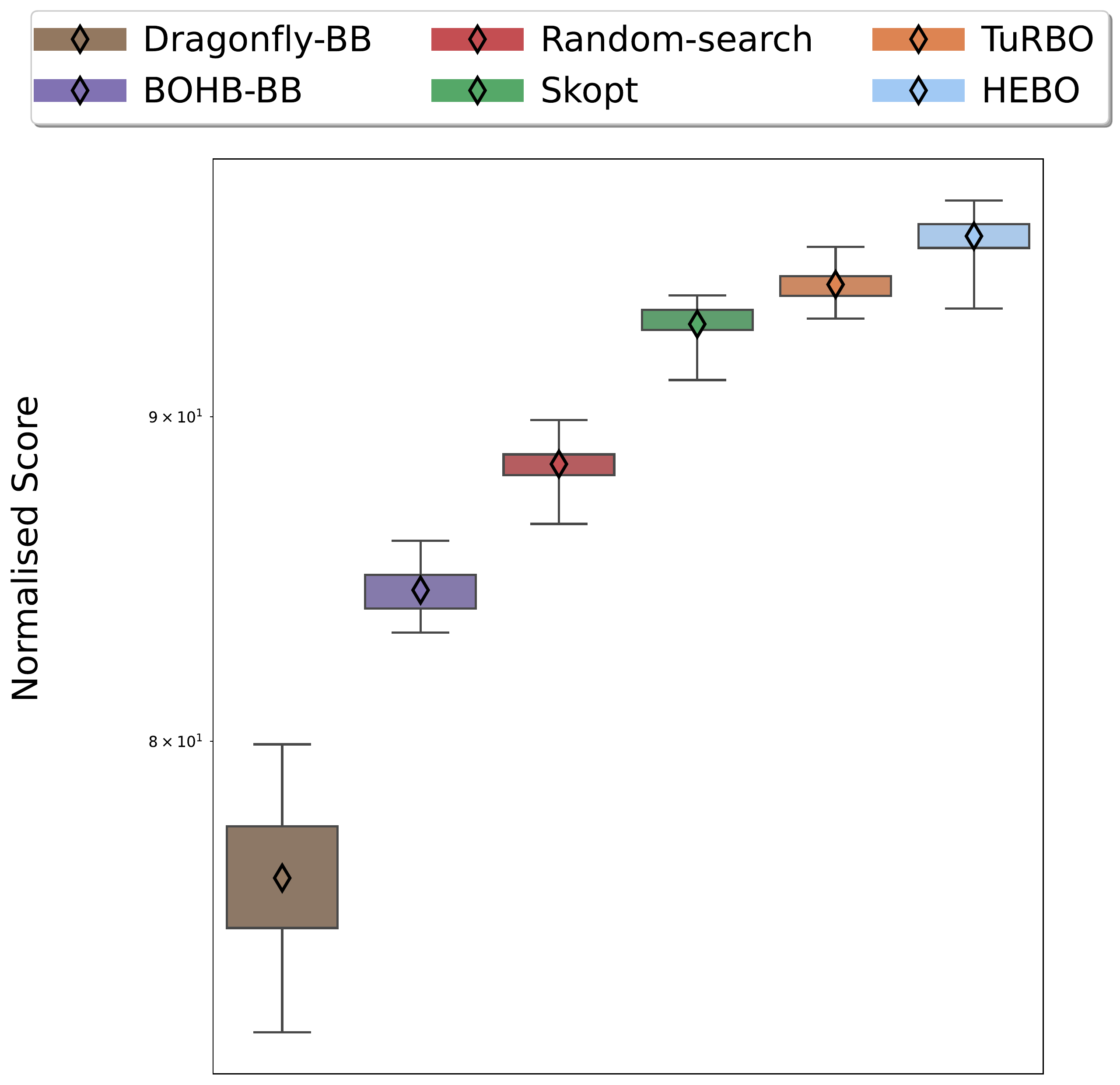}\label{singlebatch}}
\qquad
\subfloat[Ablation Study]{%
\includegraphics[width=0.6\linewidth]{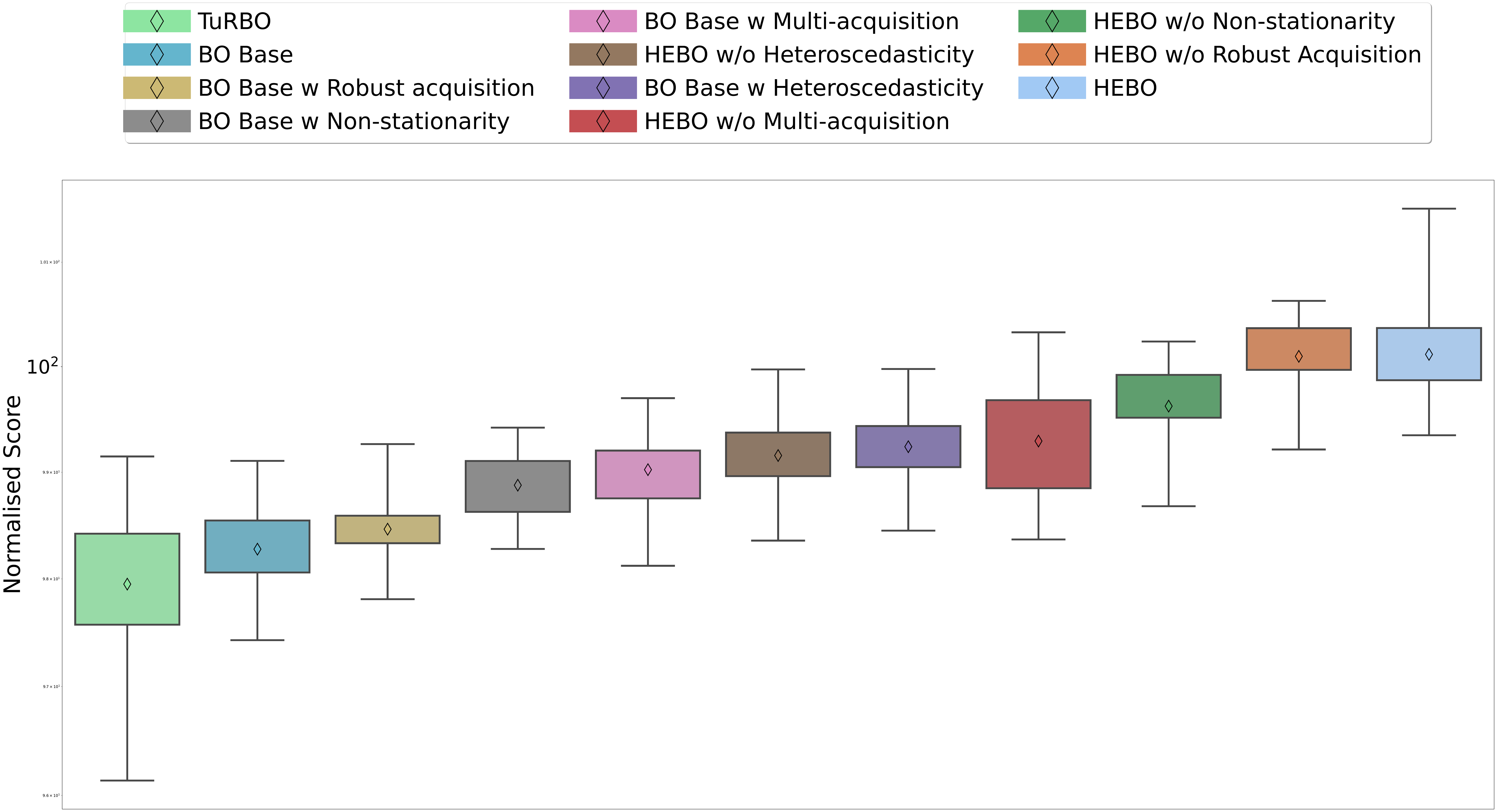}\label{ablation}}
\caption{(a) We compare \texttt{HEBO} against several popular hyper-parameter tuning approaches including \texttt{BOHB-BB} and \texttt{Dragonfly-BB}, running all methods for 100 iterations with a batch size of 1 (i.e. one set of hyper-parameters queried per iteration). \texttt{BOHB-BB} and \texttt{Dragonfly-BB} feature asynchronous queries, suggesting a batch of one set of hyper-parameters at each iteration. We remove the multi-fidelity components from \texttt{BOHB} and \texttt{Dragonfly} to assess Black-Box optimisation alone, hence the additional BB appended to their label. (b) Ablation study where X denotes a general component of \texttt{HEBO}. \texttt{HEBO} w/o X takes one component X out at a time and BO Base w X adds one component X in at a time. We show \texttt{TuRBO} as a baseline and refer to \texttt{HEBO} with all significant components removed as \texttt{BO Base}. The ablation demonstrates that the corrections for each misspecified modelling assumption yield a tangible gain in empirical performance.}

\label{Fig:ResOne2}
\end{figure}

\subsection{Ablation Results} To better understand the relative importance of each component of the \texttt{HEBO} algorithm, we conduct an ablation study by first removing each component of \texttt{HEBO} and testing the remaining components and second, by starting with basic BO and sequentially adding and testing each component of \texttt{HEBO}. The components comprise the consideration of heteroscedasticity, non-stationarity and robustness, as well as the use of a multiobjective acquisition function. We report average normalised scores in Figure~\autoref{ablation}. The precedence order observed is: heteroscedasticity, multi-objective acquisition functions, non-stationarity and robustness. 
\section{Related Work}
We introduce work on the following topics relating to modelling, acquisition and optimisers in Bayesian optimisation: \\

\noindent{\textbf{Heteroscedasticity with output transforms:}} Among various approaches to handling heteroscedasticity~\cite{2007_Kersting,2011_Lazaro,2013_Kuindersma,2017_Calandra,2021_Griffiths}, transforming the output variables is a straightforward option giving rise to warped Gaussian processes \cite{snelson2004warped}. More recently, output transformations have been extended to compositions of elementary functions \cite{2019_Rios} and normalising flows \cite{2015_Rezende,2020_Maronas}. Output transformations have not featured prominently in the Bayesian optimisation literature, perhaps due to the commonly-held opinion that warped GPs require more data relative to standard GPs in order to function as effective surrogates \cite{2020_Nguyen}. Rather than introduce additional hyper-parameters to the GP, we enable efficient output warping through methods that only require pre-training. Recent work \cite{2021_Eriksson} has also investigated Gaussian copula transforms which may prove to be particularly effective in situations where there are outliers. \\

\noindent{\textbf{Non-stationarity with input warping:}} Many surrogate models with input warping exist for optimising non-stationary black-box objectives \cite{snoek2014input,2016_Calandra,2018_Oh} and have enjoyed particular success in hyper-parameter tuning where the natural scale of parameters is often logarithmic. Traditionally, a Beta cumulative distribution function is used. In this paper, we adopt the \texttt{Kumaraswamy} warping which is another instance of the generalised Beta class of distributions which we have observed to achieve superior performance~\cite{snoek2014input}~\footnote{For clarity we note that the input warping function used in~\cite{snoek2014input} is the same one used in this work.}; confirming results reported in~\cite{balandat2020botorch}.     \\  

\noindent{\textbf{Multi-objective acquisition ensembles:}} Multi-objective acquisition ensembles were first proposed in \cite{lyu2018batch} and are closely related to portfolios of acquisition functions \cite{2011_Hoffman,2014_Shahriari,balandat2020botorch}. In this form, the optimisation problem involves at least two conflicting and expensive black-box objectives and as such, solutions are located along the Pareto-efficient frontier. The multi-objective acquisition ensemble employs these ideas to find a Pareto-efficient solution amongst multiple acquisition functions. Although we utilised the multi-objective acquisition ensemble, we note that our framework is solver agnostic in so far as any multi-objective optimiser~\cite{2019_Abdolshah} may be applied.  \\

\noindent{\textbf{Robustness of Acquisitions:}} Methods achieving robustness with respect to either surrogates \cite{2020_Park} or the optimisation process \cite{2018_Bogunovic,2010_Bertsimas} have been previously proposed. Most relevant to our setting, is the approach of \cite{2018_Bogunovic} that introduces robustness to BO by solving a $\max \min$ objective to determine optimal input perturbations. Their method, however, relies on gradient ascent-descent-type algorithms that require real-valued variables and are not guaranteed to converge in the general non-convex, non-concave setting~\cite{MJ}. On the other hand, our solution possesses two advantages: 1) simplicity of implementation as we merely require random perturbations of acquisition functions to guarantee robustness, and 2) support for mixed variable solutions through the use of evolutionary solvers.

\section{Conclusion \& Future Work}
In this paper, we presented an in-depth empirical study of Bayesian optimisation for hyper-parameter tuning tasks. We demonstrated that even the simplest among machine learning problems can exhibit heteroscedasticity and non-stationarity. We also reflected on the affects of misspecified models and conflicting acquisition functions. We augmented BO algorithms with various enhancements and revealed that with a revised set of assumptions BO can in fact act as a competitive baseline in hyper-parameter tuning. We highlight the large discrepancy between suggestion power of synchronous and asynchronous methods. We hope for future work to focus on integrating the best of asynchronous and synchronous methods for optimal performance. We hope this paper's findings can guide the community when employing black-box and Bayesian optimisation in practice. 


\appendix
\section{Addition Detail Of Hypothesis Tests}\label{test:additionalinfo}

\noindent \textbf{Levene's Test}
Levene's test statistic is defined as
\begin{equation*}
    W = \frac{N - k}{k - 1} \cdot \frac{\sum_{i=1}^{k} n (\bar{Z}_{i \cdot} - \bar{Z}_{\cdot \cdot})^2}{\sum_{i=1}^{k} \sum_{j=1}^{n} (Z_{i j} - \bar{Z}_{i \cdot})^2},
\end{equation*}
where $N = k\times n$, $Z_{i j} = |Y_{i j} - \tfrac{1}{n}\sum_{j=1}^{n} Y_{i j}|$, $\bar{Z}_{i \cdot} = \tfrac{1}{n}\sum_{j=1}^{n} Z_{i j}$ and $\bar{Z}_{\cdot \cdot} = \tfrac{1}{k}\sum_{i=1}^{k} \bar{Z}_{i\cdot}$, for all $i = 1, \dots, k$, $j = 1, \dots, n$.
Levene's test rejects the null hypothesis of homoscedasticity $H_0$ if
\begin{equation*}
    W > F_{\alpha, k-1, N - k},
\end{equation*}
where $F_{\alpha, k-1, N - k}$ is the upper critical value at a significance level $\alpha$ of the $F$ distribution with $k-1$ and $N-k$ degrees of freedom. The Fligner-Killeen test is an alternative to Levene's test that is particularly robust to outliers.  \\

\noindent \textbf{Fligner-Killeen Test:}
Computation of the Fligner-Killeen test involves ranking all the absolute values $\{|Y_{i j} - \Tilde{Y}_i|\}_{1\leq i \leq k, 1 \leq j \leq n}$, where $\Tilde{Y}_i$ is the median of $\{Y_{i j}\}_{1 \leq j \leq n}$. Increasing scores $a_{N,r} = \Phi^{-1}\left(\tfrac{1 + \tfrac{r}{N + 1}}{2}\right)$ are associated with each rank $r = 1,\dots, N$, where $N =kn$ and $\Phi(\cdot)$ is the cumulative distribution function for a standard normal random variable. We denote the rank score associated with $Y_{ij}$ as $r_{ij}$. The Fligner-Killeen test statistic is given by

\begin{equation*}
    \chi_{o}^{2}=\frac{\sum_{i=1}^{k} n\left(\bar{A}_{i}-\bar{a}\right)^{2}}{V^{2}},
\end{equation*}

\noindent where $\bar{A}_i = \tfrac{1}{n}\sum_{j=1}^{n} a_{N, r_ij}$, $\bar{a} = \tfrac{1}{N} \sum_{r=1}^{N} a_{N, r}$ and $V^2 = \tfrac{1}{N - 1}\sum_{r=1}^N (a_{N,r} - \bar{a})^2$. As $\chi_0$ has an asymptotic $\mathcal{X}^2$ distribution with $(k-1)$ degrees of freedom, the test rejects the null hypothesis of homoscedasticity $H_0$ if

\begin{equation*}
    \chi_0 > \mathcal{X}^2_{\alpha, k-1}
\end{equation*}

\noindent where $\mathcal{X}^2_{\alpha, k-1}$ is the upper critical value at a significance level $\alpha$ of the $\mathcal{X}^2$ distribution with $k-1$ degrees of freedom.

\section{Details of Robust Acquisition}\label{app:robacq}

Though appealing, our formulation still assumes access to the GP hyper-parameters which complicates the implementation by restricting models and optimisers to the same programming paradigm. Ideally, we would wish to illicit robustness through only the GP predictive mean and predictive variance. Fortunately, we are able to show that upon a simple acquisition perturbation it becomes possible to approximate $\alpha_{\text{rob}}(\cdot)$ above. As such, we demonstrate that robust acquisition formulations are achievable using only the GP predictive mean and variance. \\ 

\begin{theorem}
Let us consider the  stochastic version of the acquisition function utilised in HEBO and given by $\overline{\alpha}^{\bm{\theta}} (\bm{x}|\mathcal{D}) = \alpha^{\bm{\theta}} (\bm{x}|\mathcal{D}) + \eta \sigma_{n}$ with $\eta \sim \mathcal{N}(0, 1)$ and  standard deviation parameter $\sigma_{n}>0$ \footnote{We note that gradient-based algorithms remain applicable upon addition of the $\eta \sigma_n$ term. In our formulation however, we use an evolutionary method which utilises acquisition function values. Consequently, the path followed by the optimiser will be altered based on $\eta$ samples leading to more robust query locations.}. Let $\alpha^{\bm{\theta}}_{\text{rob.}}(\bm{x}|\mathcal{D}) \equiv \mathbb{E}_{\epsilon \sim \mathcal{N}(\bm{0}, \sigma_{\epsilon}^{2}\bm{I})}\left[\alpha^{\bm{\theta}+\epsilon}(\bm{x}|\mathcal{D})\right]$ be the robust form of the standard acquisition function given as expectation over random perturbation of parameter $\bm{\theta}$. Then, by properly choosing parameters $\sigma_n$ and  $\sigma_{\epsilon}$ with high probability\footnote{Here we use the common approach for proving stochastic expressions with high probability (see \cite{jordan_cubic}, \cite{AZ01}). Specifically, we show that for any confidence parameter   $\delta\in(0,1)$ the stochastic expression under consideration is valid with probability at least $1 - \delta$.}, \text{HEBO} acquisition function  $\overline{\alpha}^{\bm{\theta}} (\bm{x}|\mathcal{D}) $ accurately approximates the robust acquisition function $\alpha^{\bm{\theta}}_{\text{rob.}}(\bm{x}|\mathcal{D})$ for any $\bm{\theta},\bm{x}$. More formally, for any $\rho\in(0,1)$ and $\delta \in(0,1)$, there are parameters $\sigma_n = \sigma_n(\rho,\delta)$ and $\sigma_{\epsilon} = \sigma_{\epsilon}(\rho, \delta)$ such that:    
\begin{equation*}
    \left|\overline{\alpha}^{\bm{\theta}} (\bm{x}|\mathcal{D}) - \alpha^{\bm{\theta}}_{\text{rob.}}(\bm{x}|\mathcal{D})\right| \leq \rho,  \ \ \ \ \ \ \forall \bm{\theta,x}
\end{equation*}
with probability at least $1 - \delta$.\\ 
\end{theorem}


\subsubsection{Proof of the Robustness Bound}\label{sec:prooflemma}


Let $\delta \in (0,1)$ be the desired probability threshold, and $\rho\in(0,1)$ be a desired accuracy parameter. Consider the GP with mean function $m(\bm{x})$ and covariance function $k_{\bm{\theta}}(\bm{x},\bm{x}^{\prime})$ such that $\forall \bm{x},\bm{x}^{\prime}\in\mathcal{X}$, $\bm{\theta}\in\mathbb{R}^{p}$:  
\begin{align}\label{lemma_conditions}
    &|k_{\bm{\theta}}(\bm{x},\bm{x})|\ge M_0,\ \  |k_{\bm{\theta}}(\bm{x},\bm{x}^{\prime})|\le M_1, \\\nonumber
    &||\nabla_{\boldsymbol{\theta}}k_{\bm{\theta}}(\bm{x},\bm{x}^{\prime})||_2\le M_2, \ \ |m(\bm{x})| \le M_4.
\end{align}
Moreover, assume that observations $y\in\mathcal{D}$ are bounded, i.e. $|y|\le C$ and let $\overline{\alpha}^{\bm{\theta}} (\bm{x}|\mathcal{D}) = \alpha^{\bm{\theta}} (\bm{x}|\mathcal{D}) + \eta \sigma_{n}$ with $\eta$ a standard normal random variable. Then, we are going to show that there are values $c_1 = c_1(\rho, \delta)$ and $c_2 = c_2(\rho,\delta)$, such that choosing  $\sigma_{n} \le c_1$ and $\sigma_{\epsilon}\le c_2$: 
\begin{equation*}
    \left|\overline{\alpha}^{\bm{\theta}} (\bm{x}|\mathcal{D}) - \mathbb{E}_{\epsilon \sim \mathcal{N}(\bm{0}, \sigma_{\epsilon}^{2}\bm{I})}\left[\alpha^{\bm{\theta}+\epsilon}(\bm{x}|\mathcal{D})\right]\right| \leq \rho. 
\end{equation*}
with probability at least $1 - \delta$. Note, the robust form of the acquisition function given as $\alpha^{\bm{\theta}}_{\text{rob.}}(\bm{x}|\mathcal{D}) \equiv \mathbb{E}_{\epsilon \sim \mathcal{N}(\bm{0}, \sigma_{\epsilon}^{2}\bm{I})}\left[\alpha^{\bm{\theta}+\epsilon}(\bm{x}|\mathcal{D})\right]$ constitutes an intractable integral. Therefore, in order to be optimised during the course of Bayesian optimisation, the intractable integral must be replaced by an accurate approximation. Without loss of generality we choose the UCB acquisition function  $\alpha^{\bm{\theta}}(\bm{x}|\mathcal{D}) = \alpha_{\text{UCB}}^{\bm{\theta}}(\bm{x}|\mathcal{D})$ and to avoid technical complications relating to multivariate calculus we consider a batch size $q=1$.
In this case, the UCB acquisition function can be written as $\alpha_{\text{UCB}}^{\bm{\theta}}(\bm{x}|\mathcal{D}) = \mu_{\bm{\theta}}(\bm{x}|\mathcal{D}) + \sqrt{\frac{\beta\pi}{2}}\sigma_{\boldsymbol{\theta}}(\bm{x}|\mathcal{D})$, where $\mu_{\bm{\theta}}(\bm{x}|\mathcal{D})$ and $\sigma_{\boldsymbol{\theta}}(\bm{x}|\mathcal{D})$ are the posterior mean and posterior standard deviation respectively. Consider a Monte-Carlo estimation of
$\alpha^{\bm{\theta}}_{\text{rob.}}(\bm{x}|\mathcal{D}) \equiv \mathbb{E}_{\epsilon \sim \mathcal{N}(\bm{0}, \sigma_{\epsilon}^{2}\bm{I})}\left[\alpha^{\bm{\theta}+\epsilon}(\bm{x}|\mathcal{D})\right]$:
\begin{align*}
    \hat{\alpha}^{\bm{\theta}}(\bm{x}|\mathcal{D}) = \frac{1}{N_{\epsilon}}\sum_{j=1}^{N_{\epsilon}}\alpha^{\bm{\theta}+ \bm{\epsilon}_j}(\bm{x}|\mathcal{D}) 
\end{align*}
where $\bm{\epsilon}_j $ are i.i.d samples drawn from $ \mathcal{N}(\bm{0},\sigma^2_{\epsilon}\bm{I})$. Then, adding and subtracting $\hat{\alpha}^{\bm{\theta}}(\bm{x}|\mathcal{D})$ gives:
\begin{align*}
    &\left|\overline{\alpha}^{\bm{\theta}} (\bm{x}|\mathcal{D}) - \mathbb{E}_{\epsilon \sim \mathcal{N}(\bm{0}, \sigma_{\epsilon}^{2}\bm{I})}\left[\alpha^{\bm{\theta}+\epsilon}(\bm{x}|\mathcal{D})\right]\right| \le \\\nonumber
    &\left|\overline{\alpha}^{\bm{\theta}} (\bm{x}|\mathcal{D}) - \hat{\alpha}^{\bm{\theta}}(\bm{x}|\mathcal{D})\right| + \nonumber
    \left| \hat{\alpha}^{\bm{\theta}}(\bm{x}|\mathcal{D}) - \mathbb{E}_{\epsilon \sim \mathcal{N}(\bm{0}, \sigma_{\epsilon}^{2}\bm{I})}\left[\alpha^{\bm{\theta}+\epsilon}(\bm{x}|\mathcal{D})\right]\right|.
\end{align*}
Using the definition of $\hat{\alpha}^{\bm{\theta}}(\bm{x}|\mathcal{D})$ in the above result yields:
\begin{align}\label{Eq:result_one_overall}
    &\left|\overline{\alpha}^{\bm{\theta}} (\bm{x}|\mathcal{D}) - \mathbb{E}_{\epsilon \sim \mathcal{N}(\bm{0}, \sigma_{\epsilon}^{2}\bm{I})}\left[\alpha^{\bm{\theta}+\epsilon}(\bm{x}|\mathcal{D})\right]\right| \le \\\nonumber
    &\frac{1}{N_{\epsilon}}\sum_{j=1}^{N_{\epsilon}}\left|\overline{\alpha}^{\bm{\theta}} (\bm{x}|\mathcal{D}) - \alpha^{\bm{\theta}+\bm{\epsilon}_j}(\bm{x}|\mathcal{D})\right| + \nonumber
    \left|\frac{1}{N_{\epsilon}}\sum_{j=1}^{N_{\epsilon}}\alpha^{\bm{\theta}+\bm{\epsilon}_j}(\bm{x}|\mathcal{D}) -  \mathbb{E}_{\epsilon \sim \mathcal{N}(\bm{0}, \sigma_{\epsilon}^{2}\bm{I})}\left[\alpha^{\bm{\theta}+\epsilon}(\bm{x}|\mathcal{D})\right]\right|
\end{align}
Let us now study separately each term in the above result. Applying the Chebyshev inequality for the second term in the above expression, we have that with probability at least $p_1 = 1 - \frac{8\left[\mathbb{E}_{\bm{\epsilon}}\left[\mu^2_{\bm{\theta} + \bm{\epsilon}}(\bm{x}|\mathcal{D})\right] + \frac{\beta\pi}{2}\mathbb{E}_{\bm{\epsilon}}\left[\sigma^2_{\bm{\theta} + \bm{\epsilon}}(\bm{x}|\mathcal{D})\right]\right]}{N_{\epsilon}\rho^2}$:
\begin{align}\label{Eq:result_first_bound}
    &\left|\frac{1}{N_{\epsilon}}\sum_{j=1}^{N_{\epsilon}}\alpha^{\bm{\theta}+\bm{\epsilon}_j}(\bm{x}|\mathcal{D}) -  \mathbb{E}_{\epsilon \sim \mathcal{N}(\bm{0}, \sigma_{\epsilon}^{2}\bm{I})}\left[\alpha^{\bm{\theta}+\epsilon}(\bm{x}|\mathcal{D})\right]\right|\le \frac{\rho}{2}.
\end{align}
In order to ensure that $p_1 = 1 -  \frac{\delta}{2}$, the sample number $\bm{\epsilon}_j$ should be taken:
\begin{align*}
    N_{\epsilon} = \left\lceil\frac{16\left[\mathbb{E}_{\bm{\epsilon}}\left[\mu^2_{\bm{\theta} + \bm{\epsilon}}(\bm{x}|\mathcal{D})\right] + \frac{\beta\pi}{2}\mathbb{E}_{\bm{\epsilon}}\left[\sigma^2_{\bm{\theta} + \bm{\epsilon}}(\bm{x}|\mathcal{D})\right]\right]}{\delta\rho^2}\right\rceil.
\end{align*}
We will later simplify this expression using the bounds in (\ref{lemma_conditions}). For now, we restrict our focus on the second term in (\ref{Eq:result_one_overall}). To bound it, we will establish a bound on $|\overline{\alpha}^{\bm{\theta}} (\bm{x}|\mathcal{D}) - \alpha^{\bm{\theta}+\bm{\epsilon}_j}(\bm{x}|\mathcal{D})$. For a small random perturbation $\bm{\epsilon}_j$ we have (with probability 1):
\begin{align*}
    &\alpha^{\bm{\theta}+\bm{\epsilon}_j}(\bm{x}|\mathcal{D}) =  \alpha^{\bm{\theta}}(\bm{x}|\mathcal{D}) + \bm{\epsilon}^{\mathsf{T}}_j\nabla_{\bm{\theta}}\alpha^{\bm{\theta}}(\bm{x}|\mathcal{D}) + o(||\bm{\epsilon_j}||)=\\\nonumber
    &\alpha^{\bm{\theta}}(\bm{x}|\mathcal{D}) + \bm{\epsilon}^{\mathsf{T}}_j\nabla_{\bm{\theta}}\left[\mu_{\bm{\theta}}(\bm{x}|\mathcal{D}) + \sqrt{\frac{\beta\pi}{2}}\sigma_{\bm{\theta}}(\bm{x}|\mathcal{D})\right]+o(||\bm{\epsilon_j}||_2).
\end{align*}
Let us define 
\begin{equation*}
\bm{h}_{ \bm{\theta}}(\bm{x}|\mathcal{D}) = \nabla_{\bm{\theta}}\left[\mu_{\bm{\theta}}(\bm{x}|\mathcal{D}) + \sqrt{\frac{\beta\pi}{2}}\sigma_{\bm{\theta}}(\bm{x}|\mathcal{D})\right],    
\end{equation*}
then, using the Cauchy–Schwarz inequality we have:
\begin{align*}
    \left|\alpha^{\bm{\theta}+\bm{\epsilon}_j}(\bm{x}|\mathcal{D}) - \alpha^{\bm{\theta}}(\bm{x}|\mathcal{D})\right|&\le  ||\bm{\epsilon_j}||_2 ||[\bm{h}_{ \bm{\theta}}(\bm{x}|\mathcal{D})||_2 +o(1)]
\end{align*}
Since $\bm{\epsilon}_j\sim\mathcal{N}(0,1)$, then with probability at least $1 - \frac{\delta}{4N_{\epsilon}}$:
\begin{align*}
    ||\bm{\epsilon}_j||_2 \le 4\sigma_{\epsilon}\sqrt{p} + 2\sigma_{\epsilon}\sqrt{\log\frac{4N_{\epsilon}}{\delta}}
\end{align*}
Let us assume (and later we will prove the existence of such a bound) that $||\bm{h}_{ \bm{\theta}}(\bm{x}|\mathcal{D})||_2 \le A_1$. Then, with probability at least $1 - \frac{\delta}{4N_{\epsilon}}$:
\begin{align*}
    &\left|\alpha^{\bm{\theta}+\bm{\epsilon}_j}(\bm{x}|\mathcal{D}) - \alpha^{\bm{\theta}}(\bm{x}|\mathcal{D})\right|\le \nonumber
    \left[4\sigma_{\epsilon}\sqrt{p} + 2\sigma_{\epsilon}\sqrt{\log\frac{4N_{\epsilon}}{\delta}}\right]\left[A_1 + o(1)\right]
\end{align*}
On the other hand, for $\overline{\alpha}^{\bm{\theta}} (\bm{x}|\mathcal{D}) = \alpha^{\bm{\theta}} (\bm{x}|\mathcal{D}) + \eta \sigma_{\eta}$ with probability at least $1 - \frac{\delta}{4N_{\epsilon}}$ we have:
\begin{align*}
    &\left|\overline{\alpha}^{\bm{\theta}} (\bm{x}|\mathcal{D}) - \alpha^{\bm{\theta}} (\bm{x}|\mathcal{D})\right| \le \Phi^{-1}\left(1 - \frac{\delta}{8N_{\epsilon}}\right)\sigma_{n}.
\end{align*}
where $\Phi(\cdot)$ is the cumulative distribution function for a standard Gaussian random variable. Hence, by choosing $\sigma_{\bm{\epsilon}} = \min\left\{1, \frac{\Phi^{-1}\left(1 - \frac{\delta}{8N_{\epsilon}}\right)\sigma_{n}}{\left[4\sqrt{p} + 2\sqrt{\log\frac{4N_{\epsilon}}{\delta}}\right]\left[A_1 + o(1)\right]}\right\}$ 
with probability at least $1 - \frac{\delta}{2N_\epsilon}$ we have that both $\overline{\alpha}^{\bm{\theta}} (\bm{x}|\mathcal{D})$ and $\alpha^{\bm{\theta}+\bm{\epsilon}_j}(\bm{x}|\mathcal{D})$ belong to the interval centred at $\alpha^{\bm{\theta}}(\bm{x}|\mathcal{D})$ of size $\Phi^{-1}\left(1 - \frac{\delta}{8N_{\epsilon}}\right)\sigma_{n}$. Therefore, with probability at least $1 - \frac{\delta}{2N_{\epsilon}}$:
\begin{align*}
    \left|\overline{\alpha}^{\bm{\theta}} (\bm{x}|\mathcal{D}) - \alpha^{\bm{\theta}+\bm{\epsilon}_j}(\bm{x}|\mathcal{D})\right| \le 2\Phi^{-1}\left(1 - \frac{\delta}{8N_{\epsilon}}\right)\sigma_{n}
\end{align*}
Hence, by choosing $\sigma_n = \frac{\rho}{4\Phi^{-1}\left(1 - \frac{\delta}{8N_{\epsilon}}\right)}$ we arrive at:
\begin{align*}
    &\left|\overline{\alpha}^{\bm{\theta}} (\bm{x}|\mathcal{D}) - \alpha^{\bm{\theta}+\bm{\epsilon}_j}(\bm{x}|\mathcal{D})\right| \le \frac{\rho}{2}
\end{align*}
and, therefore, for the first term in (\ref{Eq:result_one_overall}) with probability at least $1 - \frac{\delta}{2}$ we have:
\begin{align*}
    &\frac{1}{N_{\epsilon}}\sum_{j=1}^{N_{\epsilon}}\left|\overline{\alpha}^{\bm{\theta}} (\bm{x}|\mathcal{D}) - \alpha^{\bm{\theta}+\bm{\epsilon}_j}(\bm{x}|\mathcal{D})\right| \le \frac{\rho}{2}
\end{align*}
Combining this result with (\ref{Eq:result_first_bound}) gives, that with probability at least $1 - \delta$:
\begin{align*}
     \left|\overline{\alpha}^{\bm{\theta}} (\bm{x}|\mathcal{D}) - \mathbb{E}_{\epsilon \sim \mathcal{N}(\bm{0}, \sigma_{\epsilon}^{2}\bm{I})}\left[\alpha^{\bm{\theta}+\epsilon}(\bm{x}|\mathcal{D})\right]\right| \le \rho, \ \ \ \ \ \ \ \ \forall\bm{\theta,x},
\end{align*}
upon defining:
\begin{align}
\label{Eq:parameter_setup}
    &\sigma_n = \frac{\rho}{4\Phi^{-1}\left(1 - \frac{\delta}{8N_{\epsilon}}\right)}, \ 
    \sigma_{\bm{\epsilon}} = \min\left\{1, \frac{\rho}{8\left[2\sqrt{p} + \sqrt{\log\frac{4N_{\epsilon}}{\delta}}\right]\left[A_1 + o(1)\right]}\right\},
\end{align}
with
\begin{align*}
    &N_{\epsilon} = \lceil\frac{16\left[\mathbb{E}_{\bm{\epsilon}}\left[\mu^2_{\bm{\theta} + \bm{\epsilon}}(\bm{x}|\mathcal{D})\right] + \frac{\beta\pi}{2}\mathbb{E}_{\bm{\epsilon}}\left[\sigma^2_{\bm{\theta} + \bm{\epsilon}}(\bm{x}|\mathcal{D})\right]\right]}{\delta\rho^2}\rceil.
\end{align*}
Our last step is to prove the existence of a constant $A_1$ such that $||\bm{h}_{\bm{\theta}}(\bm{x})||_2\le A_1$ and also to simplify these expressions by deriving bounds on  $\mathbb{E}_{\bm{\epsilon}}\left[\mu_{\bm{\theta}+\bm{\epsilon}}(\bm{x}|\mathcal{D})\right]$ and $\mathbb{E}_{\bm{\epsilon}}\left[\sigma^2_{\bm{\theta}+\bm{\epsilon}}(\bm{x}|\mathcal{D})\right]$. This will be provided as a separate Claim: \\

\begin{claim}
Let the bounds in (\ref{lemma_conditions}) hold, then there are positive constants $A_1,A_2$ and $A_3$, such that
\begin{equation}
\label{claim_results}
    ||\bm{h}_{\bm{\theta}}(\bm{x})||_2 \le A_1,\ \ 
    \mathbb{E}_{\bm{\epsilon}}\left[\mu_{\bm{\theta}+\bm{\epsilon}}(\bm{x}|\mathcal{D})\right] \le A_2, \ \ \ \mathbb{E}_{\bm{\epsilon}}\left[\sigma^2_{\bm{\theta}+\bm{\epsilon}}(\bm{x}|\mathcal{D})\right] \le A_3.
\end{equation}
\end{claim}

\begin{proof}
We start with the bound on $||\bm{h}_{\bm{\theta}}(\bm{x})||_2$. Let us denote for simplicity $\bm{a}_{\bm{\theta}} = [k_{\bm{\theta}}(\bm{x},\bm{x}_i)]_{\bm{x}_i\in\mathcal{D}}$, $\bm{B}_{\bm{\theta}} = \left[\left[k_{\bm{\theta}}(\bm{x},\bm{x}^{\prime})\right]_{\bm{x}\in\mathcal{D},\bm{x}^{\prime}\in\mathcal{D}}+ \bm{I}\right]^{-1}$, $\bm{y} = [y(\bm{x})]_{\bm{x\in\mathcal{D}}}$,  $\bm{m}_{\mathcal{D}} = [m(\bm{x})]_{\bm{x}\in\mathcal{D}}$, $m = m(\bm{x})$, and $k_{\bm{\theta}} = k_{\bm{\theta}}(\bm{x},\bm{x})$, then
\begin{align*}
    &\mu_{\bm{\theta}}(\bm{x}|\mathcal{D}) = \bm{a}^{\mathsf{T}}_{\bm{\theta}}\bm{B}_{\bm{\theta}}[\bm{y} - \bm{m}_{\mathcal{D}}] + m,\ \ \nonumber
    \sigma^2_{\bm{\theta}}(\bm{x}|\mathcal{D}) = \bm{a}_{\bm{\theta}}^{\mathsf{T}}\bm{B}_{\bm{\theta}}\bm{a}_{\bm{\theta}} + k_{\bm{\theta}}
\end{align*}
Let us also denote the size of $\mathcal{D}$ as $N$, then we have:
\begin{align*}
    &\nabla_{\boldsymbol{\theta}}\mu_{\bm{\theta}}(\bm{x}|\mathcal{D}) = \sum_{i=1}^N\sum_{j=1}^N\nabla_{\bm{\theta}}\left[[y_j - m_j][\bm{a}_{\bm{\theta}}]_i[\boldsymbol{B}_{\bm{\theta}}]_{ij}\right] = \\\nonumber
    &\sum_{i=1}^N\sum_{j=1}^N\left[[y_j - m_j][\bm{B}_{\bm{\theta}}]_{ij}\nabla_{\bm{\theta}}[\bm{a}_{\bm{\theta}}]_i\right]+\nonumber
    \sum_{i=1}^N\sum_{j=1}^N\left[[y_j - m_j][\bm{a}_{\bm{\theta}}]_i\nabla_{\bm{\theta}}[[\bm{B}_{\bm{\theta}}]_{ij}]\right].
\end{align*}
Consider each term in this expression separately:
\begin{align*}
    &\left|\left|\sum_{i=1}^N\sum_{j=1}^N[y_j - m_j][\bm{B}_{\bm{\theta}}]_{ij}\nabla_{\bm{\theta}}[\bm{a}_{\bm{\theta}}]_i\right|\right|_2 = \nonumber
    \left|\left|\sum_{i=1}^N[\bm{B}_{\bm{\theta}}[\bm{y} - \bm{m}_{\mathcal{D}}]]_{i}\nabla_{\bm{\theta}}[\bm{a}_{\bm{\theta}}]_i\right|\right|_2 \le \\\nonumber
    &\sum_{i=1}^N||\bm{B}_{\bm{\theta}}(i,:)||_2||\bm{y} - \bm{m}_{\mathcal{D}}||_2\left|\left|\nabla_{\bm{\theta}}[\bm{a}_{\bm{\theta}}]_i\right|\right|_2.
\end{align*}
Using $|y|\le C$ and $|m(\boldsymbol{x})|\le M_4$, we have $||\bm{y} - \bm{m}_{\mathcal{D}}||_2 \le (C+M_4)\sqrt{N}$ and $||\nabla_{\boldsymbol{\theta}}[\boldsymbol{a}(\boldsymbol{\theta})]_i||_2\le M_2$ and as such:
\begin{align}
\label{part_prev_11}
    &\left|\left|\sum_{i=1}^N\sum_{j=1}^N[y_j - m_j][\bm{B}_{\bm{\theta}}]_{ij}\nabla_{\bm{\theta}}[\bm{a}_{\bm{\theta}}]_i\right|\right|_2 \le
    (C+M_4)\sqrt{N}\sum_{i=1}^N||\bm{B}(\bm{\theta})(i,:)||_2\left|\left|\nabla_{\bm{\theta}}[\bm{a}_{\bm{\theta}}]_i\right|\right|_2 \le\\\nonumber
    &(C+M_4)\sqrt{N}||\bm{B}_{\bm{\theta}}||_F\sum_{i=1}^N||\nabla_{\bm{\theta}}[\boldsymbol{a}_{\bm{\theta}}]_i||_2 \le \nonumber
    (C+M_4)\sqrt{N}\sqrt{\text{rank}(\bm{B}_{\bm{\theta}}}||\bm{B}_{\bm{\theta}}||_2\sum_{i=1}^N||\nabla_{\bm{\theta}}[\bm{a}_{\bm{\theta}}]_i||_2 \le \\\nonumber
    &\frac{(C+M_4)N}{\sigma^2_{n}}\sum_{i=1}^N||\nabla_{\boldsymbol{\theta}}[\boldsymbol{a}(\boldsymbol{\theta})]_i||_2 = \frac{(C+M_4)N^2M_2}{\sigma^2_{n}}
\end{align}
Now, let us consider the second term in the expression for the posterior mean:
\begin{align*}
    &\sum_{i=1}^N\sum_{j=1}^N[y_j - m_j][\bm{a}_{\bm{\theta}}]_i\nabla_{\bm{\theta}}\left[[\bm{B}_{\bm{\theta}}]_{ij}\right] = \nonumber
    \sum_{i=1}^N[\bm{a}_{\bm{\theta}}]_i\left[\sum_{j=1}[y_j - m_j]\nabla_{\boldsymbol{\theta}}\left[[\bm{B}_{\bm{\theta}}]_{ij}\right]\right].
\end{align*}
Notice, that the gradient expression above is presented in the form of a vector:
\begin{align*}
    \nabla_{\bm{\theta}}\left[[\bm{B}_{\bm{\theta}}]_{ij}\right]  = \left[\begin{array}{c}
      \frac{\partial}{\partial\theta_1}\left[\bm{K}_{\boldsymbol{\theta}} + \sigma_{n}\boldsymbol{I}\right]^{-1}_{ij}, \\
      \vdots \\
      \frac{\partial}{\partial\theta_p}\left[\bm{K}_{\boldsymbol{\theta}} + \sigma_{n}\boldsymbol{I}\right]^{-1}_{ij}\\ 
        \end{array},
    \right]
\end{align*}
where we use the notation $\bm{K}_{\bm{\theta}} = [k_{\bm{\theta}}(\bm{x}_i, \bm{x}_j)]^{N,N}_{i=1,j=1}$. For the $r^{th}$ component we have
\begin{equation}
\label{Eq:interm_result}
    \frac{\partial}{\partial\theta_r}\left[\bm{K}_{\boldsymbol{\theta}} + \sigma_{n}\boldsymbol{I}\right]^{-1}_{ij} = 
    \left[-\left[\bm{K}_{\boldsymbol{\theta}} + \sigma_{n}\boldsymbol{I}\right]^{-1}\frac{\partial}{\partial\theta_r}\left[\bm{K}_{\boldsymbol{\theta}} + \sigma_{n}\boldsymbol{I}\right]\left[\bm{K}_{\boldsymbol{\theta}} + \sigma_{n}\boldsymbol{I}\right]^{-1}\right]_{ij}
\end{equation}
Now we can study the gradient of the second term in the posterior mean expression,
\begin{align*}
    &\left|\left|\sum_{i=1}^N[\bm{a}_{\bm{\theta}}]_i\left[\sum_{j=1}[y_j - m_j]\nabla_{\bm{\theta}}\left[[\bm{B}_{\bm{\theta}}]_{ij}\right]\right]\right|\right|_2 \le \nonumber
    \sum_{i=1}^N\left|[\bm{a}_{\bm{\theta}}]_i\right|\left[\sum_{j=1}^N||\bm{y} - \bm{m}_{\mathcal{D}}||_2\left|\left|\nabla_{\bm{\theta}}\left[[\bm{B}_{\bm{\theta}}]_{ij}\right]\right|\right|_2\right] \le\\\nonumber
    &(C+M_4)\sqrt{N}M_1\sum_{i=1}^N\sum_{j=1}^N\sum_{r=1}^p\left|\frac{\partial}{\partial\theta_r}\left[\bm{K}_{\boldsymbol{\theta}} + \sigma_{n}\boldsymbol{I}\right]^{-1}_{ij}\right|.
\end{align*}
Using result (\ref{Eq:interm_result}) in the above expression we now have 
\begin{align*}
    &\left|\left|\sum_{i=1}^N[\bm{a}_{\bm{\theta}}]_i\left[\sum_{j=1}[y_j - m_j]\nabla_{\bm{\theta}}\left[[\bm{B}_{\bm{\theta}}]_{ij}\right]\right]\right|\right|_2 \le \\\nonumber
    &(C+M_4)\sqrt{N}M_1\sum_{i=1}^N\sum_{j=1}^N\sum_{r=1}^p\left|\frac{\partial}{\partial\theta_r}\left[\bm{K}_{\boldsymbol{\theta}} + \sigma_{n}\boldsymbol{I}\right]^{-1}_{ij}\right| \le \\\nonumber
    &(C+M_4)\sqrt{N}M_1\sum_{r=1}^p\sum_{i=1}^N\sum_{j=1}^N\left|\frac{\partial}{\partial\theta_r}\left[\bm{K}_{\boldsymbol{\theta}} + \sigma_{n}\boldsymbol{I}\right]^{-1}_{ij}\right|\le \\\nonumber
    &(C+M_4)N\sqrt{N}M_1\times
    \nonumber
    \sum_{r=1}^p\left|\left|\left[\bm{K}_{\boldsymbol{\theta}} + \sigma_{n}\boldsymbol{I}\right]^{-1}\frac{\partial}{\partial\theta_r}\left[\bm{K}_{\boldsymbol{\theta}} + \sigma_{n}\boldsymbol{I}\right]\left[\bm{K}_{\boldsymbol{\theta}} + \sigma_{n}\boldsymbol{I}\right]^{-1}\right|\right|_F,
\end{align*}
where we used that $\sum_{i=1}^N\sum_{j=1}^N|\boldsymbol{C}_{ij}|\le N||\boldsymbol{C}||_F$ for any arbitrary matrix $\boldsymbol{C}\in\mathbb{R}^{N\times N}$. Because $\frac{\partial}{\partial\theta_r}[\bm{K}_{\boldsymbol{\theta}} + \sigma_{n}\boldsymbol{I}] = \frac{\partial}{\partial\theta_r}\bm{K}_{\boldsymbol{\theta}}$
\begin{align*}
    &\frac{\left|\left|\sum_{i=1}^N[\bm{a}_{\bm{\theta}}]_i\left[\sum_{j=1}[y_j - m_j]\nabla_{\bm{\theta}}\left[[\bm{B}_{\bm{\theta}}]_{ij}\right]\right]\right|\right|_2}{(C+M_4)N\sqrt{N}M_1} \le \\\nonumber
    &\sum_{r=1}^p\left|\left|\left[\bm{K}_{\boldsymbol{\theta}} + \sigma_{n}\boldsymbol{I}\right]^{-1}\frac{\partial}{\partial\theta_r}\bm{K}_{\boldsymbol{\theta}}\left[\bm{K}_{\boldsymbol{\theta}} + \sigma_{n}\boldsymbol{I}\right]^{-1}\right|\right|_F \le \\\nonumber
    &\sqrt{N}\sum_{r=1}^p\left|\left|\left[\bm{K}_{\boldsymbol{\theta}} + \sigma_{n}\boldsymbol{I}\right]^{-1}\frac{\partial}{\partial\theta_r}\bm{K}_{\boldsymbol{\theta}}\left[\bm{K}_{\boldsymbol{\theta}} + \sigma_{n}\boldsymbol{I}\right]^{-1}\right|\right|_2
\end{align*}
and using the properties of the matrix 2-norm
$||\cdot||_2$
\begin{align*}
    \left|\left|\left[\bm{K}_{\boldsymbol{\theta}} + \sigma_{n}\boldsymbol{I}\right]^{-1}\right|\right|_2 \le \frac{1}{\sigma^2_n}
\end{align*}
Hence,
\begin{align}\label{Eq:intermd_result_two}
    &\frac{\left|\left|\sum_{i=1}^N[\bm{a}_{\bm{\theta}}]_i\left[\sum_{j=1}[y_j - m_j]\nabla_{\bm{\theta}}\left[[\bm{B}_{\bm{\theta}}]_{ij}\right]\right]\right|\right|_2}{(C+M_4)N^2M_1} \le\\\nonumber
    &\sum_{r=1}^p\left|\left|\left[\bm{K}_{\boldsymbol{\theta}} + \sigma_{n}\boldsymbol{I}\right]^{-1}\right|\right|_2\left|\left|\frac{\partial}{\partial\theta_r}\bm{K}_{\boldsymbol{\theta}}\right|\right|_2\left|\left| \left[\bm{K}_{\boldsymbol{\theta}} + \sigma_{n}\boldsymbol{I}\right]^{-1}\right|\right|_2\le \nonumber
    \frac{1}{\sigma^4_n}\sum_{r=1}^p\left|\left|\frac{\partial}{\partial\theta_r}\bm{K}_{\boldsymbol{\theta}}\right|\right|_2.
\end{align}
Let us study the last term in the expression above. Using $\sqrt{c^2_1 + \ldots + c^2_R} \le |c_1| + \ldots + |c_R|$ for any set of real numbers $c_1,\ldots,c_R\in\mathbb{R}$ we have
\begin{align*}
    &\left|\left|\frac{\partial}{\partial\theta_r}\bm{K}_{\boldsymbol{\theta}}\right|\right|_2 = \nonumber
    \left|\left|\left[\begin{array}{ccc}
    \frac{\partial}{\partial\theta_r}k_{\boldsymbol{\theta}}(\boldsymbol{x}_1, \boldsymbol{x}_1), & \ldots & \frac{\partial}{\partial\theta_r}k_{\boldsymbol{\theta}}(\boldsymbol{x}_1, \boldsymbol{x}_N)  \\
      \vdots & \vdots & \vdots \\
      \frac{\partial}{\partial\theta_r}k_{\boldsymbol{\theta}}(\boldsymbol{x}_N, \boldsymbol{x}_1), & \ldots & \frac{\partial}{\partial\theta_r}k_{\boldsymbol{\theta}}(\boldsymbol{x}_N, \boldsymbol{x}_N)  \\ 
        \end{array}l
    \right]\right|\right|_2 \le \\\nonumber
    &\left|\left|\left[\begin{array}{ccc}
      \frac{\partial}{\partial\theta_r}k_{\boldsymbol{\theta}}(\boldsymbol{x}_1, \boldsymbol{x}_1), & \ldots & \frac{\partial}{\partial\theta_r}k_{\boldsymbol{\theta}}(\boldsymbol{x}_1, \boldsymbol{x}_N)  \\
      \vdots & \vdots & \vdots \\
      \frac{\partial}{\partial\theta_r}k_{\boldsymbol{\theta}}(\boldsymbol{x}_N, \boldsymbol{x}_1), & \ldots & \frac{\partial}{\partial\theta_r}k_{\boldsymbol{\theta}}(\boldsymbol{x}_N, \boldsymbol{x}_N)  \\ 
        \end{array}
    \right]\right|\right|_F = \\\nonumber
    &\sqrt{\sum_{i=1}^N\sum_{j=1}^N\left[\frac{\partial}{\partial\theta_r}k_{\boldsymbol{\theta}}(\boldsymbol{x}_i, \boldsymbol{x}_j)\right]^2} \le \sum_{i=1}^N\sum_{j=1}^N\left|\frac{\partial}{\partial\theta_r}k_{\boldsymbol{\theta}}(\boldsymbol{x}_i, \boldsymbol{x}_j)\right|.
\end{align*}
Substituting this expression in (\ref{Eq:intermd_result_two}) gives us
\begin{align}
\label{Eq:part_12}
    &\frac{\left|\left|\sum_{i=1}^N[\bm{a}_{\bm{\theta}}]_i\left[\sum_{j=1}[y_j - m_j]\nabla_{\bm{\theta}}\left[[\bm{B}_{\bm{\theta}}]_{ij}\right]\right]\right|\right|_2}{(C+M_4)N^2M_1} \le 
    \frac{1}{\sigma^4_{n}}\sum_{r=1}^d\sum_{i=1}^N\sum_{j=1}^N\left|\frac{\partial}{\partial\theta_r}k_{\boldsymbol{\theta}}(\boldsymbol{x}_i, \boldsymbol{x}_j)\right|\le\\\nonumber
    &\frac{\sqrt{p}}{\sigma^4_{n}}\sum_{i=1}^N\sum_{j=1}^N\left|\left|\nabla_{\boldsymbol{\theta}}k_{\boldsymbol{\theta}}(\boldsymbol{x}_i, \boldsymbol{x}_j)\right|\right|_2 \le \frac{N^2\sqrt{p}M_2}{\sigma^4_{n}}.
\end{align}
Hence, combining results (\ref{part_prev_11}) and (\ref{Eq:part_12}) we have
\begin{equation}
\label{norm_posterior_mean_result}
    ||\nabla_{\boldsymbol{\theta}}\mu_{\bm{\theta}}(\bm{x}|\mathcal{D})||_2 \le \frac{(C+M_4)N^2M_2}{\sigma^2_n}\left[1 + \frac{N^2M_1\sqrt{p}}{\sigma^2_n}\right].
\end{equation}
Now, let us focus on the gradient of the posterior standard deviation,
\begin{align}\label{part_20}
    &\nabla_{\bm{\theta}}\sigma_{\bm{\theta}}(\bm{x}|\mathcal{D}) = \nonumber
    \nabla_{\bm{\theta}}\left[\sqrt{k_{\bm{\theta}}(\boldsymbol{x}, \bm{x}) - \bm{a}_{\bm{\theta}}^{\mathsf{T}} [\bm{K}_{\bm{\theta}} + \sigma^{2}_{n} \textbf{I}]^{-1} \bm{a}_{\bm{\theta}}}\right] = \\\nonumber
    &\frac{1}{2\sigma_{\bm{\theta}}(\bm{x}|\mathcal{D})}\nabla_{\bm{\theta}}\left[k_{\bm{\theta}}(\bm{x}, \boldsymbol{x}) - \bm{a}_{\bm{\theta}}^{\mathsf{T}} [\bm{K}_{\bm{\theta}} + \sigma^{2}_{n} \textbf{I}]^{-1} \bm{a}_{\bm{\theta}}\right] = \nonumber
    \frac{1}{2\sigma_{\bm{\theta}}(\bm{x})}\left[\nabla_{\bm{\theta}}k_{\bm{\theta}}(\bm{x}, \bm{x}) - \nabla_{\bm{\theta}}\left[\bm{a}_{\bm{\theta}}^{\mathsf{T}} [\bm{K}_{\bm{\theta}} + \sigma^{2}_{n} \textbf{I}]^{-1} \bm{a}_{\bm{\theta}}\right]\right].
\end{align}
Let us study the second gradient expression. Using our notation we have
\begin{align*}
    &\bm{a}_{\bm{\theta}}^{\mathsf{T}} [\bm{K}_{\bm{\theta}} + \sigma^{2}_{n} \textbf{I}]^{-1} \bm{a}_{\bm{\theta}} = \nonumber
    \bm{a}^{\mathsf{T}}_{\bm{\theta}}\bm{B}_{\bm{\theta}}\bm{a}_{\bm{\theta}} =  \sum_{i=1}^N\sum_{j=1}^N[\bm{a}_{\bm{\theta}}]_i[\bm{a}_{\bm{\theta}}]_j\left[\bm{B}_{\bm{\theta}}\right]_{ij}.
\end{align*}
Hence, for the gradient,
\begin{align*}
    &\nabla_{\bm{\theta}}\left[\bm{a}^{\mathsf{T}}_{\bm{\theta}}\bm{B}_{\bm{\theta}}\bm{a}_{\bm{\theta}}\right] = \sum_{i=1}^N\sum_{j=1}^N\nabla_{\bm{\theta}}\left[[\bm{a}_{\bm{\theta}}]_i[\bm{a}_{\bm{\theta}}]_j\left[\bm{B}_{\bm{\theta}}\right]_{ij}\right] = \\\nonumber
    &\sum_{i=1}^N\sum_{j=1}^N\nabla_{\bm{\theta}}\left[[\bm{a}_{\bm{\theta}}]_i\right][\bm{a}_{\bm{\theta}}]_j\left[\bm{B}_{\bm{\theta}}\right]_{ij} + \nonumber
    \sum_{i=1}^N\sum_{j=1}^N\nabla_{\bm{\theta}}\left[[\bm{a}_{\bm{\theta}}]_j\right][\bm{a}_{\bm{\theta}}]_i\left[\bm{B}_{\bm{\theta}}\right]_{ij} + \\\nonumber
    &\sum_{i=1}^N\sum_{j=1}^N\nabla_{\bm{\theta}}\left[\left[\bm{B}_{\bm{\theta}}\right]_{ij}\right][\bm{a}_{\bm{\theta}}]_i[\bm{a}_{\bm{\theta}}]_j.
\end{align*}
and for the norm of the above expression we have
\begin{align*}
    &\left|\left|\nabla_{\bm{\theta}}\left[\bm{a}^{\mathsf{T}}_{\bm{\theta}}\bm{B}_{\bm{\theta}}\bm{a}_{\bm{\theta}}\right]\right|\right|_2 = \sum_{i=1}^N\sum_{j=1}^N\nabla_{\bm{\theta}}\left[[\bm{a}_{\bm{\theta}}]_i[\bm{a}_{\bm{\theta}}]_j\left[\bm{B}_{\bm{\theta}}\right]_{ij}\right] = \\\nonumber
    &\sum_{i=1}^N\sum_{j=1}^N\left|\left|\nabla_{\bm{\theta}}\left[\left[\bm{B}_{\bm{\theta}}\right]_{ij}\right]\right|\right|_2\left|[\bm{a}_{\bm{\theta}}]_i[\bm{a}_{\bm{\theta}}]_j\right| + \nonumber
    \sum_{i=1}^N\sum_{j=1}^N\left|\left|\nabla_{\bm{\theta}}\right|\right|_2\left|\left[[\bm{a}_{\bm{\theta}}]_i\right][\bm{a}_{\bm{\theta}}]_j\left[\bm{B}_{\bm{\theta}}\right]_{ij}\right| + \\\nonumber
    &\sum_{i=1}^N\sum_{j=1}^N\left|\left|\nabla_{\bm{\theta}}\left[[\bm{a}_{\bm{\theta}}]_j\right]\right|\right|_2\left|[\bm{a}_{\bm{\theta}}]_i\left[\bm{B}_{\bm{\theta}}\right]_{ij}\right|.
\end{align*}
Let us now bound each term in this expression:
\begin{enumerate}
    \item The first term:
    \begin{align*}
        &\sum_{i=1}^N\sum_{j=1}^N\left|\left|\nabla_{\bm{\theta}}\left[\left[\bm{B}_{\bm{\theta}}\right]_{ij}\right]\right|\right|_2\left|[\bm{a}_{\bm{\theta}}]_i[\bm{a}_{\bm{\theta}}]_j\right| \le \\\nonumber
        &\sum_{i=1}^N\sum_{j=1}^N\left|\left|\nabla_{\bm{\theta}}\left[\left[\bm{B}_{\bm{\theta}}\right]_{ij}\right]\right|\right|_2\left|\left|\bm{a}_{\bm{\theta}}\right|\right|_2\left|\left|\bm{a}_{\bm{\theta}}\right|\right|_2\le\nonumber
        M^2_1\sum_{i=1}^N\sum_{j=1}^N\left|\left|\nabla_{\bm{\theta}}\left[\left[\bm{B}_{\bm{\theta}}\right]_{ij}\right]\right|\right|_2
    \end{align*}
    Using the previous bound for $\left|\left|\nabla_{\bm{\theta}}\left[\left[\bm{B}_{\bm{\theta}}\right]_{ij}\right]\right|\right|_2$ we have:
    \begin{align*}
        &\sum_{i=1}^N\sum_{j=1}^N\left|\left|\nabla_{\bm{\theta}}\left[\left[\bm{B}_{\bm{\theta}}\right]_{ij}\right]\right|\right|_2\left|[\bm{a}_{\bm{\theta}}]_i[\bm{a}_{\bm{\theta}}]_j\right| \le\nonumber
        M^2_1\sum_{i=1}^N\sum_{j=1}^N\sum_{r=1}^p\left|\frac{\partial}{\partial\theta_r}\left[\bm{K}_{\bm{\theta}} + \sigma^{2}_{n} \textbf{I}\right]^{-1}_{ij}\right| = \\\nonumber
        &NM^2_1\sum_{r=1}^p\left|\left|\left[\bm{K}_{\bm{\theta}} + \sigma_{n}\bm{I}\right]^{-1}\frac{\partial}{\partial\theta_r}\bm{K}_{\bm{\theta}}\left[\bm{K}_{\bm{\theta}} + \sigma_{n}\bm{I}\right]^{-1}\right|\right|_F \le\\\nonumber
        &N^{\frac{3}{2}}M^2_1\sum_{r=1}^p\left|\left|\left[\bm{K}_{\bm{\theta}} + \sigma_{n}\bm{I}\right]^{-1}\frac{\partial}{\partial\theta_r}\bm{K}_{\boldsymbol{\theta}}\left[\bm{K}_{\bm{\theta}} + \sigma_{n}\bm{I}\right]^{-1}\right|\right|_2
    \end{align*}
    Since $\left|\left|\left[\bm{K}_{\boldsymbol{\theta}} + \sigma_{n}\boldsymbol{I}\right]^{-1}\right|\right|_2 \le\frac{1}{\sigma^2_n}$ we have:
    \begin{align*}
        &\sum_{i=1}^N\sum_{j=1}^N\left|\left|\nabla_{\bm{\theta}}\left[\left[\bm{B}_{\bm{\theta}}\right]_{ij}\right]\right|\right|_2\left|[\bm{a}_{\bm{\theta}}]_i[\bm{a}_{\bm{\theta}}]_j\right| \le \nonumber
        \frac{N\sqrt{N}M^2_1}{\sigma^4_n}\sum_{r=1}^p\sum_{i=1}^N\sum_{j=1}^N\left|\frac{\partial}{\partial\theta_r}k_{\boldsymbol{\theta}}(\boldsymbol{x}_i,\boldsymbol{x}_j)\right|.
    \end{align*}
    Using $\sum_{r=1}^p\sum_{i=1}^N\sum_{j=1}^N\left|\frac{\partial}{\partial\theta_r}k_{\bm{\theta}}(\bm{x}_i,\bm{x}_j)\right| = \sqrt{p}\sum_{i=1}^N\sum_{j=1}^N\left|\left|\nabla_{\boldsymbol{\theta}}k_{\boldsymbol{\theta}}(\boldsymbol{x}_i,\boldsymbol{x}_j)\right|\right|_2 \le N^2\sqrt{p}M_2$, we have:
    \begin{align*}
        &\sum_{i=1}^N\sum_{j=1}^N\left|\left|\nabla_{\bm{\theta}}\left[\left[\bm{B}_{\bm{\theta}}\right]_{ij}\right]\right|\right|_2\left|[\bm{a}_{\bm{\theta}}]_i[\bm{a}_{\bm{\theta}}]_j\right| \le \frac{N^{\frac{7}{2}}\sqrt{p}M^2_1M_2}{\sigma^4_n}
    \end{align*}
    
    \item The second and the third terms are identical with respect to the bounding strategy,
    \begin{align*}
        &\sum_{i=1}^N\sum_{j=1}^N||\nabla_{\bm{\theta}}\left[[\bm{a}_{\bm{\theta}}]_i\right]||_2\left|[\bm{a}_{\bm{\theta}}]_j\left[\bm{B}_{\bm{\theta}}\right]_{ij}\right| = \nonumber
        \sum_{i=1}^N\left|\bm{B}_{\bm{\theta}}(i,:)\bm{a}_{\bm{\theta}}\right|\left|\left|\nabla_{\bm{\theta}}\left[[\bm{a}_{\bm{\theta}}]_i\right]\right|\right|_2 \le\\\nonumber
        &\sum_{i=1}^N\left|\left|\bm{B}_{\bm{\theta}}(i,:)\right|\right|_2\left|\left|\bm{a}_{\bm{\theta}}\right|\right|_2\left|\left|\nabla_{\bm{\theta}}\left[[\bm{a}_{\bm{\theta}}]_i\right]\right|\right|_2 \le \nonumber
        ||\bm{B}_{\bm{\theta}}||_F||\bm{a}_{\bm{\theta}}||_2\sum_{i=1}^N\left|\left|\nabla_{\bm{\theta}}\left[[\bm{a}_{\bm{\theta}}]_i\right]\right|\right|_2,
    \end{align*}
    since $||\bm{B}_{\bm{\theta}}||_F \le \sqrt{\text{rank}(\bm{B}_{\bm{\theta}})}||\bm{B}_{\bm{\theta}}||_2 \le \frac{\sqrt{N}}{\sigma^2_{n}}$. Hence,
    \begin{align*}
        &\sum_{i=1}^N\sum_{j=1}^N||\nabla_{\bm{\theta}}\left[[\bm{a}_{\bm{\theta}}]_i\right]||_2\left|[\bm{a}_{\bm{\theta}}]_j\left[\bm{B}_{\bm{\theta}}\right]_{ij}\right|\le \frac{N\sqrt{N}M_1M_2}{\sigma^2_n}.
    \end{align*}
\end{enumerate}
Combining these results and using $||\nabla_{\bm{\theta}}k_{\bm{\theta}}(\bm{x},\bm{x})|| \le M_2$,  $\left|\sigma_{\bm{\theta}}(\bm{x}|\mathcal{D})\right| \ge k_{\bm{\theta}}(\bm{x},\bm{x}) \ge M_0$, we have
\begin{equation}
\label{norm_posterior_deviat_result}
    \left|\left|\nabla_{\bm{\theta}}\left[\sigma_{\bm{\theta}}(\bm{x}|\mathcal{D})\right]\right|\right|_2 \le \frac{N\sqrt{N}M_1M_2}{2\sigma^2_nM_0}\left[\frac{N^2\sqrt{p}M_1}{\sigma^2_n}+2\right]
\end{equation}
Hence, combining (\ref{norm_posterior_mean_result}) and (\ref{norm_posterior_deviat_result})  we have
\begin{align*}
    &||\bm{h}_{\bm{\theta}}(\bm{x}|\mathcal{D})||_2\le \nonumber
    ||\nabla_{\boldsymbol{\theta}}\mu_{\bm{\theta}}(\bm{x}|\mathcal{D})||_2 + \sqrt{\frac{\beta\pi}{2}}\left|\left|\nabla_{\bm{\theta}}\left[\sigma_{\bm{\theta}}(\bm{x}|\mathcal{D})\right]\right|\right|_2 \le\\\nonumber
    &\frac{(C+M_4)N^2M_2}{\sigma^2_n}\left[1 + \frac{N^2M_1\sqrt{p}}{\sigma^2_n}\right] + \nonumber
    \sqrt{\frac{\beta\pi}{2}}\frac{N\sqrt{N}M_1M_2}{2\sigma^2_nM_0}\left[\frac{N^2\sqrt{p}M_1}{\sigma^2_n}+2\right] \triangleq A_1.
\end{align*}
Now, we are ready to bound the other two terms in the claim:
\begin{align*}
    &\mu^2_{\bm{\theta}+\bm{\epsilon}}(\bm{x}|\mathcal{D}) \le 2\left[\bm{a}^{\mathsf{T}}_{\bm{\theta}+\bm{\epsilon}}\bm{B}_{\bm{\theta}+\bm{\epsilon}}[\bm{y} - \bm{m}_{\mathcal{D}}]\right]^2 + 2|m|^2 \le \nonumber
    2\frac{(C + M_4)^2M^2_1}{\sigma^4_n} + 2M^2_4
\end{align*}
Therefore, for $\mathbb{E}_{\bm{\epsilon}}\left[\mu^2_{\bm{\theta}+\bm{\epsilon}}(\bm{x}|\mathcal{D})\right]$ we have
\begin{align*}
    &\mathbb{E}_{\bm{\epsilon}}\left[\mu^2_{\bm{\theta}+\bm{\epsilon}}(\bm{x}|\mathcal{D})\right] \le 2\frac{(C + M_4)^2M^2_1}{\sigma^4_n} + 2M^2_4 \triangleq A_2.
\end{align*}
Finally, for the posterior mean
\begin{align*}
    &\sigma^2_{\bm{\theta}+\bm{\epsilon}}(\bm{x}|\mathcal{D}) \le k_{\bm{\theta}}(\bm{x},\bm{x}) + \bm{a}^{\mathsf{T}}_{\bm{\theta}+\bm{\epsilon}}\bm{B}_{\bm{\theta}+\bm{\epsilon}}\bm{a}_{\bm{\theta}+\bm{\epsilon}} \le \nonumber
    M_1 + \frac{M^2_1}{\sigma^2_n}
\end{align*}
Therefore, for $\mathbb{E}_{\bm{\epsilon}}\left[\sigma^2_{\bm{\theta}+\bm{\epsilon}}(\bm{x}|\mathcal{D})\right]$ we have
\begin{align*}
    &\mathbb{E}_{\bm{\epsilon}}\left[\sigma^2_{\bm{\theta}+\bm{\epsilon}}(\bm{x}|\mathcal{D})\right] \le M_1 + \frac{M^2_1}{\sigma^2_n} \triangleq A_3.
\end{align*}
This finishes the proof of the claim.
\end{proof}

\noindent Equipped with these results, we can further simplify the expressions (\ref{Eq:parameter_setup}):
\begin{align*}
    &\sigma_n = \frac{\rho}{4\Phi^{-1}\left(1 - \frac{\delta}{8N_{\epsilon}}\right)},\ \ \nonumber
    \sigma_{\bm{\epsilon}} = \min\left\{1, \frac{\rho}{8\left[2\sqrt{p} + \sqrt{\log\frac{4N_{\epsilon}}{\delta}}\right]\left[A_1 + o(1)\right]}\right\},
\end{align*}
with
\begin{align*}
    &N_{\epsilon} = \left\lceil\frac{16\left[A_2 + \frac{\beta\pi}{2}A_3\right]}{\delta\rho^2}\right\rceil.
\end{align*}
This finishes the proof of the lemma. \\

As such, we may now implement robust formulations of acquisition functions using only the GP predictive mean and variance.

\section{Statistical Hypothesis Tests for Heteroscedasticity}\label{sec:hyp_test_app}

In this section we present the full results for the statistical hypothesis testing using Levene's test and the Fligner-Killeen test in \autoref{tab:search-spacehetero-tests-boston}, \autoref{tab:search-spacehetero-tests-Breast cancer dataset}, \autoref{tab:search-spacehetero-tests-diabetes}, \autoref{tab:search-spacehetero-tests-digits}, \autoref{tab:search-spacehetero-tests-iris} and \autoref{tab:search-spacehetero-tests-wine} for the Boston, breast cancer, diabetes, digits, iris and wine datasets respectively.

\begin{table*}
\centering
\caption{Heteroscedasticity tests on tasks involving \texttt{Boston} data set.}
\label{tab:search-spacehetero-tests-boston}
\begin{tabular}{lllrrrr}
\toprule
Data set & Model & Metric & Fligner Statistic & p-value & Levene Statistic & p-value \\ 
\midrule
\textbf{Boston} & DT & mae & 73.51 & \textbf{0.01327} & 1.752 & \textbf{1.900e-03} \\ 
  & MLP-adam & mae & 336.3 & \textbf{1.737e-44} & 14.4 & \textbf{3.611e-65} \\ 
  & MLP-SGD & mae & 272.6 & \textbf{8.694e-33} & 6.561 & \textbf{1.480e-29} \\ 
  & RF & mae & 28.79 & 0.9906 & 0.6768 & 0.9537 \\ 
  & SVM & mae & 48.08 & 0.5106 & 0.9612 & 0.5508 \\ 
  & ada & mae & 218.7 & \textbf{2.692e-23} & 13.59 & \textbf{5.542e-62} \\ 
  & kNN & mae & 33.15 & 0.9597 & 0.619 & 0.98 \\ 
  & lasso & mae & 30.4 & 0.983 & 0.6091 & 0.983 \\ 
  & linear & mae & 16.17 & 1 & 0.251 & 1 \\ 
  & DT & mse & 60.75 & 0.1211 & 1.33 & 0.07387 \\ 
  & MLP-adam & mse & 387 & \textbf{4.504e-54} & 15.32 & \textbf{1.147e-68} \\ 
  & MLP-SGD & mse & 353.2 & \textbf{1.185e-47} & 8.239 & \textbf{3.548e-38} \\ 
  & RF & mse & 35.59 & 0.9242 & 0.8985 & 0.6692 \\ 
  & SVM & mse & 25.01 & 0.9983 & 0.4491 & 0.9996 \\ 
  & ada & mse & 249.1 & \textbf{1.398e-28} & 14.4 & \textbf{3.682e-65} \\ 
  & kNN & mse & 27.75 & 0.9938 & 0.8247 & 0.7951 \\ 
  & lasso & mse & 31.38 & 0.9764 & 0.5397 & 0.9955 \\ 
  & linear & mse & 16.67 & 1 & 0.1726 & 1 \\ 
\bottomrule
\end{tabular}
\end{table*}

\begin{table*}
\centering
\caption{Heteroscedasticity tests on tasks involving \texttt{Breast cancer} (BC) data set.}
\label{tab:search-spacehetero-tests-Breast cancer dataset}
\begin{tabular}{lllrrrr}
\toprule
Data set & Model & Metric & Fligner Statistic & p-value &  Levene Statistic & p-value \\ 
\midrule
\textbf{BC} & DT & acc & 97.79 & \textbf{4.302e-05} & 4.62 & \textbf{6.650e-19} \\ 
  & MLP-adam & acc & 133 & \textbf{1.113e-09} & 2.939 & \textbf{1.923e-09} \\ 
  & MLP-SGD & acc & 116.8 & \textbf{1.854e-07} & 2.469 & \textbf{6.495e-07} \\ 
  & RF & acc & 154.9 & \textbf{6.469e-13} & 6.661 & \textbf{4.353e-30} \\ 
  & SVM & acc & 20.7 & 0.9999 & 0.3995 & 0.9999 \\ 
  & ada & acc & 272.5 & \textbf{9.178e-33} & 13.57 & \textbf{6.582e-62} \\ 
  & kNN & acc & 33.16 & 0.9596 & 0.5519 & 0.9941 \\ 
  & lasso & acc & 20.78 & 0.9999 & 0.4291 & 0.9998 \\ 
  & linear & acc & 21.15 & 0.9998 & 0.4545 & 0.9995 \\ 
  & DT & nll & 260.5 & \textbf{1.280e-30} & 9.52 & \textbf{2.294e-44} \\ 
  & MLP-adam & nll & 166.6 & \textbf{1.008e-14} & 3.643 & \textbf{2.247e-13} \\ 
  & MLP-SGD & nll & 141.2 & \textbf{7.115e-11} & 2.669 & \textbf{5.661e-08} \\ 
  & RF & nll & 185.8 & \textbf{8.495e-18} & 7.553 & \textbf{1.013e-34} \\ 
  & SVM & nll & 76.98 & \textbf{6.526e-03} & 1.707 & \textbf{2.970e-03} \\ 
  & ada & nll & 142 & \textbf{5.458e-11} & 4.283 & \textbf{5.274e-17} \\ 
  & kNN & nll & 125.7 & \textbf{1.155e-08} & 4.337 & \textbf{2.635e-17} \\ 
  & lasso & nll & 71.41 & \textbf{0.02} & 1.011 & 0.4565 \\ 
  & linear & nll & 18.55 & 1 & 0.2714 & 1 \\ 
\bottomrule
\end{tabular}
\end{table*}

\begin{table*}
\centering
\caption{Heteroscedasticity tests on tasks involving \texttt{diabetes} data set.}
\label{tab:search-spacehetero-tests-diabetes}
\begin{tabular}{lllrrrr}
\toprule
Dataset & Model & Metric & Fligner Statistic & p-value & Levene Statistic & p-value \\ 
\midrule
\textbf{Diabetes} & DT & mae & 56.52 & 0.2146 & 1.131 & 0.2601 \\ 
  & MLP-adam & mae & 74.64 & \textbf{0.01059} & 2.573 & \textbf{1.847e-07} \\ 
  & MLP-SGD & mae & 191.3 & \textbf{1.062e-18} & 17.87 & \textbf{8.498e-78} \\ 
  & RF & mae & 79.38 & \textbf{3.898e-03} & 1.558 & \textbf{0.01174} \\ 
  & SVM & mae & 2.436 & 1 & 1.810e-04 & 1 \\ 
  & ada & mae & 179.8 & \textbf{7.883e-17} & 7.542 & \textbf{1.154e-34} \\ 
  & kNN & mae & 67.48 & \textbf{0.04106} & 2.101 & \textbf{4.747e-05} \\ 
  & lasso & mae & 176.2 & \textbf{2.950e-16} & 4.75 & \textbf{1.225e-19} \\ 
  & linear & mae & 206 & \textbf{3.792e-21} & 5.714 & \textbf{5.490e-25} \\ 
  & DT & mse & 44.52 & 0.6551 & 0.8264 & 0.7925 \\ 
  & MLP-adam & mse & 100.4 & \textbf{2.109e-05} & 3.582 & \textbf{4.951e-13} \\ 
  & MLP-SGD & mse & 202.9 & \textbf{1.257e-20} & 14.31 & \textbf{7.960e-65} \\ 
  & RF & mse & 37.1 & 0.8938 & 0.8063 & 0.8224 \\ 
  & SVM & mse & 4.004 & 1 & 4.740e-04 & 1 \\ 
  & ada & mse & 189 & \textbf{2.510e-18} & 7.348 & \textbf{1.138e-33} \\ 
  & kNN & mse & 88.62 & \textbf{4.545e-04} & 2.964 & \textbf{1.407e-09} \\ 
  & lasso & mse & 257.6 & \textbf{4.341e-30} & 10.86 & \textbf{1.637e-50} \\ 
  & linear & mse & 278.2 & \textbf{8.540e-34} & 10.01 & \textbf{1.216e-46} \\ 
\bottomrule
\end{tabular}
\end{table*}

\begin{table*}
\centering
\caption{Heteroscedasticity tests on tasks involving \texttt{digits} data set.}
\label{tab:search-spacehetero-tests-digits}
\begin{tabular}{lllrrrr}
\toprule
Data set & Model & Metric & Fligner Statistic & p-value & Levene Statistic & p-value \\ 
\midrule
\textbf{Digits} & DT & acc & 205 & \textbf{5.670e-21} & 14.29 & \textbf{9.219e-65} \\ 
  & MLP-adam & acc & 256.7 & \textbf{6.239e-30} & 7.342 & \textbf{1.219e-33} \\ 
  & MLP-SGD & acc & 210 & \textbf{8.188e-22} & 6.53 & \textbf{2.167e-29} \\ 
  & RF & acc & 184.3 & \textbf{1.458e-17} & 15.61 & \textbf{9.379e-70} \\ 
  & SVM & acc & 91.72 & \textbf{2.093e-04} & 2.187 & \textbf{1.790e-05} \\ 
  & ada & acc & 99.34 & \textbf{2.832e-05} & 2.305 & \textbf{4.601e-06} \\ 
  & kNN & acc & 35 & 0.9343 & 0.7042 & 0.9349 \\ 
  & lasso & acc & 22.97 & 0.9994 & 0.4292 & 0.9998 \\ 
  & linear & acc & 17.3 & 1 & 0.2963 & 1 \\ 
  & DT & nll & 249.6 & \textbf{1.140e-28} & 15.71 & \textbf{3.892e-70} \\ 
  & MLP-adam & nll & 339.8 & \textbf{3.816e-45} & 6.882 & \textbf{3.012e-31} \\ 
  & MLP-SGD & nll & 244.8 & \textbf{7.740e-28} & 6.104 & \textbf{4.129e-27} \\ 
  & RF & nll & 144 & \textbf{2.791e-11} & 7.435 & \textbf{4.059e-34} \\ 
  & SVM & nll & 4.373 & 1 & 0.06091 & 1 \\ 
  & ada & nll & 135.1 & \textbf{5.444e-10} & 3.294 & \textbf{2.061e-11} \\ 
  & kNN & nll & 108.2 & \textbf{2.326e-06} & 3.059 & \textbf{4.211e-10} \\ 
  & lasso & nll & 88.4 & \textbf{4.799e-04} & 2.116 & \textbf{3.995e-05} \\ 
  & linear & nll & 103 & \textbf{1.024e-05} & 3.328 & \textbf{1.335e-11} \\ 
\bottomrule
\end{tabular}
\end{table*}

\begin{table*}
\centering
\caption{Heteroscedasticity tests on tasks involving \texttt{iris} data set.}
\label{tab:search-spacehetero-tests-iris}
\begin{tabular}{lllrrrr}
\toprule
Data set & Model & Metric & Fligner Statistic & p-value & Levene Statistic & p-value \\ 
\midrule
\textbf{Iris} & DT & acc & 207.1 & \textbf{2.440e-21} & 6.523 & \textbf{2.355e-29} \\ 
  & MLP-adam & acc & 83.81 & \textbf{1.436e-03} & 1.838 & \textbf{7.989e-04} \\ 
  & MLP-SGD & acc & 68.52 & \textbf{0.03413} & 1.409 & \textbf{0.04082} \\ 
  & RF & acc & 155.5 & \textbf{5.311e-13} & 6.138 & \textbf{2.726e-27} \\ 
  & SVM & acc & 198.4 & \textbf{6.990e-20} & 3.345 & \textbf{1.065e-11} \\ 
  & ada & acc & 155.7 & \textbf{4.788e-13} & 5.018 & \textbf{3.858e-21} \\ 
  & kNN & acc & 55.68 & 0.2378 & 1.124 & 0.2701 \\ 
  & lasso & acc & 19.72 & 0.9999 & 0.4045 & 0.9999 \\ 
  & linear & acc & 106.4 & \textbf{3.965e-06} & 2.959 & \textbf{1.502e-09} \\ 
  & DT & nll & 322.2 & \textbf{7.375e-42} & 6.118 & \textbf{3.506e-27} \\ 
  & MLP-adam & nll & 106.3 & \textbf{4.070e-06} & 3.123 & \textbf{1.869e-10} \\ 
  & MLP-SGD & nll & 155.6 & \textbf{4.966e-13} & 6.386 & \textbf{1.264e-28} \\ 
  & RF & nll & 321.3 & \textbf{1.066e-41} & 8.339 & \textbf{1.136e-38} \\ 
  & SVM & nll & 188.4 & \textbf{3.217e-18} & 4.736 & \textbf{1.470e-19} \\ 
  & ada & nll & 74.04 & \textbf{0.01194} & 1.414 & \textbf{0.03938} \\ 
  & kNN & nll & 212.6 & \textbf{2.863e-22} & 8.838 & \textbf{4.118e-41} \\ 
  & lasso & nll & 45.45 & 0.6177 & 0.5045 & 0.998 \\ 
  & linear & nll & 36.64 & 0.9037 & 0.733 & 0.9101 \\ 
\bottomrule
\end{tabular}
\end{table*}

\begin{table*}
\centering
\caption{Heteroscedasticity tests on tasks involving the \texttt{wine} data set.}
\label{tab:search-spacehetero-tests-wine}
\begin{tabular}{lllrrrr}
\toprule
Data set & Model & Metric & Fligner Statistic & p-value &  Levene Statistic & p-value \\ 
\midrule
\textbf{Wine} & DT & acc & 127.3 & \textbf{6.912e-09} & 3.553 & \textbf{7.195e-13} \\ 
  & MLP-adam & acc & 85.37 & \textbf{9.945e-04} & 1.874 & \textbf{5.544e-04} \\ 
  & MLP-SGD & acc & 109 & \textbf{1.845e-06} & 2.48 & \textbf{5.701e-07} \\ 
  & RF & acc & 128.5 & \textbf{4.717e-09} & 5.069 & \textbf{2.014e-21} \\ 
  & SVM & acc & 28.73 & 0.9908 & 0.5136 & 0.9975 \\ 
  & ada & acc & 156.6 & \textbf{3.527e-13} & 3.968 & \textbf{3.215e-15} \\ 
  & kNN & acc & 37.67 & 0.8807 & 0.6869 & 0.9473 \\ 
  & lasso & acc & 29.8 & 0.9862 & 0.5981 & 0.9859 \\ 
  & linear & acc & 21.28 & 0.9998 & 0.3839 & 1 \\ 
  & DT & nll & 349.2 & \textbf{6.614e-47} & 10.46 & \textbf{1.115e-48} \\ 
  & MLP-adam & nll & 57.19 & 0.1971 & 1.21 & 0.1646 \\ 
  & MLP-SGD & nll & 110.1 & \textbf{1.362e-06} & 2.597 & \textbf{1.380e-07} \\ 
  & RF & nll & 258 & \textbf{3.660e-30} & 6.468 & \textbf{4.597e-29} \\ 
  & SVM & nll & 57.18 & 0.1975 & 1.006 & 0.4663 \\ 
  & ada & nll & 152.8 & \textbf{1.323e-12} & 3.072 & \textbf{3.555e-10} \\ 
  & kNN & nll & 178.2 & \textbf{1.410e-16} & 5.446 & \textbf{1.635e-23} \\ 
  & lasso & nll & 83.94 & \textbf{1.394e-03} & 1.782 & \textbf{1.416e-03} \\ 
  & linear & nll & 185.8 & \textbf{8.404e-18} & 5.01 & \textbf{4.312e-21} \\ 
\bottomrule
\end{tabular}
\end{table*}

\section{Task-Level Results Breakdown}\label{fig:summary_all_models}

In this section we present the full task-level breakdown of the results with each metric, data set and model combination for each black-box optimiser summarised with the mean and variance achieved across 20 seeds. We show a summary plot in Table~\ref{tab:tasksSummary}.

\begin{table}[h!]
\centering
\caption{Number of tasks for which each optimiser performed best.}
\label{tab:tasksSummary}
\resizebox{\linewidth}{!}{%
\begin{tabular}{rrrrrrrrr}
\toprule
 HEBO &  TuRBO &  PySOT &  Skopt &  Nevergrad (1+1)&  BOHB-BB &  Opentuner &  Hyperopt &  TuRBO+ \\
\midrule
  71 (65.7\%) &     14 (13.0\%) &      7 (6.5 \%) &      5  (4.6 \%)&          4 (3.7 \%)&     3 (2.8\%) &          2 (1.9\%) &         1 (0.9\%) &       1 (0.9\%) \\
\bottomrule
\end{tabular}}
\end{table}

We now present sequentially the full results for each of the 6 datasets: Boston (Section~\ref{sec:boston}), Breast cancer (Section~\ref{sec:Breast cancer dataset}), Diabetes (Section~\ref{sec:diabetes}), Digits (Section~\ref{sec:digits}), Iris (Section~\ref{sec:iris}), Wine (Section~\ref{sec:wine}). For each optimiser we give the mean and variance of the performance metric across all 18 tasks (2 metrics x 9 models) for a given data set.

\newpage

\subsection{Boston Data Set}\label{sec:boston}

\begin{table}[h!]
\centering
\caption{Boston with MAE loss for tuning DT model, averaged over 20 seeds. Optimiser for this task with highest mean normalised score is HEBO.}
\begin{tabular}{llllrr}
\toprule
    Algorithm & Dataset & Model & Metric &  Normalised Score &   Variance \\
\midrule
         HEBO &  Boston &    DT &    MAE &        106.658094 &   3.514569 \\
        PySOT &  Boston &    DT &    MAE &        104.811444 &   4.220520 \\
        Skopt &  Boston &    DT &    MAE &        104.303237 &   4.311549 \\
        TuRBO &  Boston &    DT &    MAE &        103.863640 &   6.184650 \\
    Nevergrad (1+1)&  Boston &    DT &    MAE &        102.938925 &  17.667195 \\
     Hyperopt &  Boston &    DT &    MAE &         99.231898 &   5.681478 \\
      TuRBO+ &  Boston &    DT &    MAE &         98.464758 &  35.071496 \\
Random-search &  Boston &    DT &    MAE &         95.273667 &  21.448475 \\
         BOHB-BB &  Boston &    DT &    MAE &         93.360381 &  37.490514 \\
    Opentuner &  Boston &    DT &    MAE &         86.421660 & 313.239599 \\
\bottomrule
\end{tabular}
\end{table}

\begin{table}[h!]
\centering
\caption{Boston with MAE loss for tuning MLP-adam model, averaged over 20 seeds. Optimiser for this task with highest mean normalised score is TuRBO.}
\begin{tabular}{llllrr}
\toprule
    Algorithm & Dataset &    Model & Metric &  Normalised Score &  Variance \\
\midrule
        TuRBO &  Boston & MLP-adam &    MAE &        104.423509 & 17.216156 \\
        PySOT &  Boston & MLP-adam &    MAE &        103.369026 & 15.037764 \\
         HEBO &  Boston & MLP-adam &    MAE &        101.377994 & 31.981786 \\
      TuRBO+ &  Boston & MLP-adam &    MAE &         98.759781 & 46.065189 \\
    Nevergrad (1+1)&  Boston & MLP-adam &    MAE &         96.279956 & 45.560329 \\
     Hyperopt &  Boston & MLP-adam &    MAE &         95.255250 & 28.728828 \\
Random-search &  Boston & MLP-adam &    MAE &         92.866456 & 23.342156 \\
         BOHB-BB &  Boston & MLP-adam &    MAE &         92.001680 & 17.523725 \\
    Opentuner &  Boston & MLP-adam &    MAE &         89.595926 & 47.161525 \\
        Skopt &  Boston & MLP-adam &    MAE &         86.257482 & 24.245633 \\
\bottomrule
\end{tabular}
\end{table}

\begin{table}[h!]
\centering
\caption{Boston with MAE loss for tuning MLP-SGD model, averaged over 20 seeds. Optimiser for this task with highest mean normalised score is PySOT.}
\begin{tabular}{llllrr}
\toprule
    Algorithm & Dataset &   Model & Metric &  Normalised Score &   Variance \\
\midrule
        PySOT &  Boston & MLP-SGD &    MAE &         96.325329 &  34.592455 \\
         HEBO &  Boston & MLP-SGD &    MAE &         95.885011 &   4.520840 \\
        TuRBO &  Boston & MLP-SGD &    MAE &         94.936543 &  21.232934 \\
     Hyperopt &  Boston & MLP-SGD &    MAE &         93.605852 &   4.184636 \\
        Skopt &  Boston & MLP-SGD &    MAE &         91.084570 &   8.855305 \\
Random-search &  Boston & MLP-SGD &    MAE &         89.854633 &   9.131124 \\
         BOHB-BB &  Boston & MLP-SGD &    MAE &         89.622364 &  20.485508 \\
      TuRBO+ &  Boston & MLP-SGD &    MAE &         89.288302 &  18.617165 \\
    Nevergrad (1+1)&  Boston & MLP-SGD &    MAE &         87.149541 & 115.487770 \\
    Opentuner &  Boston & MLP-SGD &    MAE &         85.130342 &  26.218217 \\
\bottomrule
\end{tabular}
\end{table}

\begin{table}[h!]
\centering
\caption{Boston with MAE loss for tuning RF model, averaged over 20 seeds. Optimiser for this task with highest mean normalised score is HEBO.}
\begin{tabular}{llllrr}
\toprule
    Algorithm & Dataset & Model & Metric &  Normalised Score &  Variance \\
\midrule
         HEBO &  Boston &    RF &    MAE &        101.726637 &  0.282878 \\
        TuRBO &  Boston &    RF &    MAE &        100.095822 &  2.339170 \\
        PySOT &  Boston &    RF &    MAE &         99.773979 &  1.147489 \\
        Skopt &  Boston &    RF &    MAE &         99.352605 &  0.805112 \\
     Hyperopt &  Boston &    RF &    MAE &         98.142703 &  1.649156 \\
    Nevergrad (1+1)&  Boston &    RF &    MAE &         96.790676 & 20.349353 \\
      TuRBO+ &  Boston &    RF &    MAE &         95.817383 &  7.005061 \\
    Opentuner &  Boston &    RF &    MAE &         94.012783 & 32.962245 \\
         BOHB-BB &  Boston &    RF &    MAE &         93.869871 &  6.026597 \\
Random-search &  Boston &    RF &    MAE &         93.223774 &  8.626423 \\
\bottomrule
\end{tabular}
\end{table}

\begin{table}[h!]
\centering
\caption{Boston with MAE loss for tuning SVM model, averaged over 20 seeds. Optimiser for this task with highest mean normalised score is HEBO.}
\begin{tabular}{llllrr}
\toprule
    Algorithm & Dataset & Model & Metric &  Normalised Score &   Variance \\
\midrule
         HEBO &  Boston &   SVM &    MAE &        102.469764 &   0.015274 \\
    Nevergrad (1+1)&  Boston &   SVM &    MAE &        102.236442 &   0.135306 \\
        TuRBO &  Boston &   SVM &    MAE &        101.826793 &   0.169932 \\
        PySOT &  Boston &   SVM &    MAE &        101.262466 &   0.082329 \\
        Skopt &  Boston &   SVM &    MAE &        100.914194 &   0.281085 \\
     Hyperopt &  Boston &   SVM &    MAE &         98.943800 &   9.092074 \\
    Opentuner &  Boston &   SVM &    MAE &         97.906378 &  12.428820 \\
      TuRBO+ &  Boston &   SVM &    MAE &         90.935974 & 182.645481 \\
Random-search &  Boston &   SVM &    MAE &         88.908494 &  80.396927 \\
         BOHB-BB &  Boston &   SVM &    MAE &         86.714806 &  46.398841 \\
\bottomrule
\end{tabular}
\end{table}

\begin{table}[h!]
\centering
\caption{Boston with MAE loss for tuning Ada model, averaged over 20 seeds. Optimiser for this task with highest mean normalised score is HEBO.}
\begin{tabular}{llllrr}
\toprule
    Algorithm & Dataset & Model & Metric &  Normalised Score &   Variance \\
\midrule
         HEBO &  Boston &   Ada &    MAE &        103.445454 &  19.376575 \\
    Opentuner &  Boston &   Ada &    MAE &        101.293346 &  54.800061 \\
        TuRBO &  Boston &   Ada &    MAE &         99.774003 &  84.126539 \\
    Nevergrad (1+1)&  Boston &   Ada &    MAE &         97.261376 & 394.018233 \\
        Skopt &  Boston &   Ada &    MAE &         97.141270 &  48.739222 \\
      TuRBO+ &  Boston &   Ada &    MAE &         97.035052 & 178.220595 \\
     Hyperopt &  Boston &   Ada &    MAE &         93.394811 &  59.032286 \\
        PySOT &  Boston &   Ada &    MAE &         91.618187 &  25.548574 \\
Random-search &  Boston &   Ada &    MAE &         85.889844 &  39.688171 \\
         BOHB-BB &  Boston &   Ada &    MAE &         84.104672 &  99.526333 \\
\bottomrule
\end{tabular}
\end{table}

\begin{table}[h!]
\centering
\caption{Boston with MAE loss for tuning Knn model, averaged over 20 seeds. Optimiser for this task with highest mean normalised score is HEBO.}
\begin{tabular}{llllrr}
\toprule
    Algorithm & Dataset & Model & Metric &  Normalised Score &  Variance \\
\midrule
         HEBO &  Boston &   Knn &    MAE &        100.000000 &  0.000000 \\
    Opentuner &  Boston &   Knn &    MAE &        100.000000 &  0.000000 \\
        Skopt &  Boston &   Knn &    MAE &        100.000000 &  0.000000 \\
        TuRBO &  Boston &   Knn &    MAE &        100.000000 &  0.000000 \\
      TuRBO+ &  Boston &   Knn &    MAE &         99.689888 &  1.923386 \\
        PySOT &  Boston &   Knn &    MAE &         99.069665 &  5.162772 \\
     Hyperopt &  Boston &   Knn &    MAE &         98.023553 &  9.829689 \\
    Nevergrad (1+1)&  Boston &   Knn &    MAE &         97.379434 & 28.707287 \\
         BOHB-BB &  Boston &   Knn &    MAE &         97.134442 & 18.411664 \\
Random-search &  Boston &   Knn &    MAE &         96.297696 & 45.566838 \\
\bottomrule
\end{tabular}
\end{table}

\begin{table}[h!]
\centering
\caption{Boston with MAE loss for tuning Lasso model, averaged over 20 seeds. Optimiser for this task with highest mean normalised score is HEBO.}
\begin{tabular}{llllrr}
\toprule
    Algorithm & Dataset & Model & Metric &  Normalised Score &  Variance \\
\midrule
         HEBO &  Boston & Lasso &    MAE &        100.031658 &  0.000141 \\
        PySOT &  Boston & Lasso &    MAE &        100.004491 &  0.001119 \\
    Nevergrad (1+1)&  Boston & Lasso &    MAE &         99.989103 &  0.005199 \\
        Skopt &  Boston & Lasso &    MAE &         99.979143 &  0.001668 \\
        TuRBO &  Boston & Lasso &    MAE &         99.977645 &  0.002141 \\
     Hyperopt &  Boston & Lasso &    MAE &         99.967041 &  0.002263 \\
         BOHB-BB &  Boston & Lasso &    MAE &         99.926445 &  0.006747 \\
      TuRBO+ &  Boston & Lasso &    MAE &         99.921337 &  0.031836 \\
Random-search &  Boston & Lasso &    MAE &         99.917670 &  0.002902 \\
    Opentuner &  Boston & Lasso &    MAE &         99.306403 &  0.956136 \\
\bottomrule
\end{tabular}
\end{table}

\begin{table}[h!]
\centering
\caption{Boston with MAE loss for tuning Linear model, averaged over 20 seeds. Optimiser for this task with highest mean normalised score is HEBO.}
\begin{tabular}{llllrr}
\toprule
    Algorithm & Dataset &  Model & Metric &  Normalised Score &    Variance \\
\midrule
         HEBO &  Boston & Linear &    MAE &         99.964657 &    0.002542 \\
     Hyperopt &  Boston & Linear &    MAE &         99.922142 &    0.006819 \\
        TuRBO &  Boston & Linear &    MAE &         99.695751 &    1.213154 \\
        PySOT &  Boston & Linear &    MAE &         99.394823 &    0.799652 \\
         BOHB-BB &  Boston & Linear &    MAE &         98.547907 &    5.197760 \\
        Skopt &  Boston & Linear &    MAE &         98.327627 &    6.150494 \\
Random-search &  Boston & Linear &    MAE &         97.843101 &    5.252852 \\
      TuRBO+ &  Boston & Linear &    MAE &         95.874716 &  153.782738 \\
    Nevergrad (1+1)&  Boston & Linear &    MAE &         80.996813 & 1523.062377 \\
    Opentuner &  Boston & Linear &    MAE &         45.221227 & 2080.382131 \\
\bottomrule
\end{tabular}
\end{table}

\begin{table}[h!]
\centering
\caption{Boston with MSE loss for tuning DT model, averaged over 20 seeds. Optimiser for this task with highest mean normalised score is Skopt.}
\begin{tabular}{llllrr}
\toprule
    Algorithm & Dataset & Model & Metric &  Normalised Score &   Variance \\
\midrule
        Skopt &  Boston &    DT &    MSE &        105.587343 &  12.872983 \\
        PySOT &  Boston &    DT &    MSE &        105.072075 &   4.900371 \\
         HEBO &  Boston &    DT &    MSE &        104.838117 &   6.188087 \\
        TuRBO &  Boston &    DT &    MSE &        104.030707 &  10.799242 \\
     Hyperopt &  Boston &    DT &    MSE &        102.229770 &  11.149492 \\
      TuRBO+ &  Boston &    DT &    MSE &        100.450728 &  12.162198 \\
    Nevergrad (1+1)&  Boston &    DT &    MSE &         97.629442 & 407.356319 \\
Random-search &  Boston &    DT &    MSE &         95.359894 &  28.675740 \\
         BOHB-BB &  Boston &    DT &    MSE &         94.520125 &  20.675992 \\
    Opentuner &  Boston &    DT &    MSE &         87.826986 & 205.761429 \\
\bottomrule
\end{tabular}
\end{table}

\begin{table}[h!]
\centering
\caption{Boston with MSE loss for tuning MLP-adam model, averaged over 20 seeds. Optimiser for this task with highest mean normalised score is HEBO.}
\begin{tabular}{llllrr}
\toprule
    Algorithm & Dataset &    Model & Metric &  Normalised Score &  Variance \\
\midrule
         HEBO &  Boston & MLP-adam &    MSE &        101.227304 & 15.854379 \\
        TuRBO &  Boston & MLP-adam &    MSE &        100.906498 &  4.709893 \\
        PySOT &  Boston & MLP-adam &    MSE &         98.878397 &  7.952874 \\
      TuRBO+ &  Boston & MLP-adam &    MSE &         98.090398 & 17.299141 \\
    Nevergrad (1+1)&  Boston & MLP-adam &    MSE &         97.631386 & 45.709714 \\
     Hyperopt &  Boston & MLP-adam &    MSE &         95.833644 & 15.829315 \\
         BOHB-BB &  Boston & MLP-adam &    MSE &         95.267050 & 21.466780 \\
Random-search &  Boston & MLP-adam &    MSE &         93.691552 & 18.033615 \\
    Opentuner &  Boston & MLP-adam &    MSE &         92.489905 & 52.692896 \\
        Skopt &  Boston & MLP-adam &    MSE &         87.610428 & 15.511744 \\
\bottomrule
\end{tabular}
\end{table}

\begin{table}[h!]
\centering
\caption{Boston with MSE loss for tuning MLP-SGD model, averaged over 20 seeds. Optimiser for this task with highest mean normalised score is TuRBO.}
\begin{tabular}{llllrr}
\toprule
    Algorithm & Dataset &   Model & Metric &  Normalised Score &   Variance \\
\midrule
        TuRBO &  Boston & MLP-SGD &    MSE &        104.817056 &  17.856834 \\
         HEBO &  Boston & MLP-SGD &    MSE &        103.285788 &   4.467401 \\
     Hyperopt &  Boston & MLP-SGD &    MSE &        102.049189 &   6.848005 \\
        PySOT &  Boston & MLP-SGD &    MSE &         99.992157 &  17.117778 \\
      TuRBO+ &  Boston & MLP-SGD &    MSE &         98.973335 &  45.162334 \\
        Skopt &  Boston & MLP-SGD &    MSE &         98.293397 &   6.746134 \\
         BOHB-BB &  Boston & MLP-SGD &    MSE &         94.731025 &  17.820451 \\
Random-search &  Boston & MLP-SGD &    MSE &         94.524667 &  28.931228 \\
    Nevergrad (1+1)&  Boston & MLP-SGD &    MSE &         94.267722 & 207.308935 \\
    Opentuner &  Boston & MLP-SGD &    MSE &         88.380591 &  58.962581 \\
\bottomrule
\end{tabular}
\end{table}

\begin{table}[h!]
\centering
\caption{Boston with MSE loss for tuning RF model, averaged over 20 seeds. Optimiser for this task with highest mean normalised score is HEBO.}
\begin{tabular}{llllrr}
\toprule
    Algorithm & Dataset & Model & Metric &  Normalised Score &  Variance \\
\midrule
         HEBO &  Boston &    RF &    MSE &        103.461058 &  0.440697 \\
        PySOT &  Boston &    RF &    MSE &        101.853905 &  0.492860 \\
        TuRBO &  Boston &    RF &    MSE &        101.839078 &  0.800142 \\
        Skopt &  Boston &    RF &    MSE &        101.472976 &  0.606579 \\
    Nevergrad (1+1)&  Boston &    RF &    MSE &        100.814244 &  2.514980 \\
     Hyperopt &  Boston &    RF &    MSE &        100.547643 &  0.964402 \\
      TuRBO+ &  Boston &    RF &    MSE &         98.915015 &  3.362617 \\
Random-search &  Boston &    RF &    MSE &         98.167459 &  4.278077 \\
    Opentuner &  Boston &    RF &    MSE &         98.049981 & 11.073665 \\
         BOHB-BB &  Boston &    RF &    MSE &         97.839608 &  6.645583 \\
\bottomrule
\end{tabular}
\end{table}

\begin{table}[h!]
\centering
\caption{Boston with MSE loss for tuning SVM model, averaged over 20 seeds. Optimiser for this task with highest mean normalised score is HEBO.}
\begin{tabular}{llllrr}
\toprule
    Algorithm & Dataset & Model & Metric &  Normalised Score &   Variance \\
\midrule
         HEBO &  Boston &   SVM &    MSE &        104.852372 &   0.005667 \\
    Nevergrad (1+1)&  Boston &   SVM &    MSE &        104.578053 &   0.288069 \\
        TuRBO &  Boston &   SVM &    MSE &        103.821417 &   1.334006 \\
        PySOT &  Boston &   SVM &    MSE &        103.398778 &   0.101558 \\
     Hyperopt &  Boston &   SVM &    MSE &        101.271084 &   3.903429 \\
        Skopt &  Boston &   SVM &    MSE &        100.753282 &  76.344135 \\
      TuRBO+ &  Boston &   SVM &    MSE &         97.946333 &  50.658268 \\
    Opentuner &  Boston &   SVM &    MSE &         91.009871 & 467.963686 \\
         BOHB-BB &  Boston &   SVM &    MSE &         90.350500 & 127.546283 \\
Random-search &  Boston &   SVM &    MSE &         83.204838 & 135.308788 \\
\bottomrule
\end{tabular}
\end{table}

\begin{table}[h!]
\centering
\caption{Boston with MSE loss for tuning Ada model, averaged over 20 seeds. Optimiser for this task with highest mean normalised score is HEBO.}
\begin{tabular}{llllrr}
\toprule
    Algorithm & Dataset & Model & Metric &  Normalised Score &   Variance \\
\midrule
         HEBO &  Boston &   Ada &    MSE &         96.942494 &  18.658049 \\
    Opentuner &  Boston &   Ada &    MSE &         94.274205 &  34.779775 \\
        TuRBO &  Boston &   Ada &    MSE &         94.131115 &  43.476954 \\
    Nevergrad (1+1)&  Boston &   Ada &    MSE &         92.333218 & 117.061847 \\
      TuRBO+ &  Boston &   Ada &    MSE &         91.017786 &  86.796910 \\
     Hyperopt &  Boston &   Ada &    MSE &         89.721591 &  41.637778 \\
        PySOT &  Boston &   Ada &    MSE &         88.996278 &  29.924305 \\
        Skopt &  Boston &   Ada &    MSE &         87.835681 &  55.855654 \\
         BOHB-BB &  Boston &   Ada &    MSE &         82.947338 &  67.514235 \\
Random-search &  Boston &   Ada &    MSE &         80.119051 &  17.165168 \\
\bottomrule
\end{tabular}
\end{table}

\begin{table}[h!]
\centering
\caption{Boston with MSE loss for tuning Knn model, averaged over 20 seeds. Optimiser for this task with highest mean normalised score is HEBO.}
\begin{tabular}{llllrr}
\toprule
    Algorithm & Dataset & Model & Metric &  Normalised Score &   Variance \\
\midrule
         HEBO &  Boston &   Knn &    MSE &        100.000000 &   0.000000 \\
    Opentuner &  Boston &   Knn &    MSE &        100.000000 &   0.000000 \\
        PySOT &  Boston &   Knn &    MSE &        100.000000 &   0.000000 \\
        TuRBO &  Boston &   Knn &    MSE &        100.000000 &   0.000000 \\
      TuRBO+ &  Boston &   Knn &    MSE &         99.704490 &   1.253051 \\
     Hyperopt &  Boston &   Knn &    MSE &         99.319467 &   6.373596 \\
Random-search &  Boston &   Knn &    MSE &         99.068697 &   7.272026 \\
        Skopt &  Boston &   Knn &    MSE &         99.014267 &  19.433403 \\
         BOHB-BB &  Boston &   Knn &    MSE &         96.129465 &  62.658638 \\
    Nevergrad (1+1)&  Boston &   Knn &    MSE &         88.882774 & 227.331571 \\
\bottomrule
\end{tabular}
\end{table}

\begin{table}[h!]
\centering
\caption{Boston with MSE loss for tuning Lasso model, averaged over 20 seeds. Optimiser for this task with highest mean normalised score is HEBO.}
\begin{tabular}{llllrr}
\toprule
    Algorithm & Dataset & Model & Metric &  Normalised Score &  Variance \\
\midrule
         HEBO &  Boston & Lasso &    MSE &        100.037610 &  0.000010 \\
        PySOT &  Boston & Lasso &    MSE &        100.008571 &  0.000269 \\
     Hyperopt &  Boston & Lasso &    MSE &        100.007350 &  0.000123 \\
        TuRBO &  Boston & Lasso &    MSE &         99.982366 &  0.021257 \\
        Skopt &  Boston & Lasso &    MSE &         99.806671 &  0.433679 \\
         BOHB-BB &  Boston & Lasso &    MSE &         99.679375 &  0.206443 \\
    Nevergrad (1+1)&  Boston & Lasso &    MSE &         99.372815 &  4.020588 \\
Random-search &  Boston & Lasso &    MSE &         99.325799 &  0.769118 \\
      TuRBO+ &  Boston & Lasso &    MSE &         98.971017 &  2.235443 \\
    Opentuner &  Boston & Lasso &    MSE &         97.388970 &  6.008643 \\
\bottomrule
\end{tabular}
\end{table}

\begin{table}[h!]
\centering
\caption{Boston with MSE loss for tuning Linear model, averaged over 20 seeds. Optimiser for this task with highest mean normalised score is HEBO.}
\begin{tabular}{llllrr}
\toprule
    Algorithm & Dataset &  Model & Metric &  Normalised Score &     Variance \\
\midrule
         HEBO &  Boston & Linear &    MSE &        100.000001 & 3.207627e-10 \\
     Hyperopt &  Boston & Linear &    MSE &         99.999348 & 1.340314e-06 \\
        PySOT &  Boston & Linear &    MSE &         99.994196 & 1.378338e-04 \\
        Skopt &  Boston & Linear &    MSE &         99.986643 & 6.872575e-04 \\
    Nevergrad (1+1)&  Boston & Linear &    MSE &         99.976999 & 5.017296e-03 \\
        TuRBO &  Boston & Linear &    MSE &         99.889263 & 2.451020e-01 \\
         BOHB-BB &  Boston & Linear &    MSE &         99.837413 & 2.455636e-01 \\
Random-search &  Boston & Linear &    MSE &         99.802681 & 2.731228e-01 \\
      TuRBO+ &  Boston & Linear &    MSE &         99.332841 & 1.085780e+00 \\
    Opentuner &  Boston & Linear &    MSE &         98.545654 & 1.344348e+00 \\
\bottomrule
\end{tabular}
\end{table}

\clearpage
\subsection{Breast Cancer Data set}\label{sec:Breast cancer dataset}

\begin{table}[h!]
\centering
\caption{Breast cancer dataset with ACC loss for tuning DT model, averaged over 20 seeds. Optimiser for this task with highest mean normalised score is HEBO.}
\begin{tabular}{llllrr}
\toprule
    Algorithm & Dataset & Model & Metric &  Normalised Score &  Variance \\
\midrule
         HEBO &  Breast cancer dataset &    DT &    ACC &        106.694915 & 52.457703 \\
        PySOT &  Breast cancer dataset &    DT &    ACC &        103.220339 & 32.930646 \\
        TuRBO &  Breast cancer dataset &    DT &    ACC &         99.830508 & 32.930646 \\
    Nevergrad (1+1)&  Breast cancer dataset &    DT &    ACC &         99.661017 & 47.052420 \\
        Skopt &  Breast cancer dataset &    DT &    ACC &         99.322034 & 27.033974 \\
    Opentuner &  Breast cancer dataset &    DT &    ACC &         97.457627 & 19.201984 \\
     Hyperopt &  Breast cancer dataset &    DT &    ACC &         97.288136 & 17.357384 \\
      TuRBO+ &  Breast cancer dataset &    DT &    ACC &         96.610169 & 26.913017 \\
Random-search &  Breast cancer dataset &    DT &    ACC &         93.813559 & 29.415322 \\
         BOHB-BB &  Breast cancer dataset &    DT &    ACC &         93.305085 & 22.823145 \\
\bottomrule
\end{tabular}
\end{table}

\begin{table}[h!]
\centering
\caption{Breast cancer dataset with ACC loss for tuning MLP-adam model, averaged over 20 seeds. Optimiser for this task with highest mean normalised score is Hyperopt.}
\begin{tabular}{llllrr}
\toprule
    Algorithm & Dataset &    Model & Metric &  Normalised Score &  Variance \\
\midrule
     Hyperopt &  Breast cancer dataset & MLP-adam &    ACC &         99.078947 & 14.779851 \\
         HEBO &  Breast cancer dataset & MLP-adam &    ACC &         98.684211 & 20.046654 \\
    Nevergrad (1+1)&  Breast cancer dataset & MLP-adam &    ACC &         96.973684 & 53.415221 \\
        PySOT &  Breast cancer dataset & MLP-adam &    ACC &         96.578947 & 11.736405 \\
        TuRBO &  Breast cancer dataset & MLP-adam &    ACC &         96.447368 & 17.695728 \\
      TuRBO+ &  Breast cancer dataset & MLP-adam &    ACC &         96.447368 & 18.424697 \\
        Skopt &  Breast cancer dataset & MLP-adam &    ACC &         95.921053 & 20.757399 \\
    Opentuner &  Breast cancer dataset & MLP-adam &    ACC &         95.000000 & 21.067211 \\
Random-search &  Breast cancer dataset & MLP-adam &    ACC &         94.078947 & 13.759294 \\
         BOHB-BB &  Breast cancer dataset & MLP-adam &    ACC &         93.684211 & 11.955095 \\
\bottomrule
\end{tabular}
\end{table}

\begin{table}[h!]
\centering
\caption{Breast cancer dataset with ACC loss for tuning MLP-SGD model, averaged over 20 seeds. Optimiser for this task with highest mean normalised score is HEBO.}
\begin{tabular}{llllrr}
\toprule
    Algorithm & Dataset &   Model & Metric &  Normalised Score &   Variance \\
\midrule
         HEBO &  Breast cancer dataset & MLP-SGD &    ACC &        100.847458 &   1.209574 \\
        TuRBO &  Breast cancer dataset & MLP-SGD &    ACC &         99.491525 &   1.391010 \\
        Skopt &  Breast cancer dataset & MLP-SGD &    ACC &         99.237288 &   0.824022 \\
        PySOT &  Breast cancer dataset & MLP-SGD &    ACC &         99.067797 &   1.504407 \\
     Hyperopt &  Breast cancer dataset & MLP-SGD &    ACC &         98.771186 &   1.396680 \\
    Nevergrad (1+1)&  Breast cancer dataset & MLP-SGD &    ACC &         98.771186 &   4.647409 \\
      TuRBO+ &  Breast cancer dataset & MLP-SGD &    ACC &         98.559322 &   2.502306 \\
Random-search &  Breast cancer dataset & MLP-SGD &    ACC &         96.864407 &   1.973117 \\
         BOHB-BB &  Breast cancer dataset & MLP-SGD &    ACC &         96.355932 &   3.787478 \\
    Opentuner &  Breast cancer dataset & MLP-SGD &    ACC &         84.703390 & 190.241386 \\
\bottomrule
\end{tabular}
\end{table}

\begin{table}[h!]
\centering
\caption{Breast cancer dataset with ACC loss for tuning RF model, averaged over 20 seeds. Optimiser for this task with highest mean normalised score is Skopt.}
\begin{tabular}{llllrr}
\toprule
    Algorithm & Dataset & Model & Metric &  Normalised Score &  Variance \\
\midrule
        Skopt &  Breast cancer dataset &    RF &    ACC &         98.846154 &  5.985365 \\
         HEBO &  Breast cancer dataset &    RF &    ACC &         98.356643 &  8.250307 \\
        PySOT &  Breast cancer dataset &    RF &    ACC &         97.377622 &  5.964775 \\
     Hyperopt &  Breast cancer dataset &    RF &    ACC &         96.888112 &  3.988099 \\
        TuRBO &  Breast cancer dataset &    RF &    ACC &         96.853147 &  8.416317 \\
    Nevergrad (1+1)&  Breast cancer dataset &    RF &    ACC &         96.083916 & 13.249908 \\
      TuRBO+ &  Breast cancer dataset &    RF &    ACC &         95.419580 &  6.047137 \\
         BOHB-BB &  Breast cancer dataset &    RF &    ACC &         94.335664 &  2.980457 \\
    Opentuner &  Breast cancer dataset &    RF &    ACC &         93.636364 &  3.289313 \\
Random-search &  Breast cancer dataset &    RF &    ACC &         93.321678 &  2.495296 \\
\bottomrule
\end{tabular}
\end{table}

\begin{table}[h!]
\centering
\caption{Breast cancer dataset with ACC loss for tuning SVM model, averaged over 20 seeds. Optimiser for this task with highest mean normalised score is TuRBO.}
\begin{tabular}{llllrr}
\toprule
    Algorithm & Dataset & Model & Metric &  Normalised Score &   Variance \\
\midrule
        TuRBO &  Breast cancer dataset &   SVM &    ACC &         89.285714 &  40.279270 \\
     Hyperopt &  Breast cancer dataset &   SVM &    ACC &         86.428571 &  53.168636 \\
      TuRBO+ &  Breast cancer dataset &   SVM &    ACC &         86.428571 &  96.133190 \\
        PySOT &  Breast cancer dataset &   SVM &    ACC &         85.714286 &  42.964554 \\
         HEBO &  Breast cancer dataset &   SVM &    ACC &         84.285714 &  40.816327 \\
         BOHB-BB &  Breast cancer dataset &   SVM &    ACC &         80.714286 & 113.319012 \\
    Opentuner &  Breast cancer dataset &   SVM &    ACC &         78.571429 & 182.599356 \\
    Nevergrad (1+1)&  Breast cancer dataset &   SVM &    ACC &         77.142857 & 481.203007 \\
Random-search &  Breast cancer dataset &   SVM &    ACC &         76.428571 &  48.872180 \\
        Skopt &  Breast cancer dataset &   SVM &    ACC &         76.428571 & 349.624060 \\
\bottomrule
\end{tabular}
\end{table}

\begin{table}[h!]
\centering
\caption{Breast cancer dataset with ACC loss for tuning Ada model, averaged over 20 seeds. Optimiser for this task with highest mean normalised score is HEBO.}
\begin{tabular}{llllrr}
\toprule
    Algorithm & Dataset & Model & Metric &  Normalised Score &  Variance \\
\midrule
         HEBO &  Breast cancer dataset &   Ada &    ACC &          99.06250 &  4.214638 \\
    Opentuner &  Breast cancer dataset &   Ada &    ACC &          98.59375 &  6.656044 \\
        TuRBO &  Breast cancer dataset &   Ada &    ACC &          98.43750 &  3.597862 \\
      TuRBO+ &  Breast cancer dataset &   Ada &    ACC &          97.96875 &  8.506373 \\
    Nevergrad (1+1)&  Breast cancer dataset &   Ada &    ACC &          97.65625 & 12.207031 \\
        PySOT &  Breast cancer dataset &   Ada &    ACC &          97.03125 &  5.628084 \\
     Hyperopt &  Breast cancer dataset &   Ada &    ACC &          96.87500 &  8.223684 \\
Random-search &  Breast cancer dataset &   Ada &    ACC &          95.31250 &  6.681743 \\
         BOHB-BB &  Breast cancer dataset &   Ada &    ACC &          94.53125 &  3.983347 \\
        Skopt &  Breast cancer dataset &   Ada &    ACC &          93.12500 &  9.868421 \\
\bottomrule
\end{tabular}
\end{table}

\begin{table}[h!]
\centering
\caption{Breast cancer dataset with ACC loss for tuning Knn model, averaged over 20 seeds. Optimiser for this task with highest mean normalised score is HEBO.}
\begin{tabular}{llllrr}
\toprule
    Algorithm & Dataset & Model & Metric &  Normalised Score &  Variance \\
\midrule
         HEBO &  Breast cancer dataset &   Knn &    ACC &           100.000 &  0.000000 \\
    Opentuner &  Breast cancer dataset &   Knn &    ACC &           100.000 &  0.000000 \\
     Hyperopt &  Breast cancer dataset &   Knn &    ACC &            97.500 & 26.315789 \\
    Nevergrad (1+1)&  Breast cancer dataset &   Knn &    ACC &            95.625 & 53.865132 \\
Random-search &  Breast cancer dataset &   Knn &    ACC &            95.625 & 53.865132 \\
        Skopt &  Breast cancer dataset &   Knn &    ACC &            95.000 & 55.921053 \\
         BOHB-BB &  Breast cancer dataset &   Knn &    ACC &            95.000 & 39.473684 \\
      TuRBO+ &  Breast cancer dataset &   Knn &    ACC &            92.500 & 55.921053 \\
        TuRBO &  Breast cancer dataset &   Knn &    ACC &            91.875 & 37.417763 \\
        PySOT &  Breast cancer dataset &   Knn &    ACC &            91.250 & 34.539474 \\
\bottomrule
\end{tabular}
\end{table}

\begin{table}[h!]
\centering
\caption{Breast cancer dataset with ACC loss for tuning Lasso model, averaged over 20 seeds. Optimiser for this task with highest mean normalised score is PySOT.}
\begin{tabular}{llllrr}
\toprule
    Algorithm & Dataset & Model & Metric &  Normalised Score &   Variance \\
\midrule
        PySOT &  Breast cancer dataset & Lasso &    ACC &         95.000000 &  32.163743 \\
        TuRBO &  Breast cancer dataset & Lasso &    ACC &         93.888889 &  32.163743 \\
    Opentuner &  Breast cancer dataset & Lasso &    ACC &         93.888889 &  32.163743 \\
         HEBO &  Breast cancer dataset & Lasso &    ACC &         93.333333 &  31.189084 \\
      TuRBO+ &  Breast cancer dataset & Lasso &    ACC &         92.222222 &  27.290448 \\
     Hyperopt &  Breast cancer dataset & Lasso &    ACC &         90.555556 &  16.569201 \\
Random-search &  Breast cancer dataset & Lasso &    ACC &         90.555556 &  16.569201 \\
         BOHB-BB &  Breast cancer dataset & Lasso &    ACC &         90.000000 &  11.695906 \\
    Nevergrad (1+1)&  Breast cancer dataset & Lasso &    ACC &         88.333333 &  58.154646 \\
        Skopt &  Breast cancer dataset & Lasso &    ACC &         87.777778 & 180.636777 \\
\bottomrule
\end{tabular}
\end{table}

\begin{table}[h!]
\centering
\caption{Breast cancer dataset with ACC loss for tuning Linear model, averaged over 20 seeds. Optimiser for this task with highest mean normalised score is HEBO.}
\begin{tabular}{llllrr}
\toprule
    Algorithm & Dataset &  Model & Metric &  Normalised Score &   Variance \\
\midrule
         HEBO &  Breast cancer dataset & Linear &    ACC &             107.0 &  95.789474 \\
    Nevergrad (1+1)&  Breast cancer dataset & Linear &    ACC &             103.0 & 390.526316 \\
        TuRBO &  Breast cancer dataset & Linear &    ACC &              99.0 & 188.421053 \\
        PySOT &  Breast cancer dataset & Linear &    ACC &              98.0 & 164.210526 \\
      TuRBO+ &  Breast cancer dataset & Linear &    ACC &              98.0 & 164.210526 \\
    Opentuner &  Breast cancer dataset & Linear &    ACC &              98.0 & 290.526316 \\
     Hyperopt &  Breast cancer dataset & Linear &    ACC &              96.0 & 109.473684 \\
        Skopt &  Breast cancer dataset & Linear &    ACC &              95.0 & 205.263158 \\
Random-search &  Breast cancer dataset & Linear &    ACC &              89.0 & 146.315789 \\
         BOHB-BB &  Breast cancer dataset & Linear &    ACC &              85.0 &  78.947368 \\
\bottomrule
\end{tabular}
\end{table}

\begin{table}[h!]
\centering
\caption{Breast cancer dataset with NLL loss for tuning DT model, averaged over 20 seeds. Optimiser for this task with highest mean normalised score is HEBO.}
\begin{tabular}{llllrr}
\toprule
    Algorithm & Dataset & Model & Metric &  Normalised Score &   Variance \\
\midrule
         HEBO &  Breast cancer dataset &    DT &    NLL &        111.047413 &  51.627976 \\
      TuRBO+ &  Breast cancer dataset &    DT &    NLL &         98.912425 & 121.134669 \\
        PySOT &  Breast cancer dataset &    DT &    NLL &         98.696805 & 115.036349 \\
     Hyperopt &  Breast cancer dataset &    DT &    NLL &         98.201008 &  75.868808 \\
        Skopt &  Breast cancer dataset &    DT &    NLL &         97.647135 & 165.617171 \\
        TuRBO &  Breast cancer dataset &    DT &    NLL &         95.406797 & 230.498379 \\
    Opentuner &  Breast cancer dataset &    DT &    NLL &         91.076386 &  92.789945 \\
    Nevergrad (1+1)&  Breast cancer dataset &    DT &    NLL &         90.579282 &  92.638289 \\
         BOHB-BB &  Breast cancer dataset &    DT &    NLL &         90.418909 &  43.016544 \\
Random-search &  Breast cancer dataset &    DT &    NLL &         87.752435 &  64.953742 \\
\bottomrule
\end{tabular}
\end{table}

\begin{table}[h!]
\centering
\caption{Breast cancer dataset with NLL loss for tuning MLP-adam model, averaged over 20 seeds. Optimiser for this task with highest mean normalised score is HEBO.}
\begin{tabular}{llllrr}
\toprule
    Algorithm & Dataset &    Model & Metric &  Normalised Score &  Variance \\
\midrule
         HEBO &  Breast cancer dataset & MLP-adam &    NLL &        100.303210 &  0.252546 \\
        TuRBO &  Breast cancer dataset & MLP-adam &    NLL &        100.260440 &  0.386150 \\
        PySOT &  Breast cancer dataset & MLP-adam &    NLL &         99.965212 &  0.410841 \\
     Hyperopt &  Breast cancer dataset & MLP-adam &    NLL &         99.778153 &  0.295595 \\
    Opentuner &  Breast cancer dataset & MLP-adam &    NLL &         99.614006 &  0.559840 \\
        Skopt &  Breast cancer dataset & MLP-adam &    NLL &         99.422393 &  0.444736 \\
      TuRBO+ &  Breast cancer dataset & MLP-adam &    NLL &         99.298544 &  1.362840 \\
    Nevergrad (1+1)&  Breast cancer dataset & MLP-adam &    NLL &         98.865542 &  5.064746 \\
Random-search &  Breast cancer dataset & MLP-adam &    NLL &         98.538089 &  1.227034 \\
         BOHB-BB &  Breast cancer dataset & MLP-adam &    NLL &         98.396155 &  0.577588 \\
\bottomrule
\end{tabular}
\end{table}

\begin{table}[h!]
\centering
\caption{Breast cancer dataset with NLL loss for tuning MLP-SGD model, averaged over 20 seeds. Optimiser for this task with highest mean normalised score is HEBO.}
\begin{tabular}{llllrr}
\toprule
    Algorithm & Dataset &   Model & Metric &  Normalised Score &  Variance \\
\midrule
         HEBO &  Breast cancer dataset & MLP-SGD &    NLL &        100.721473 &  0.061946 \\
        TuRBO &  Breast cancer dataset & MLP-SGD &    NLL &        100.162853 &  1.622656 \\
        PySOT &  Breast cancer dataset & MLP-SGD &    NLL &        100.091236 &  0.153482 \\
     Hyperopt &  Breast cancer dataset & MLP-SGD &    NLL &        100.049807 &  0.198361 \\
      TuRBO+ &  Breast cancer dataset & MLP-SGD &    NLL &         99.717751 &  0.569165 \\
        Skopt &  Breast cancer dataset & MLP-SGD &    NLL &         99.682763 &  0.194010 \\
Random-search &  Breast cancer dataset & MLP-SGD &    NLL &         99.084890 &  0.573484 \\
         BOHB-BB &  Breast cancer dataset & MLP-SGD &    NLL &         99.061649 &  0.934630 \\
    Nevergrad (1+1)&  Breast cancer dataset & MLP-SGD &    NLL &         98.762368 &  4.114520 \\
    Opentuner &  Breast cancer dataset & MLP-SGD &    NLL &         97.518858 &  6.031878 \\
\bottomrule
\end{tabular}
\end{table}

\begin{table}[h!]
\centering
\caption{Breast cancer dataset with NLL loss for tuning RF model, averaged over 20 seeds. Optimiser for this task with highest mean normalised score is HEBO.}
\begin{tabular}{llllrr}
\toprule
    Algorithm & Dataset & Model & Metric &  Normalised Score &  Variance \\
\midrule
         HEBO &  Breast cancer dataset &    RF &    NLL &        104.284930 &  1.565493 \\
        PySOT &  Breast cancer dataset &    RF &    NLL &        102.796116 &  4.187314 \\
        Skopt &  Breast cancer dataset &    RF &    NLL &        101.998789 &  4.898058 \\
     Hyperopt &  Breast cancer dataset &    RF &    NLL &         98.198100 & 10.076521 \\
        TuRBO &  Breast cancer dataset &    RF &    NLL &         97.942753 & 25.132499 \\
    Nevergrad (1+1)&  Breast cancer dataset &    RF &    NLL &         96.820234 & 31.448689 \\
      TuRBO+ &  Breast cancer dataset &    RF &    NLL &         96.029192 & 12.874949 \\
    Opentuner &  Breast cancer dataset &    RF &    NLL &         92.790267 & 10.042108 \\
         BOHB-BB &  Breast cancer dataset &    RF &    NLL &         92.660352 &  5.982402 \\
Random-search &  Breast cancer dataset &    RF &    NLL &         92.434443 &  6.196296 \\
\bottomrule
\end{tabular}
\end{table}

\begin{table}[h!]
\centering
\caption{Breast cancer dataset with NLL loss for tuning SVM model, averaged over 20 seeds. Optimiser for this task with highest mean normalised score is HEBO.}
\begin{tabular}{llllrr}
\toprule
    Algorithm & Dataset & Model & Metric &  Normalised Score &  Variance \\
\midrule
         HEBO &  Breast cancer dataset &   SVM &    NLL &        100.999289 &  0.870906 \\
        TuRBO &  Breast cancer dataset &   SVM &    NLL &        100.331139 &  0.502305 \\
        PySOT &  Breast cancer dataset &   SVM &    NLL &        100.038823 &  0.614796 \\
     Hyperopt &  Breast cancer dataset &   SVM &    NLL &         99.238507 &  1.374858 \\
    Opentuner &  Breast cancer dataset &   SVM &    NLL &         99.130932 &  3.472981 \\
      TuRBO+ &  Breast cancer dataset &   SVM &    NLL &         98.787566 &  7.629457 \\
    Nevergrad (1+1)&  Breast cancer dataset &   SVM &    NLL &         98.258990 & 41.913612 \\
        Skopt &  Breast cancer dataset &   SVM &    NLL &         98.153794 & 16.959059 \\
         BOHB-BB &  Breast cancer dataset &   SVM &    NLL &         97.569333 &  1.733710 \\
Random-search &  Breast cancer dataset &   SVM &    NLL &         97.191463 &  1.506776 \\
\bottomrule
\end{tabular}
\end{table}

\begin{table}[h!]
\centering
\caption{Breast cancer dataset with NLL loss for tuning Ada model, averaged over 20 seeds. Optimiser for this task with highest mean normalised score is HEBO.}
\begin{tabular}{llllrr}
\toprule
    Algorithm & Dataset & Model & Metric &  Normalised Score &   Variance \\
\midrule
         HEBO &  Breast cancer dataset &   Ada &    NLL &        103.298980 &   1.970935 \\
        TuRBO &  Breast cancer dataset &   Ada &    NLL &        100.445760 &   7.172282 \\
        PySOT &  Breast cancer dataset &   Ada &    NLL &        100.402760 &   4.365384 \\
        Skopt &  Breast cancer dataset &   Ada &    NLL &         98.923004 &   5.054022 \\
    Nevergrad (1+1)&  Breast cancer dataset &   Ada &    NLL &         98.874005 &   6.293192 \\
     Hyperopt &  Breast cancer dataset &   Ada &    NLL &         98.645011 &   1.948260 \\
      TuRBO+ &  Breast cancer dataset &   Ada &    NLL &         98.474514 &   3.187387 \\
Random-search &  Breast cancer dataset &   Ada &    NLL &         97.922922 &   2.325453 \\
         BOHB-BB &  Breast cancer dataset &   Ada &    NLL &         97.671007 &   2.051746 \\
    Opentuner &  Breast cancer dataset &   Ada &    NLL &         96.538817 & 164.478434 \\
\bottomrule
\end{tabular}
\end{table}

\begin{table}[h!]
\centering
\caption{Breast cancer dataset with NLL loss for tuning Knn model, averaged over 20 seeds. Optimiser for this task with highest mean normalised score is HEBO.}
\begin{tabular}{llllrr}
\toprule
    Algorithm & Dataset & Model & Metric &  Normalised Score &   Variance \\
\midrule
         HEBO &  Breast cancer dataset &   Knn &    NLL &        100.000000 &   0.000000 \\
    Opentuner &  Breast cancer dataset &   Knn &    NLL &        100.000000 &   0.000000 \\
        TuRBO &  Breast cancer dataset &   Knn &    NLL &         99.835048 &   0.257772 \\
     Hyperopt &  Breast cancer dataset &   Knn &    NLL &         99.815764 &   0.321563 \\
Random-search &  Breast cancer dataset &   Knn &    NLL &         99.248073 &   0.729147 \\
         BOHB-BB &  Breast cancer dataset &   Knn &    NLL &         99.121241 &   2.900580 \\
        PySOT &  Breast cancer dataset &   Knn &    NLL &         98.175743 &   5.063927 \\
    Nevergrad (1+1)&  Breast cancer dataset &   Knn &    NLL &         97.334954 &   8.601850 \\
      TuRBO+ &  Breast cancer dataset &   Knn &    NLL &         95.752544 & 121.788536 \\
        Skopt &  Breast cancer dataset &   Knn &    NLL &         85.667024 & 457.645212 \\
\bottomrule
\end{tabular}
\end{table}

\begin{table}[h!]
\centering
\caption{Breast cancer dataset with NLL loss for tuning Lasso model, averaged over 20 seeds. Optimiser for this task with highest mean normalised score is PySOT.}
\begin{tabular}{llllrr}
\toprule
    Algorithm & Dataset & Model & Metric &  Normalised Score &   Variance \\
\midrule
        PySOT &  Breast cancer dataset & Lasso &    NLL &         92.444124 &  41.420464 \\
    Opentuner &  Breast cancer dataset & Lasso &    NLL &         92.364922 &  52.804059 \\
     Hyperopt &  Breast cancer dataset & Lasso &    NLL &         90.483697 &  40.153553 \\
      TuRBO+ &  Breast cancer dataset & Lasso &    NLL &         89.669648 &  32.541741 \\
        TuRBO &  Breast cancer dataset & Lasso &    NLL &         89.275892 &  40.406114 \\
Random-search &  Breast cancer dataset & Lasso &    NLL &         89.267387 &  86.949148 \\
         HEBO &  Breast cancer dataset & Lasso &    NLL &         87.648453 &   9.938537 \\
    Nevergrad (1+1)&  Breast cancer dataset & Lasso &    NLL &         84.512735 &  81.507309 \\
         BOHB-BB &  Breast cancer dataset & Lasso &    NLL &         81.089709 &  45.062520 \\
        Skopt &  Breast cancer dataset & Lasso &    NLL &         80.270494 & 104.357765 \\
\bottomrule
\end{tabular}
\end{table}

\begin{table}[h!]
\centering
\caption{Breast cancer dataset with NLL loss for tuning Linear model, averaged over 20 seeds. Optimiser for this task with highest mean normalised score is HEBO.}
\begin{tabular}{llllrr}
\toprule
    Algorithm & Dataset &  Model & Metric &  Normalised Score &  Variance \\
\midrule
         HEBO &  Breast cancer dataset & Linear &    NLL &         97.097724 &  4.317349 \\
        TuRBO &  Breast cancer dataset & Linear &    NLL &         96.073657 & 20.501975 \\
        PySOT &  Breast cancer dataset & Linear &    NLL &         95.497805 &  3.788872 \\
    Opentuner &  Breast cancer dataset & Linear &    NLL &         94.970576 &  5.694999 \\
        Skopt &  Breast cancer dataset & Linear &    NLL &         93.683985 &  8.175137 \\
      TuRBO+ &  Breast cancer dataset & Linear &    NLL &         93.460437 & 18.265925 \\
     Hyperopt &  Breast cancer dataset & Linear &    NLL &         92.065782 &  7.999446 \\
    Nevergrad (1+1)&  Breast cancer dataset & Linear &    NLL &         91.268968 & 86.801548 \\
Random-search &  Breast cancer dataset & Linear &    NLL &         90.477323 &  8.865931 \\
         BOHB-BB &  Breast cancer dataset & Linear &    NLL &         89.401005 & 15.561542 \\
\bottomrule
\end{tabular}
\end{table}

\clearpage
\subsection{Diabetes Data Set}\label{sec:diabetes}

\begin{table}[h!]
\centering
\caption{Diabetes with MAE loss for tuning DT model, averaged over 20 seeds. Optimiser for this task with highest mean normalised score is HEBO.}
\begin{tabular}{llllrr}
\toprule
    Algorithm &  Dataset & Model & Metric &  Normalised Score &   Variance \\
\midrule
         HEBO & Diabetes &    DT &    MAE &         96.638714 &  19.424434 \\
        Skopt & Diabetes &    DT &    MAE &         96.283909 &  12.675138 \\
        TuRBO & Diabetes &    DT &    MAE &         96.161335 &  22.916642 \\
        PySOT & Diabetes &    DT &    MAE &         95.330459 &  14.421246 \\
     Hyperopt & Diabetes &    DT &    MAE &         92.694296 &  30.442937 \\
      TuRBO+ & Diabetes &    DT &    MAE &         92.417202 &  24.654829 \\
Random-search & Diabetes &    DT &    MAE &         89.414992 &  33.396944 \\
         BOHB-BB & Diabetes &    DT &    MAE &         88.578413 &  21.742304 \\
    Nevergrad (1+1)& Diabetes &    DT &    MAE &         87.192577 & 491.678538 \\
    Opentuner & Diabetes &    DT &    MAE &         84.459098 & 199.389559 \\
\bottomrule
\end{tabular}
\end{table}

\begin{table}[h!]
\centering
\caption{Diabetes with MAE loss for tuning MLP-adam model, averaged over 20 seeds. Optimiser for this task with highest mean normalised score is PySOT.}
\begin{tabular}{llllrr}
\toprule
    Algorithm &  Dataset &    Model & Metric &  Normalised Score &  Variance \\
\midrule
        PySOT & Diabetes & MLP-adam &    MAE &        100.011743 &  0.039479 \\
        TuRBO & Diabetes & MLP-adam &    MAE &        100.003598 &  0.132742 \\
         HEBO & Diabetes & MLP-adam &    MAE &         99.905117 &  0.071792 \\
    Nevergrad (1+1)& Diabetes & MLP-adam &    MAE &         99.783143 &  0.276722 \\
     Hyperopt & Diabetes & MLP-adam &    MAE &         99.747348 &  0.024272 \\
        Skopt & Diabetes & MLP-adam &    MAE &         99.691162 &  0.057345 \\
    Opentuner & Diabetes & MLP-adam &    MAE &         99.564832 &  0.075922 \\
      TuRBO+ & Diabetes & MLP-adam &    MAE &         99.518025 &  0.125512 \\
         BOHB-BB & Diabetes & MLP-adam &    MAE &         99.343161 &  0.055340 \\
Random-search & Diabetes & MLP-adam &    MAE &         99.288591 &  0.104943 \\
\bottomrule
\end{tabular}
\end{table}

\begin{table}[h!]
\centering
\caption{Diabetes with MAE loss for tuning MLP-SGD model, averaged over 20 seeds. Optimiser for this task with highest mean normalised score is HEBO.}
\begin{tabular}{llllrr}
\toprule
    Algorithm &  Dataset &   Model & Metric &  Normalised Score &   Variance \\
\midrule
         HEBO & Diabetes & MLP-SGD &    MAE &         99.305259 &   0.571598 \\
        TuRBO & Diabetes & MLP-SGD &    MAE &         99.173074 &   2.151207 \\
        PySOT & Diabetes & MLP-SGD &    MAE &         98.493992 &   3.290314 \\
      TuRBO+ & Diabetes & MLP-SGD &    MAE &         98.017532 &   2.468709 \\
     Hyperopt & Diabetes & MLP-SGD &    MAE &         97.991732 &   1.875066 \\
        Skopt & Diabetes & MLP-SGD &    MAE &         97.323238 &   1.402552 \\
         BOHB-BB & Diabetes & MLP-SGD &    MAE &         94.355532 &   9.155887 \\
Random-search & Diabetes & MLP-SGD &    MAE &         94.196391 &  42.702734 \\
    Opentuner & Diabetes & MLP-SGD &    MAE &         93.018054 & 116.857957 \\
    Nevergrad (1+1)& Diabetes & MLP-SGD &    MAE &         87.052894 & 249.800579 \\
\bottomrule
\end{tabular}
\end{table}

\begin{table}[h!]
\centering
\caption{Diabetes with MAE loss for tuning RF model, averaged over 20 seeds. Optimiser for this task with highest mean normalised score is TuRBO.}
\begin{tabular}{llllrr}
\toprule
    Algorithm &  Dataset & Model & Metric &  Normalised Score &  Variance \\
\midrule
        TuRBO & Diabetes &    RF &    MAE &        100.689329 &  1.625789 \\
        PySOT & Diabetes &    RF &    MAE &        100.621188 &  2.067673 \\
         HEBO & Diabetes &    RF &    MAE &        100.420933 &  1.415993 \\
        Skopt & Diabetes &    RF &    MAE &         99.206462 &  1.447476 \\
    Nevergrad (1+1)& Diabetes &    RF &    MAE &         98.695634 & 22.941598 \\
     Hyperopt & Diabetes &    RF &    MAE &         98.321835 &  1.287212 \\
      TuRBO+ & Diabetes &    RF &    MAE &         97.985211 &  5.657556 \\
         BOHB-BB & Diabetes &    RF &    MAE &         96.181909 &  6.997435 \\
Random-search & Diabetes &    RF &    MAE &         95.677899 &  5.078577 \\
    Opentuner & Diabetes &    RF &    MAE &         93.512496 & 46.716147 \\
\bottomrule
\end{tabular}
\end{table}

\begin{table}[h!]
\centering
\caption{Diabetes with MAE loss for tuning SVM model, averaged over 20 seeds. Optimiser for this task with highest mean normalised score is Nevergrad.}
\begin{tabular}{llllrr}
\toprule
    Algorithm &  Dataset & Model & Metric &  Normalised Score &     Variance \\
\midrule
    Nevergrad (1+1)& Diabetes &   SVM &    MAE &        114.952221 & 1.997569e-23 \\
         HEBO & Diabetes &   SVM &    MAE &        114.952206 & 4.667718e-10 \\
        TuRBO & Diabetes &   SVM &    MAE &        114.828240 & 9.438118e-03 \\
        PySOT & Diabetes &   SVM &    MAE &        114.410275 & 1.494253e-01 \\
        Skopt & Diabetes &   SVM &    MAE &        113.977223 & 2.466076e-01 \\
    Opentuner & Diabetes &   SVM &    MAE &        113.263069 & 9.959103e+00 \\
      TuRBO+ & Diabetes &   SVM &    MAE &        106.842965 & 1.060824e+02 \\
     Hyperopt & Diabetes &   SVM &    MAE &        104.021953 & 8.292988e+01 \\
Random-search & Diabetes &   SVM &    MAE &         76.950614 & 2.326428e+02 \\
         BOHB-BB & Diabetes &   SVM &    MAE &         70.832961 & 2.115443e+02 \\
\bottomrule
\end{tabular}
\end{table}

\begin{table}[h!]
\centering
\caption{Diabetes with MAE loss for tuning Ada model, averaged over 20 seeds. Optimiser for this task with highest mean normalised score is Opentuner.}
\begin{tabular}{llllrr}
\toprule
    Algorithm &  Dataset & Model & Metric &  Normalised Score &   Variance \\
\midrule
    Opentuner & Diabetes &   Ada &    MAE &         82.820466 & 729.771421 \\
         HEBO & Diabetes &   Ada &    MAE &         81.642682 & 638.761474 \\
        TuRBO & Diabetes &   Ada &    MAE &         80.090059 & 159.113760 \\
        PySOT & Diabetes &   Ada &    MAE &         79.899689 & 271.945194 \\
     Hyperopt & Diabetes &   Ada &    MAE &         79.323266 & 146.085370 \\
         BOHB-BB & Diabetes &   Ada &    MAE &         77.689133 & 209.828285 \\
      TuRBO+ & Diabetes &   Ada &    MAE &         76.467861 & 135.190360 \\
        Skopt & Diabetes &   Ada &    MAE &         75.430634 & 172.038737 \\
    Nevergrad (1+1)& Diabetes &   Ada &    MAE &         73.932751 & 141.078137 \\
Random-search & Diabetes &   Ada &    MAE &         73.572795 & 183.641961 \\
\bottomrule
\end{tabular}
\end{table}

\begin{table}[h!]
\centering
\caption{Diabetes with MAE loss for tuning Knn model, averaged over 20 seeds. Optimiser for this task with highest mean normalised score is HEBO.}
\begin{tabular}{llllrr}
\toprule
    Algorithm &  Dataset & Model & Metric &  Normalised Score &   Variance \\
\midrule
         HEBO & Diabetes &   Knn &    MAE &        100.000000 &   0.000000 \\
    Opentuner & Diabetes &   Knn &    MAE &        100.000000 &   0.000000 \\
        PySOT & Diabetes &   Knn &    MAE &         99.918437 &   0.133050 \\
        TuRBO & Diabetes &   Knn &    MAE &         99.276294 &  10.475009 \\
     Hyperopt & Diabetes &   Knn &    MAE &         97.929753 &  28.990634 \\
    Nevergrad (1+1)& Diabetes &   Knn &    MAE &         97.515979 &  58.516868 \\
Random-search & Diabetes &   Knn &    MAE &         95.595510 &  46.440649 \\
         BOHB-BB & Diabetes &   Knn &    MAE &         95.543615 &  56.425722 \\
        Skopt & Diabetes &   Knn &    MAE &         95.429465 &  66.513444 \\
      TuRBO+ & Diabetes &   Knn &    MAE &         91.604083 & 134.441613 \\
\bottomrule
\end{tabular}
\end{table}

\begin{table}[h!]
\centering
\caption{Diabetes with MAE loss for tuning Lasso model, averaged over 20 seeds. Optimiser for this task with highest mean normalised score is Skopt.}
\begin{tabular}{llllrr}
\toprule
    Algorithm &  Dataset & Model & Metric &  Normalised Score &  Variance \\
\midrule
        Skopt & Diabetes & Lasso &    MAE &        100.024636 &  0.001070 \\
         HEBO & Diabetes & Lasso &    MAE &        100.020048 &  0.000540 \\
        TuRBO & Diabetes & Lasso &    MAE &        100.011646 &  0.000588 \\
        PySOT & Diabetes & Lasso &    MAE &        100.006646 &  0.000432 \\
     Hyperopt & Diabetes & Lasso &    MAE &        100.004698 &  0.000154 \\
      TuRBO+ & Diabetes & Lasso &    MAE &         99.997964 &  0.000615 \\
         BOHB-BB & Diabetes & Lasso &    MAE &         99.995219 &  0.000595 \\
    Nevergrad (1+1)& Diabetes & Lasso &    MAE &         99.993918 &  0.000879 \\
Random-search & Diabetes & Lasso &    MAE &         99.993704 &  0.000037 \\
    Opentuner & Diabetes & Lasso &    MAE &         99.813361 &  0.285529 \\
\bottomrule
\end{tabular}
\end{table}

\begin{table}[h!]
\centering
\caption{Diabetes with MAE loss for tuning Linear model, averaged over 20 seeds. Optimiser for this task with highest mean normalised score is HEBO.}
\begin{tabular}{llllrr}
\toprule
    Algorithm &  Dataset &  Model & Metric &  Normalised Score &     Variance \\
\midrule
         HEBO & Diabetes & Linear &    MAE &         99.999997 & 7.405442e-11 \\
        TuRBO & Diabetes & Linear &    MAE &         99.999955 & 5.421510e-08 \\
        PySOT & Diabetes & Linear &    MAE &         99.999884 & 8.825730e-08 \\
    Nevergrad (1+1)& Diabetes & Linear &    MAE &         99.999798 & 1.823753e-07 \\
     Hyperopt & Diabetes & Linear &    MAE &         99.999756 & 1.834449e-07 \\
        Skopt & Diabetes & Linear &    MAE &         99.999461 & 3.064474e-07 \\
         BOHB-BB & Diabetes & Linear &    MAE &         99.999386 & 4.929089e-07 \\
Random-search & Diabetes & Linear &    MAE &         99.999226 & 7.463938e-07 \\
      TuRBO+ & Diabetes & Linear &    MAE &         99.998033 & 5.543436e-05 \\
    Opentuner & Diabetes & Linear &    MAE &         99.990706 & 1.793190e-04 \\
\bottomrule
\end{tabular}
\end{table}

\begin{table}[h!]
\centering
\caption{Diabetes with MSE loss for tuning DT model, averaged over 20 seeds. Optimiser for this task with highest mean normalised score is TuRBO.}
\begin{tabular}{llllrr}
\toprule
    Algorithm &  Dataset & Model & Metric &  Normalised Score &   Variance \\
\midrule
        TuRBO & Diabetes &    DT &    MSE &         99.159752 &  20.394596 \\
         HEBO & Diabetes &    DT &    MSE &         99.105320 &  27.539740 \\
        PySOT & Diabetes &    DT &    MSE &         98.463863 &  12.943768 \\
      TuRBO+ & Diabetes &    DT &    MSE &         98.146099 &  29.207851 \\
        Skopt & Diabetes &    DT &    MSE &         98.011055 &   9.114515 \\
    Nevergrad (1+1)& Diabetes &    DT &    MSE &         95.965807 &  89.742185 \\
     Hyperopt & Diabetes &    DT &    MSE &         94.106805 &  17.465775 \\
         BOHB-BB & Diabetes &    DT &    MSE &         91.746558 &  29.504056 \\
Random-search & Diabetes &    DT &    MSE &         90.512956 &  24.763756 \\
    Opentuner & Diabetes &    DT &    MSE &         85.951624 & 189.632259 \\
\bottomrule
\end{tabular}
\end{table}

\begin{table}[h!]
\centering
\caption{Diabetes with MSE loss for tuning MLP-adam model, averaged over 20 seeds. Optimiser for this task with highest mean normalised score is TuRBO.}
\begin{tabular}{llllrr}
\toprule
    Algorithm &  Dataset &    Model & Metric &  Normalised Score &  Variance \\
\midrule
        TuRBO & Diabetes & MLP-adam &    MSE &        100.080222 &  0.011338 \\
        PySOT & Diabetes & MLP-adam &    MSE &        100.021190 &  0.010076 \\
    Nevergrad (1+1)& Diabetes & MLP-adam &    MSE &        100.017846 &  0.019891 \\
         HEBO & Diabetes & MLP-adam &    MSE &         99.977968 &  0.004366 \\
     Hyperopt & Diabetes & MLP-adam &    MSE &         99.937035 &  0.011698 \\
    Opentuner & Diabetes & MLP-adam &    MSE &         99.886863 &  0.005519 \\
        Skopt & Diabetes & MLP-adam &    MSE &         99.838359 &  0.008621 \\
      TuRBO+ & Diabetes & MLP-adam &    MSE &         99.825080 &  0.033512 \\
Random-search & Diabetes & MLP-adam &    MSE &         99.777438 &  0.009177 \\
         BOHB-BB & Diabetes & MLP-adam &    MSE &         99.742448 &  0.011606 \\
\bottomrule
\end{tabular}
\end{table}

\begin{table}[h!]
\centering
\caption{Diabetes with MSE loss for tuning MLP-SGD model, averaged over 20 seeds. Optimiser for this task with highest mean normalised score is HEBO.}
\begin{tabular}{llllrr}
\toprule
    Algorithm &  Dataset &   Model & Metric &  Normalised Score &   Variance \\
\midrule
         HEBO & Diabetes & MLP-SGD &    MSE &         99.502940 &   0.179089 \\
        TuRBO & Diabetes & MLP-SGD &    MSE &         99.498069 &   0.391187 \\
     Hyperopt & Diabetes & MLP-SGD &    MSE &         98.948022 &   0.324618 \\
      TuRBO+ & Diabetes & MLP-SGD &    MSE &         98.704370 &   1.720076 \\
        Skopt & Diabetes & MLP-SGD &    MSE &         98.452683 &   1.161588 \\
        PySOT & Diabetes & MLP-SGD &    MSE &         98.357894 &   3.117919 \\
Random-search & Diabetes & MLP-SGD &    MSE &         97.569960 &   3.685257 \\
         BOHB-BB & Diabetes & MLP-SGD &    MSE &         95.630189 &  18.807287 \\
    Opentuner & Diabetes & MLP-SGD &    MSE &         94.462201 &  34.471258 \\
    Nevergrad (1+1)& Diabetes & MLP-SGD &    MSE &         92.620769 & 179.226746 \\
\bottomrule
\end{tabular}
\end{table}

\begin{table}[h!]
\centering
\caption{Diabetes with MSE loss for tuning RF model, averaged over 20 seeds. Optimiser for this task with highest mean normalised score is HEBO.}
\begin{tabular}{llllrr}
\toprule
    Algorithm &  Dataset & Model & Metric &  Normalised Score &  Variance \\
\midrule
         HEBO & Diabetes &    RF &    MSE &         99.567950 &  1.137694 \\
        PySOT & Diabetes &    RF &    MSE &         99.197939 &  1.240447 \\
        TuRBO & Diabetes &    RF &    MSE &         99.154475 &  1.995345 \\
        Skopt & Diabetes &    RF &    MSE &         98.665889 &  0.909053 \\
     Hyperopt & Diabetes &    RF &    MSE &         98.271427 &  1.306007 \\
      TuRBO+ & Diabetes &    RF &    MSE &         97.396962 &  3.233997 \\
    Nevergrad (1+1)& Diabetes &    RF &    MSE &         96.904507 & 16.420320 \\
Random-search & Diabetes &    RF &    MSE &         96.478596 &  4.042486 \\
         BOHB-BB & Diabetes &    RF &    MSE &         96.269017 &  3.544861 \\
    Opentuner & Diabetes &    RF &    MSE &         95.846844 & 22.993955 \\
\bottomrule
\end{tabular}
\end{table}

\begin{table}[h!]
\centering
\caption{Diabetes with MSE loss for tuning SVM model, averaged over 20 seeds. Optimiser for this task with highest mean normalised score is Nevergrad.}
\begin{tabular}{llllrr}
\toprule
    Algorithm &  Dataset & Model & Metric &  Normalised Score &     Variance \\
\midrule
    Nevergrad (1+1)& Diabetes &   SVM &    MSE &        102.789567 & 2.184625e-24 \\
         HEBO & Diabetes &   SVM &    MSE &        102.789540 & 1.530581e-09 \\
        TuRBO & Diabetes &   SVM &    MSE &        102.673445 & 8.390893e-03 \\
        PySOT & Diabetes &   SVM &    MSE &        102.020190 & 4.231205e-01 \\
        Skopt & Diabetes &   SVM &    MSE &        101.868102 & 1.554345e-01 \\
    Opentuner & Diabetes &   SVM &    MSE &        100.897226 & 1.367135e+01 \\
     Hyperopt & Diabetes &   SVM &    MSE &         95.268658 & 4.223803e+01 \\
      TuRBO+ & Diabetes &   SVM &    MSE &         94.606843 & 4.791899e+01 \\
         BOHB-BB & Diabetes &   SVM &    MSE &         74.069393 & 1.232001e+02 \\
Random-search & Diabetes &   SVM &    MSE &         65.895981 & 1.918664e+02 \\
\bottomrule
\end{tabular}
\end{table}

\begin{table}[h!]
\centering
\caption{Diabetes with MSE loss for tuning Ada model, averaged over 20 seeds. Optimiser for this task with highest mean normalised score is Nevergrad.}
%
\end{table}

\begin{table}[h!]
\centering
\caption{Diabetes with MSE loss for tuning Knn model, averaged over 20 seeds. Optimiser for this task with highest mean normalised score is HEBO.}
%
\end{table}

\begin{table}[h!]
\centering
\caption{Diabetes with MSE loss for tuning Lasso model, averaged over 20 seeds. Optimiser for this task with highest mean normalised score is TuRBO.}
%
\end{table}

\begin{table}[h!]
\centering
\caption{Diabetes with MSE loss for tuning Linear model, averaged over 20 seeds. Optimiser for this task with highest mean normalised score is HEBO.}
%
\end{table}

\clearpage
\subsection{Digits Data Set}\label{sec:digits}

\begin{table}[h!]
\centering
\caption{Digits with ACC loss for tuning DT model, averaged over 20 seeds. Optimiser for this task with highest mean normalised score is HEBO.}
%
\end{table}

\begin{table}[h!]
\centering
\caption{Digits with ACC loss for tuning MLP-adam model, averaged over 20 seeds. Optimiser for this task with highest mean normalised score is TuRBO.}
%
\end{table}

\begin{table}[h!]
\centering
\caption{Digits with ACC loss for tuning MLP-SGD model, averaged over 20 seeds. Optimiser for this task with highest mean normalised score is HEBO.}
%
\end{table}

\begin{table}[h!]
\centering
\caption{Digits with ACC loss for tuning RF model, averaged over 20 seeds. Optimiser for this task with highest mean normalised score is HEBO.}
%
\end{table}

\begin{table}[h!]
\centering
\caption{Digits with ACC loss for tuning SVM model, averaged over 20 seeds. Optimiser for this task with highest mean normalised score is HEBO.}
%
\end{table}

\begin{table}[h!]
\centering
\caption{Digits with ACC loss for tuning Ada model, averaged over 20 seeds. Optimiser for this task with highest mean normalised score is HEBO.}
%
\end{table}

\begin{table}[h!]
\centering
\caption{Digits with ACC loss for tuning Knn model, averaged over 20 seeds. Optimiser for this task with highest mean normalised score is HEBO.}
%
\end{table}

\begin{table}[h!]
\centering
\caption{Digits with ACC loss for tuning Lasso model, averaged over 20 seeds. Optimiser for this task with highest mean normalised score is HEBO.}
%
\end{table}

\begin{table}[h!]
\centering
\caption{Digits with ACC loss for tuning Linear model, averaged over 20 seeds. Optimiser for this task with highest mean normalised score is HEBO.}
%
\end{table}

\begin{table}[h!]
\centering
\caption{Digits with NLL loss for tuning DT model, averaged over 20 seeds. Optimiser for this task with highest mean normalised score is TuRBO.}
%
\end{table}

\begin{table}[h!]
\centering
\caption{Digits with NLL loss for tuning MLP-adam model, averaged over 20 seeds. Optimiser for this task with highest mean normalised score is TuRBO.}
%
\end{table}

\begin{table}[h!]
\centering
\caption{Digits with NLL loss for tuning MLP-SGD model, averaged over 20 seeds. Optimiser for this task with highest mean normalised score is HEBO.}
%
\end{table}

\begin{table}[h!]
\centering
\caption{Digits with NLL loss for tuning RF model, averaged over 20 seeds. Optimiser for this task with highest mean normalised score is HEBO.}
%
\end{table}

\begin{table}[h!]
\centering
\caption{Digits with NLL loss for tuning SVM model, averaged over 20 seeds. Optimiser for this task with highest mean normalised score is HEBO.}
%
\end{table}

\begin{table}[h!]
\centering
\caption{Digits with NLL loss for tuning Ada model, averaged over 20 seeds. Optimiser for this task with highest mean normalised score is HEBO.}
%
\end{table}

\begin{table}[h!]
\centering
\caption{Digits with NLL loss for tuning Knn model, averaged over 20 seeds. Optimiser for this task with highest mean normalised score is HEBO.}
%
\end{table}

\begin{table}[h!]
\centering
\caption{Digits with NLL loss for tuning Lasso model, averaged over 20 seeds. Optimiser for this task with highest mean normalised score is HEBO.}
%
\end{table}

\begin{table}[h!]
\centering
\caption{Digits with NLL loss for tuning Linear model, averaged over 20 seeds. Optimiser for this task with highest mean normalised score is HEBO.}
%
\end{table}

\clearpage

\subsection{Iris Data Set}\label{sec:iris}

\begin{table}[h!]
\centering
\caption{Iris with ACC loss for tuning DT model, averaged over 20 seeds. Optimiser for this task with highest mean normalised score is HEBO.}
%
\end{table}

\begin{table}[h!]
\centering
\caption{Iris with ACC loss for tuning MLP-adam model, averaged over 20 seeds. Optimiser for this task with highest mean normalised score is HEBO.}
%
\end{table}

\begin{table}[h!]
\centering
\caption{Iris with ACC loss for tuning MLP-SGD model, averaged over 20 seeds. Optimiser for this task with highest mean normalised score is TuRBO.}
%
\end{table}

\begin{table}[h!]
\centering
\caption{Iris with ACC loss for tuning RF model, averaged over 20 seeds. Optimiser for this task with highest mean normalised score is Nevergrad.}
%
\end{table}

\begin{table}[h!]
\centering
\caption{Iris with ACC loss for tuning SVM model, averaged over 20 seeds. Optimiser for this task with highest mean normalised score is TuRBO+.}
%
\end{table}

\begin{table}[h!]
\centering
\caption{Iris with ACC loss for tuning Ada model, averaged over 20 seeds. Optimiser for this task with highest mean normalised score is BOHB.}
%
\end{table}

\begin{table}[h!]
\centering
\caption{Iris with ACC loss for tuning Knn model, averaged over 20 seeds. Optimiser for this task with highest mean normalised score is BOHB.}
%
\end{table}

\begin{table}[h!]
\centering
\caption{Iris with ACC loss for tuning Lasso model, averaged over 20 seeds. Optimiser for this task with highest mean normalised score is BOHB.}
%
\end{table}

\begin{table}[h!]
\centering
\caption{Iris with ACC loss for tuning Linear model, averaged over 20 seeds. Optimiser for this task with highest mean normalised score is HEBO.}
%
\end{table}

\begin{table}[h!]
\centering
\caption{Iris with NLL loss for tuning DT model, averaged over 20 seeds. Optimiser for this task with highest mean normalised score is Skopt.}
%
\end{table}

\begin{table}[h!]
\centering
\caption{Iris with NLL loss for tuning MLP-adam model, averaged over 20 seeds. Optimiser for this task with highest mean normalised score is HEBO.}
%
\end{table}

\begin{table}[h!]
\centering
\caption{Iris with NLL loss for tuning MLP-SGD model, averaged over 20 seeds. Optimiser for this task with highest mean normalised score is HEBO.}
%
\end{table}

\begin{table}[h!]
\centering
\caption{Iris with NLL loss for tuning RF model, averaged over 20 seeds. Optimiser for this task with highest mean normalised score is HEBO.}
%
\end{table}

\begin{table}[h!]
\centering
\caption{Iris with NLL loss for tuning SVM model, averaged over 20 seeds. Optimiser for this task with highest mean normalised score is HEBO.}
%
\end{table}

\begin{table}[h!]
\centering
\caption{Iris with NLL loss for tuning Ada model, averaged over 20 seeds. Optimiser for this task with highest mean normalised score is Opentuner.}
%
\end{table}

\begin{table}[h!]
\centering
\caption{Iris with NLL loss for tuning Knn model, averaged over 20 seeds. Optimiser for this task with highest mean normalised score is HEBO.}
%
\end{table}

\begin{table}[h!]
\centering
\caption{Iris with NLL loss for tuning Lasso model, averaged over 20 seeds. Optimiser for this task with highest mean normalised score is TuRBO.}
%
\end{table}

\begin{table}[h!]
\centering
\caption{Iris with NLL loss for tuning Linear model, averaged over 20 seeds. Optimiser for this task with highest mean normalised score is TuRBO.}
%
\end{table}

\clearpage
\subsection{Wine Data Set}\label{sec:wine}

\begin{table}[h!]
\centering
\caption{Wine with ACC loss for tuning DT model, averaged over 20 seeds. Optimiser for this task with highest mean normalised score is PySOT.}
%
\end{table}

\begin{table}[h!]
\centering
\caption{Wine with ACC loss for tuning MLP-adam model, averaged over 20 seeds. Optimiser for this task with highest mean normalised score is HEBO.}
%
\end{table}

\begin{table}[h!]
\centering
\caption{Wine with ACC loss for tuning MLP-SGD model, averaged over 20 seeds. Optimiser for this task with highest mean normalised score is HEBO.}
%
\end{table}

\begin{table}[h!]
\centering
\caption{Wine with ACC loss for tuning RF model, averaged over 20 seeds. Optimiser for this task with highest mean normalised score is Skopt.}
%
\end{table}

\begin{table}[h!]
\centering
\caption{Wine with ACC loss for tuning SVM model, averaged over 20 seeds. Optimiser for this task with highest mean normalised score is HEBO.}
%
\end{table}

\begin{table}[h!]
\centering
\caption{Wine with ACC loss for tuning Ada model, averaged over 20 seeds. Optimiser for this task with highest mean normalised score is HEBO.}
%
\end{table}

\begin{table}[h!]
\centering
\caption{Wine with ACC loss for tuning Knn model, averaged over 20 seeds. Optimiser for this task with highest mean normalised score is HEBO.}
%
\end{table}

\begin{table}[h!]
\centering
\caption{Wine with ACC loss for tuning Lasso model, averaged over 20 seeds. Optimiser for this task with highest mean normalised score is PySOT.}
%
\end{table}

\begin{table}[h!]
\centering
\caption{Wine with ACC loss for tuning Linear model, averaged over 20 seeds. Optimiser for this task with highest mean normalised score is HEBO.}
%
\end{table}

\begin{table}[h!]
\centering
\caption{Wine with NLL loss for tuning DT model, averaged over 20 seeds. Optimiser for this task with highest mean normalised score is PySOT.}
%
\end{table}

\begin{table}[h!]
\centering
\caption{Wine with NLL loss for tuning MLP-adam model, averaged over 20 seeds. Optimiser for this task with highest mean normalised score is TuRBO.}
%
\end{table}

\begin{table}[h!]
\centering
\caption{Wine with NLL loss for tuning MLP-SGD model, averaged over 20 seeds. Optimiser for this task with highest mean normalised score is HEBO.}
%
\end{table}

\begin{table}[h!]
\centering
\caption{Wine with NLL loss for tuning RF model, averaged over 20 seeds. Optimiser for this task with highest mean normalised score is HEBO.}
%
\end{table}

\begin{table}[h!]
\centering
\caption{Wine with NLL loss for tuning SVM model, averaged over 20 seeds. Optimiser for this task with highest mean normalised score is HEBO.}
%
\end{table}

\begin{table}[h!]
\centering
\caption{Wine with NLL loss for tuning Ada model, averaged over 20 seeds. Optimiser for this task with highest mean normalised score is HEBO.}
%
\end{table}

\begin{table}[h!]
\centering
\caption{Wine with NLL loss for tuning Knn model, averaged over 20 seeds. Optimiser for this task with highest mean normalised score is HEBO.}
%
\end{table}

\begin{table}[h!]
\centering
\caption{Wine with NLL loss for tuning Lasso model, averaged over 20 seeds. Optimiser for this task with highest mean normalised score is HEBO.}
%
\end{table}

\begin{table}[h!]
\centering
\caption{Wine with NLL loss for tuning Linear model, averaged over 20 seeds. Optimiser for this task with highest mean normalised score is HEBO.}
%
\end{table}


\bibliographystyle{theapa}
\bibliography{sample}

\end{document}